\documentclass{article}
\PassOptionsToPackage{numbers}{natbib} 
\usepackage[preprint]{my_neurips}

\usepackage[utf8]{inputenc} % allow utf-8 input
\usepackage[T1]{fontenc}    % use 8-bit T1 fonts
\usepackage{hyperref}       % hyperlinks
\usepackage{url}            % simple URL typesetting
\usepackage{booktabs}       % professional-quality tables
\usepackage{amsfonts}       % blackboard math symbols
\usepackage{nicefrac}       % compact symbols for 1/2, etc.
\usepackage{microtype}      % microtypography
\usepackage{xcolor}         % colors
\usepackage{amsfonts}
\usepackage{amsmath}
\usepackage{amsthm}
\usepackage{amssymb}
\usepackage{dsfont}
\usepackage{multirow}
\usepackage{xcolor}
\usepackage{color}
\usepackage{graphicx}
\usepackage{verbatim}
\usepackage{xspace} % for algorithm names

\usepackage{enumerate}
\usepackage{enumitem}

\providecommand{\lin}[1]{\ensuremath{\left\langle #1 \right\rangle}}
\providecommand{\abs}[1]{\left\lvert#1\right\rvert}
\providecommand{\norm}[1]{\left\lVert#1\right\rVert}

  \providecommand{\R}{\mathbb{R}} % Reals
  \DeclareMathOperator{\E}{{\mathbb E}}
  \DeclareMathOperator*{\argmin}{arg\,min}

  \renewcommand{\aa}{\mathbf{a}}
  \providecommand{\bb}{\mathbf{b}}

  \renewcommand{\gg}{\mathbf{g}}
  \providecommand{\hh}{\mathbf{h}}

  \providecommand{\mm}{\mathbf{m}}

  \providecommand{\vv}{\mathbf{v}}
  \providecommand{\ww}{\mathbf{w}}
  \providecommand{\xx}{\mathbf{x}}
  \providecommand{\yy}{\mathbf{y}}

  \providecommand{\mA}{\mathbf{A}}
  \providecommand{\mB}{\mathbf{B}}

  \providecommand{\mG}{\mathbf{G}}

  \providecommand{\cA}{\mathcal{A}}

  \providecommand{\cF}{\mathcal{F}}
  
  \providecommand{\cH}{\mathcal{H}}
  \providecommand{\cI}{\mathcal{I}}

  \providecommand{\cM}{\mathcal{M}}
  
  \providecommand{\cO}{\mathcal{O}}

  \providecommand{\cR}{\mathcal{R}}

  \providecommand{\RSS}{\operatorname{R-CSS}}
  \providecommand{\OO}{\operatorname{O}}
  
  \providecommand{\FO}{\operatorname{FO}}
  
  \providecommand{\DSS}{\operatorname{D-CSS}}
  \providecommand{\ASS}{\operatorname{A-CSS}}
  \providecommand{\Avg}{{\frac{1}{n}\sum_{i=1}^n}}

  \newcommand{\algoname}[1]{\textsc{#1}\xspace} %% renamed to \algoname, as there was a clash with other packages somehow...

  \newcommand{\defeq}{:=}

  \usepackage{bm}

\RequirePackage[colorinlistoftodos,bordercolor=orange,backgroundcolor=orange!20,linecolor=orange,textsize=scriptsize]{todonotes}
\providecommand{\mycomment}[3]{\todo[caption={},color=#3!20,inline]{\textbf{#1: }#2}}%
\providecommand{\myinlinecomment}[3]{%
  {\color{#1}#2: #3}}%
\newcommand\commenter[2]%
{%
  \expandafter\newcommand\csname i#1\endcsname[1]{\myinlinecomment{#2}{#1}{##1}}
  \expandafter\newcommand\csname #1\endcsname[1]{\mycomment{#1}{##1}{#2}}
}
  
\theoremstyle{plain}
\newtheorem{theorem}{Theorem}[section]

\newtheorem{lemma}[theorem]{Lemma}
\newtheorem{corollary}[theorem]{Corollary}
\theoremstyle{definition}

\newtheorem{assumption}[theorem]{Assumption}
\theoremstyle{remark}
\newtheorem{remark}[theorem]{Remark}
\usepackage[capitalize,noabbrev]{cleveref}

\usepackage{url}

\renewcommand{\epsilon}{\varepsilon}

\usepackage{tcolorbox}

\newtcbox{\comparison}{on line,
  colframe=blue,colback=white,
  boxrule=0.5pt,arc=4pt,boxsep=0pt,left=6pt,right=6pt,top=6pt,bottom=6pt}
\usepackage{algorithm}
\usepackage{algorithmic}
\usepackage[font=small]{caption}
\usepackage{enumitem}

\usepackage{titlesec}
\titlespacing*{\section}{0pt}{1ex plus 0.75ex minus .2ex}{1ex plus 0.2ex}
\titlespacing*{\subsection}{0pt}{0.75ex plus 0.75ex minus .2ex}{0.75ex plus .2ex}

\title{Non-Convex Federated Optimization under Cost-Aware Client Selection}

\author{%
  Xiaowen Jiang  \\
  Saarland University \&
  CISPA\thanks{CISPA Helmholtz Center for Information Security, Saarbrücken, Germany}   
  \\
  \texttt{xiaowen.jiang@cispa.de}\\
  \And
  Anton Rodomanov  \\ 
  CISPA\footnotemark[1] \\ 
  \texttt{anton.rodomanov@cispa.de} \\
  \And
  Sebastian U. Stich \\ 
  CISPA\footnotemark[1] \\ 
  \texttt{stich@cispa.de} \\
}

\begin{document}

\maketitle

\begin{abstract} 
  Different federated optimization algorithms typically employ distinct client-selection strategies: some methods communicate only with a randomly sampled subset of clients at each round, while others need to periodically communicate with all clients or use a hybrid scheme that combines both strategies. However,
  existing metrics for comparing optimization methods typically do not distinguish between these strategies, which often incur different communication costs in practice.
  To address this disparity,
  we introduce a simple and natural model of federated optimization that quantifies communication and local computation complexities.
  This new model allows for several commonly used client-selection strategies and explicitly associates each with a distinct cost. Within this setting, we propose a new algorithm that achieves the best-known communication and local complexities among existing federated optimization methods for non-convex optimization. 
  This algorithm is based on the inexact composite gradient method with 
  a carefully constructed gradient estimator and a special procedure for solving the auxiliary subproblem at each iteration. The gradient estimator is based on SAGA, a popular variance-reduced gradient estimator. We first derive a new variance bound for it, showing that SAGA can exploit functional similarity. We then introduce the Recursive-Gradient technique as a general way to potentially improve the error bound of a given conditionally unbiased gradient estimator, including both SAGA and SVRG. By applying this technique to SAGA, we obtain a new estimator, RG-SAGA, which has an improved error bound compared to the original one. 
  
\end{abstract}

\makeatletter
\let\orig@addcontentsline\addcontentsline
\renewcommand{\addcontentsline}[3]{}
\makeatother

\commenter{xiaowen}{blue}

\section{Introduction}
\textbf{Motivation.}
Federated Learning (FL) is a distributed training paradigm in which a central server coordinates model updates across multiple remote clients---such as mobile devices or hospitals---without requiring access to their local data~\cite{fedavg,kairouz2021advances}. This framework enables collaborative learning on decentralized data, but introduces new algorithmic challenges due to the distributed nature of optimization.

A key issue in FL is the high cost of communication between the clients and the server. Clients may be intermittently available~\cite{konevcny2016communication} and connected over slow or unreliable networks. These constraints make it critical to design optimization algorithms that minimize communication costs, particularly in settings with partial client participation.

Various federated optimization algorithms have been proposed to address communication efficiency, each often relying on distinct client-selection strategies.
Some methods communicate only with a randomly sampled subset of clients at each round, while others need to 
select the set of participating clients
more carefully or employ hybrid schemes that combine both strategies. While prior works~\cite{woodworth2018graph,korhonen2021towards,celgd,NIPS2013_d6ef5f7f,davies2020new,JMLR:v20:19-543} introduced a few models for federated optimization, they do not account for the varying costs of each client-selection strategy, which can in practice differ due to factors such as client reliability, device heterogeneity, and network conditions. Consequently, existing metrics such as the number of communication rounds are not entirely fair for comparing methods in such scenario.

For instance,  
optimization methods based on 
\algoname{SARAH}~\cite{sarah,page} have been shown to be communication-efficient in finding an approximate stationary point \cite{saber,svrp}. This efficiency arises from the method’s ability to exploit  dissimilarity ($\delta$) between local and global objectives. In many practical scenarios—such as statistical or semi-supervised learning~\cite{chayti2022optimization,mime,svrp}—$\delta$ is often small, leading to substantial theoretical gains in communication cost. However, \algoname{SARAH}-based methods require periodic full synchronization with all clients in order to compute full gradients. This can be impractical in real-world large-scale federated systems, where clients may be intermittently unavailable due to energy constraints, network issues, or user behavior. 

In contrast to \algoname{SARAH}, methods such as SAG~\cite{sag} and SAGA~\cite{saga} are naturally better suited to the partial participation setting in FL. These methods update the model by sampling a small subset of clients at each round and using locally stored gradients. As a result, they avoid the need for periodic full synchronization, which makes them more compatible with federated systems where only a fraction of clients may be available at any given time.  Despite this advantage, the existing communication complexity of such methods depends on the individual smoothness constant $L_{\max}$~\cite{reddi2016fast,zerosarah,scaffold}, which can be significantly larger than the dissimilarity constant~$\delta$. 
Consequently, it remains unclear whether such methods are  more communication-efficient than \algoname{SARAH}-based methods, since they rely on fundamentally different client-selection strategies with different constant dependencies. 

\paragraph{Contributions.} 

In this work,  we aim to develop optimization algorithms that are efficient in both communication and local computation in the setting where client-selection strategies incur different costs.
Our main contributions are as follows:
\begin{itemize}[leftmargin=12pt,topsep=1pt,itemsep=1pt]
    \item We propose a new model formalizing the concept of federated optimization algorithms and defining information-based notions of communication and local complexities. This model associates the non-uniform costs with different client-selection strategies, enabling fair comparisons across optimization algorithms. (Section~\ref{sec:Model})
    \item 
    Within our new model, we propose a new gradient method that achieves the best communication and local complexities among existing first-order methods for non-convex optimization. This method is based on the inexact composite gradient method (\algoname{I-CGM}) with a carefully constructed gradient estimator and a special procedure for solving auxiliary subproblem at each iteration. (Section~\ref{sec:Complexity})
    \item Specifically, we first study the convergence of \algoname{I-CGM} for arbitrary gradient estimators and present an efficient technique for solving the auxiliary subproblem.
    Our technique is based on running the classical composite gradient method locally for a random number of iterations following a geometrical distribution with a carefully chosen parameter. (Section~\ref{sec:ICGM-General})
    \item 
    We then analyze the SAGA estimator and 
    establish a new variance bound for it
    that only depends on $\delta$ without requiring individual smoothness, improving upon previous results showing that SAGA can exploit
    functional similarity. 
    We also study SVRG as another example that can be incorporated into \algoname{I-CGM}. (Section~\ref{sec:BasicExamples})
    \item 
     Finally, we introduce the \emph{Recursive-Gradient (RG) technique} as a general way to potentially improve the error bound for a given conditionally unbiased gradient estimator, including both \algoname{SAGA} and \algoname{SVRG}.
     Applying this technique to SAGA and SVRG, we obtain new RG-SAGA and RG-SVRG gradient estimators with better error bounds compared to the original ones.
     (Section~\ref{sec:RG-main})
\end{itemize}
We discuss our results in detail in the context of related work in Appendix~\ref{sec:related work} and summarize them in Table~\ref{tab:method_comparison}.

\begin{table*}[htb!]
\caption{Summary of efficiency guarantees (in BigO-notation) 
for finding an $\epsilon$-stationary point. \algoname{I-CGM-RG-SAGA} achieves the best communication and local complexities. 
For the precise description of the problem classes, notations, as well as the discussions of the methods, see~\Cref{sec:related work}.
}
\label{tab:method_comparison}
\resizebox{\linewidth}{!}
{\begin{minipage}{1.45\textwidth}
\centering
\renewcommand{\arraystretch}{1.6} %to make the spacing of the rows more even
\begin{tabular}{l|cccc} 
\toprule
\textbf{Method} & \textbf{ Communication complexity} 
& \textbf{Assumption}  
& \textbf{Local complexity} 
& \textbf{VR type} 
\\
\midrule
Centralized GD 
& $C_An_m \frac{L_f F^0}{\epsilon^2}$ 
&
FS 
& $n_m \frac{L_f F^0}{\epsilon^2}$ & None \\
\hline
FedRed~\cite{fedred} 
& $C_A n_m \frac{\Delta_1 F^0}{\epsilon^2}$ 
& \ref{assump:ED} ; \ref{assump:L1}
& $\frac{L_1 F^0}{\epsilon^2}+n_m \frac{\Delta_1 F^0}{\epsilon^2}$ & None  \\
\hline
FedAvg~\cite{fedavg} &  
$C_R\bigl( \frac{\zeta_m^2 F^0}{\epsilon^4}+\frac{\sqrt{L_{\max}}\zeta}{\epsilon^{3}}+\frac{L_{\max}F^0}{\epsilon^2} \bigr)$
&
IS 
, BGD 
&
$\frac{\zeta_m^2 F^0}{\epsilon^4}+\frac{\sqrt{L_{\max}}\zeta}{\epsilon^{3}}+\frac{L_{\max}F^0}{\epsilon^2}$
&  
None  
\\
\hline
FedDyn~\cite{feddyn} & 
$C_An_m + C_R   n_m\frac{L_{\max} F^0}{\epsilon^2}$
& IS & 
unknown 
& None \\
\hline
MimeMVR~\cite{mime} & 
$C_A \bigl( \frac{\zeta_m^2 F^0}{\epsilon^2}
+ \frac{\zeta_m \Delta_{\max} F^0}{\epsilon^{3}}
+\frac{\Delta_{\max}F^0}{\epsilon^2} \bigr)$
&
IS, BGD, SD
& $\frac{L_{\max} \zeta_m^2 F^0}{\Delta_{\max} \epsilon^2}
+ \frac{\zeta_m L_{\max} F^0}{\epsilon^{3}}
+\frac{L_{\max}F^0}{\epsilon^2}$
& None 
\\
\hline
CE-LGD~\cite{celgd} & 
$C_R\bigl( \frac{\zeta^2_m F^0}{\epsilon^2}
+ \frac{\zeta_m \Delta_{\max} F^0}{\sqrt{m}\epsilon^{3}}
+\frac{\Delta_{\max}F^0}{\epsilon^2} \bigr)$
& IS, BGD, SD& 
$\frac{L_{\max}\zeta_m^2 F^0}{\Delta_{\max} \epsilon^2}
+ \frac{\zeta_m L_{\max} F^0}{\sqrt{m}\epsilon^{3}}
+\frac{L_{\max}F^0}{\epsilon^2}$
& None \\
\hline
Scaffold~\cite{scaffold} & 
$C_An_m + C_A \frac{n_m^{2/3} L_{\max}F^0}{\epsilon^2}$
& IS & $n_m + \frac{n_m^{2/3} L_{\max}F^0}{\epsilon^2}$ & 
SAG~\cite{sag} \\
\hline
SABER-full~\cite{saber} & 
$C_An_m + C_A\frac{(\Delta_{\max} + \sqrt{n_m}\delta_m)F^0}{ \epsilon^2}$
& SD & unknown & PAGE~\cite{page}
\\
\hline
SABER-partial~\cite{saber} & 
$C_A n_m + C_R \frac{\zeta^2_m}{\epsilon^2}  \frac{\Delta_{\max}F^0}{\epsilon^2}$ 
& SD, BGD & unknown & SARAH~\cite{sarah}\\
\hline
\textbf{I-CGM-RG-SVRG (ours)} & 
$C_A n_m + \frac{(C_R\Delta_1+ \sqrt{C_A C_R n_m} \delta_m) F^0}{\epsilon^2}$
& \ref{assump:ED},\ref{assump:SOD} ; \ref{assump:L1} & $n_m
+
\frac{(L_1 + \Delta_1 + \sqrt{\frac{C_A}{C_R}n_m} \delta_m)F^0}{\epsilon^2}$ & RG-SVRG~\cite{loopless-svrg}  \\
\hline
\textbf{I-CGM-RG-SAGA (ours)} & $C_A n_m+C_R \frac{(\Delta_1+ \sqrt{n_m} \delta_m) F^0}{\epsilon^2}$
& \ref{assump:ED},\ref{assump:SOD}
; \ref{assump:L1}  & 
$n_m
+
\frac{(L_1 + \Delta_1 + \sqrt{n_m} \delta_m)F^0}{\epsilon^2}$ & RG-SAGA~\cite{saga} \\
\bottomrule
\end{tabular}
\end{minipage}}
\end{table*}

\section{Problem Formulation}
\label{sec:Problem}
We consider the following federated optimization problem: 
\begin{equation}
\min_{\xx \in \R^d} \Bigl\{ f(\xx) \defeq \frac{1}{n}\sum_{i=1}^n f_i(\xx) \Bigr\},
\label{eq:problem}
\end{equation}
where each $f_i \colon \R^d \to \R$ is a differentiable function which can be directly accessed only by client $i$.

\textbf{Notation.} 
We abbreviate $[n]\defeq\{1,2,\ldots,n\}$. 
For a finite set $A$ and an integer $1 \leq m \le |A|$,  
$\binom{A}{m}$ denotes the power set comprised of all $m$-element subsets of $A$. $\norm{\cdot}$ denotes the standard Euclidean norm in $\R^d$. 
We use $\E[\cdot]$ to denote the standard (full) expectation. We write $\E_{\xi}[\cdot]$ for the expectation taken w.r.t. $\xi$.
We assume that the objective function in problem~\eqref{eq:problem} is bounded from below and denote its infimum by $f^\star$. We denote $F^0 \defeq f(\xx^0) - f^\star$ where $\xx^0$ is the initial point.

\subsection{Federated Optimization Algorithms and their Complexity}
\label{sec:Model}

\textbf{Federated Optimization Algorithm.}
We consider the standard federated optimization setting with a central server and $n$ clients. 
The server is the main entity that implements the optimization algorithm, but cannot directly access any of the local functions $(f_i)_{i=1}^n$.
Instead, it interacts with problem~\eqref{eq:problem} through communications with the clients, allowing certain information to be exchanged between them. 
Each client $i \in [n]$ has access to the information provided by the server and can interact with its own local function $f_i$, but only through the oracle $\OO_{f_i}$.
An oracle is a standard notion in optimization~\cite{nemirovskij1994complexity,nemirovskij1983problem}, which is a procedure that takes as input a point and returns certain information about the function at this point. The most commonly used oracle is the first-order oracle, which returns the function value and its gradient.
In general, an oracle can be stochastic; however, in this work, we mainly consider the standard deterministic first-order oracle $\OO_{\FO_i}(\xx) \defeq (f_i(\xx), \nabla f_i(\xx))$.
Throughout the paper, we assume that the server can communicate with up to $m \in [n]$ clients simultaneously in parallel. 
We formalize optimization algorithms in this setting as follows.

Given the oracles $\OO_{f_1},\ldots,\OO_{f_n}$, a \emph{federated optimization algorithm} for a problem class $\mathcal{F}$ is a procedure that proceeds across \emph{communication rounds}. 
A \emph{problem class} is the collection of all problems of  form~\eqref{eq:problem} satisfying certain assumptions. (We will introduce a specific problem class considered in this work in Section~\ref{sec:ProblemClass}.)  
At the beginning, the server and each client $i \in [n]$ initialize the empty information sets $\cI^0$ and $\cH_{i}^0$, respectively. At each round $r \ge 0$, the server chooses 
a subset of clients $S_r \subseteq [n]$ with at most $m$ elements (to be discussed later). The server then communicates with the clients in $S_r$, providing each client $i \in S_r$ with certain information $\bar{\cI}_i^r$ constructed from the server's information set $\cI^r$. Then it specifies a certain method $\cM_{i}^r$ (often called a \emph{local method}) for each client $i \in S_r$ to run locally. The method $\cM_i^r$ starts with the initial information $(\Bar{\cI}_i^r, \cH_i^r)$, and iteratively queries the oracle $\OO_{f_i}$, obtaining a response $\cR_i^r$, which is then sent back to the server. (The details of this procedure are discussed in the next paragraph.)
The server collects the output responses and updates the information set $\cI^{r+1}=(\cI^r, ( \cR_i^r)_{i \in S_r})$. At each round $r \ge 0$, the server also performs a termination test based on the current information set $\cI^r$. After the algorithm terminates at a certain round $R \ge 0$, the server then constructs and outputs an approximate solution $\hat{\xx}^R$ to problem~\eqref{eq:problem} based on $\cI^R$ using a certain rule specified by the algorithm. To summarize, a federated optimization algorithm is a collection of rules prescribing what to do at each communication round $r$: 1) how to select clients, 2) how to compute the information $\bar{\cI}^r_i$ that is sent to each selected client, 3) which local method each selected client runs, 4) when to terminate, and 5) how to form the approximate solution. We allow each of these rules to be randomized. See Figure~\ref{fig:Algorithm} for an illustration summarizing the procedures described above. 

\begin{figure}[tb!]
   \centering
    \includegraphics[width=1.0\linewidth]{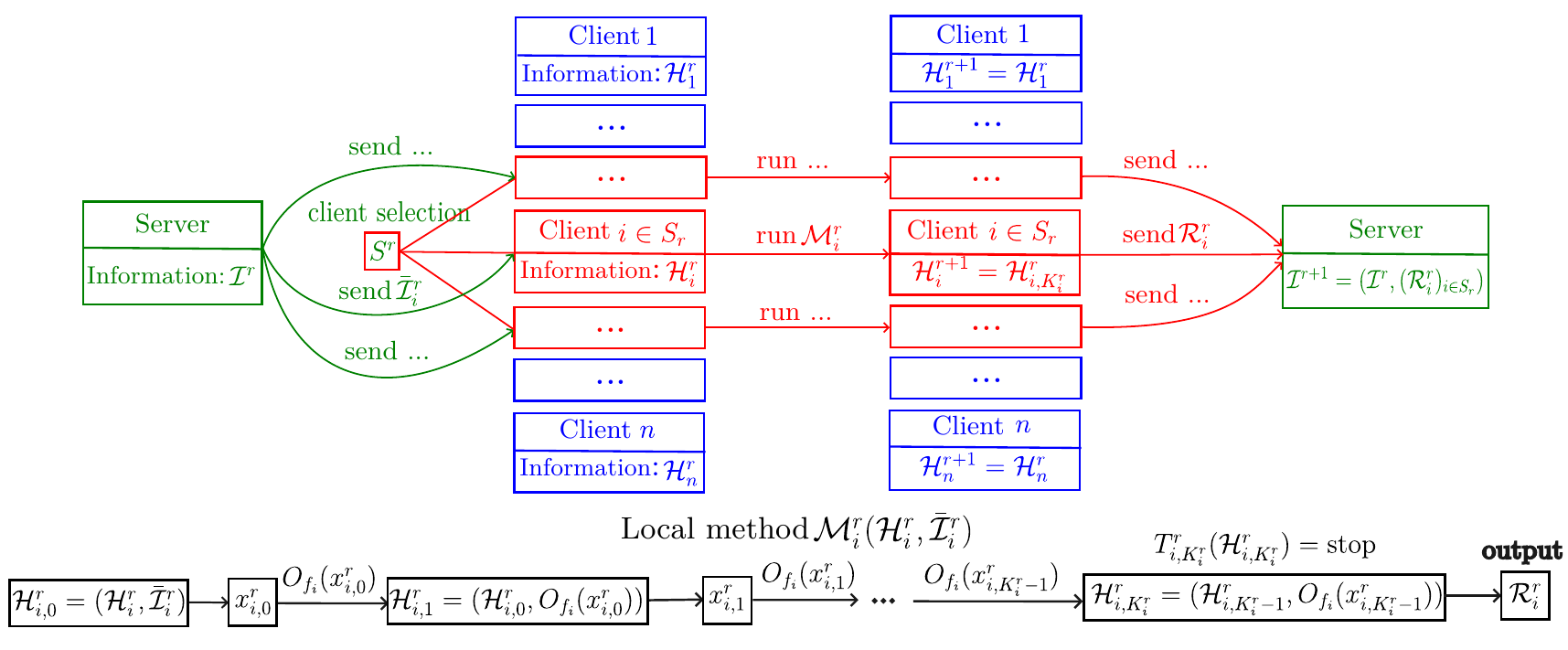}
    \vskip-2mm
    \caption{Illustration of the sequence of procedures performed by a federated optimization algorithm 
    at each communication round $r$. Each client $i \in S_r$ can make different number of local steps.}
    \label{fig:Algorithm}
\end{figure}

At the beginning of each round $r$, each selected client $i \in S_r$ receives the information $\Bar{\cI}_i^r$ from the server. Using this new information, 
it enriches its information set 
$\cH_{i,0}^{r} \defeq (\cH_i^r, \Bar{\cI}_i^r)$ and runs the specified method $\cM_{i}^r$. At each step $k \ge 0$, this method first computes a point $\xx_{i,k}^r$ based on $\cH_{i,k}^r$, queries the oracle at this point, and then updates its local information: $\cH_{i,k+1}^{r} = (\cH_{i,k}^r, \OO_{f_i}(\xx_{i,k}^r))$. At the beginning of each step $k$, the method also performs a termination test $T_{i,k}^r(\cH_{i,k}^r)$; 
Once this test is satisfied at a certain step $K_{i}^r$, the method terminates and constructs the output 
$\cR_{i}^r$ from the final information $\cH_{i, K_i^r}^{r}$. The information sets are then updated as $\cH_{i}^{r+1}\defeq 
\cH_{i,K_i^r}^r$, and remain the same ($\cH_{i}^{r+1} \defeq \cH_i^r$) for each non-selected client $i \notin S_r$.
To summarize, a local method $\cM_i^r$ is a collection of $3$ rules: 1) how to compute the next point at each step, 2) 
when to terminate, and 3) how to form the result.
We allow each of these rules to be randomized (resulting in a \emph{randomized} local method); if all the rules are deterministic, the local method is called \emph{deterministic}.

Note that the above definition of a distributed optimization algorithm is rather general and only constrains how the algorithm accesses information about the optimization problem. In particular, we do not impose any restrictions on the arithmetic or memory complexity of each step of the algorithm, nor on the size of the data transmitted between the server and the clients. This general definition is sufficient to introduce the two \emph{information-based} notions of complexity that we focus on in this work: communication and local complexities (defined below). In practice, however, both memory storage and information usage should be implemented efficiently. Typically, the accumulated information sets maintained by the clients and the server, as well as the information exchanged between them, are simply a collection of a few vectors and scalars.

\textbf{Client-Selection Strategies.}
We next introduce three commonly used client-selection strategies and  associate them with different costs. The distinction among them lies in how the set $S_r$ is selected.
\begin{itemize}
[leftmargin=12pt,topsep=1pt,itemsep=1pt]
    \item \emph{Arbitrary Client Selection Strategy} ($\ASS$):
    The set $S_r$ can be chosen in any way
    from $\binom{[n]}{m}$. We define the cost of this operation as $C_A$. 
    \item \emph{Random Client Selection Strategy} ($\RSS$):
    The set $S_r$ is sampled uniformly at random from 
    $\binom{[n]}{m}$. We define the cost of this operation as $C_R$. 
    \item 
    \emph{Delegated Client Selection Strategy} ($\DSS$):
    The set $S_r$ is chosen to be $S_D$, where $S_D$ is a fixed set of so-called delegate clients (to be discussed later) with $|S_D| \leq m$.
    We define the cost of this operation as $1$.
\end{itemize}
Clearly, $\ASS$ is the most powerful among the three strategies, as the other two could be easily implemented in terms of $\ASS$. Further, this strategy allows
the server to collect information from any subset of clients. This flexibility enables the implementation of \emph{full synchronization}, where the algorithm needs certain information to be collected from all clients. This feature appears in many algorithms, the most basic example being the usual gradient descent (GD).
Specifically, if an algorithm requires computing the full gradient $\nabla f(\xx)$ at a point $\xx$, the server can 
split $[n]$ into $m$ disjoint sets and repeatedly use $\ASS$ to make $\lceil \frac{n}{m} \rceil$ sequential communications with each set of clients, sending to each client the point $\xx$ and asking it to compute and return the gradient $\nabla f_i(\xx)$ (this corresponds to the simplest one-step local first-order method $\cM_i^r$). 

However, when clients are unreliable or slow to respond, using $\ASS$ can become costly.
In cross-device settings, it is often more efficient to communicate only with a randomly sampled subset of clients at each communication round—a strategy commonly known as partial client participation~\cite{fedavg}, which is modeled by the $\RSS$. 
Therefore, we treat $\RSS$ as a cheaper strategy compared with $\ASS$.
Unlike $\ASS$, the full-gradient computation cannot be directly implemented with $\RSS$. (But it can be recovered with high probability by using $\RSS$ multiple times~\cite{arjevani2020complexity}.)

In addition to the previous two strategies, there are scenarios where there exist so-called \emph{delegate} clients that are always reliable and efficient both in communication and performing local computations. Sometimes, it is sufficient—or even preferable—to interact with these clients.  With $\DSS$, the server can always query information about the specific functions in the delegate set. 
In this work, we focus on the 
setting where there is one delegate client (number 1), i.e., $S_D = \{1\}$. 

Based on the properties of each strategy discussed above, we assume that the above costs satisfy the following natural relations:
\[
\boxed{
1 \le C_R \le C_A 
} \;.
\]

\textbf{Communication-and Local Complexities.} Consider a federated optimization algorithm $\cA$ for solving a problem $f$ from the problem class $\mathcal{F}$. Let $R$ be the (possibly random) number of communication rounds made by $\cA$ on $f$ and let $\hat{\xx}^R$ be the corresponding output of $\cA$. We define the \emph{accuracy} of the algorithm $\cA$ at the problem $f$ as:
\[
\boxed{
\operatorname{Accur}(\cA, f) = \E[ \| \nabla f(\hat{\xx}^R) \|^2] 
}\;.
\]
Further, let $N_A$, $N_R$, and $N_D$ be the (possibly random) total number of times that the client-selection strategies $\ASS$, $\RSS$, and $\DSS$ are used by $\cA$ during the $R$ communication rounds, respectively. 
We define the \emph{communication complexity} of $\cA$ on $f$ as: 
\[
\boxed{
N_f = \E[C_A N_{A} + C_R N_{R} + N_{D} ]}
\;,
\] 
and the \emph{local complexity} of $\cA$ on $f$ as:
\[
\boxed{K_f =  \E\Biggl[ \sum_{r=0}^{R-1} K_r \Biggr]}  \;,
\]
where $K_r \defeq \max_{i \in S_r} K_i^r $ and $K_i^r \ge 0$ is the number of queries to the oracle $\OO_{f_i}$ by the client $i$ at round $r$. We next define the worst-case complexities of $\cA$ over the entire problem class $\cF$.
The communication complexity of $\cA$ for solving the problem class $\cF$ up to $\epsilon$ accuracy is defined as:
\[
\boxed{
N_{\cF}(\epsilon) = \sup_{f \in \cF}\bigl\{
N_f | \operatorname{Accur}(\cA, f) \le \epsilon^2 
\bigr\}
}\;,
\] 
and the corresponding local complexity is defined as:
\[
\boxed{
K_{\cF}(\epsilon) = \sup_{f \in \cF} \bigl\{
K_f | \operatorname{Accur}(\cA, f) \le \epsilon^2 
\bigr\}
} \;.
\]
If there exists some $f \in \cF$ such that $\cA$ fails to reach $\operatorname{Accur}(\cA,f) \le \epsilon^2$, then both complexities $N_{\cF}(\epsilon)$ and $K_{\cF}(\epsilon)$ are defined as $+\infty$. 

After fixing the desired accuracy $\epsilon$, 
we consider only federated optimization algorithms that can achieve $\operatorname{Accur}(\cA,f) \le \epsilon^2$ for all $f \in \cF$. Among these algorithms, we say that the one with smaller communication complexity $N_{\cF}(\epsilon)$ is more efficient. If two algorithms have the same communication complexity, the one with lower local complexity $K_{\cF}(\epsilon)$ is preferable.

\subsection{Problem Class}
\label{sec:ProblemClass}
 We study optimization problem~\eqref{eq:problem} in which the client objectives exhibit an underlying similarity structure. Specifically, we use the following two assumptions that relax standard smoothness assumptions. The first quantifies the deviation between the delegate function $f_1$ and $f$. For an index $i \in [n]$, we use $h_i \defeq f - f_i$ to denote the difference function.
 
\begin{assumption}
\label{assump:ED}
There exists $\Delta_1 > 0$ such that
for any $\xx, \yy \in \R^d$, we have:
\begin{equation}
    \norm{ \nabla h_1(\xx) - \nabla h_1(\yy) } 
    \le 
    \Delta_1 \norm{ \xx - \yy } .
    \label{eq:ED}
\end{equation}
\end{assumption}

Alternatively, one may define a uniform dissimilarity constant $\Delta_{\max}$~\cite{scaffold,fedred} such that for any $i \in [n]$, it holds that $\norm{\nabla h_i (\xx) - \nabla h_i (\yy)} \le \Delta_{\max} \norm{\xx - \yy}$. 
In this work, we focus on $\Delta_1$ since it can be much smaller than $\Delta_{\max}$.

The second assumption characterizes the average dissimilarity among all local functions.

\begin{assumption}[\cite{svrp,s-dane,AccSVRS,fedred,takezawa2025exploiting}]
\label{assump:SOD}%
There exists $\delta > 0$ such that for any $\xx, \yy \in \R^d$, we have:
\begin{equation}
    \Avg
    \norm{\nabla h_i(\xx) - \nabla h_i(\yy)}^2 
    \le 
    \delta^2 \norm{\xx - \yy}^2 \;.
    \label{eq:delta}
\end{equation}%
%where 
\vspace{-1.4\baselineskip}
\end{assumption}
The left-hand side of~\eqref{eq:delta} is equal to $\Avg \| \nabla f_i(\xx) - \nabla f_i(\yy)\|^2 - \|\nabla f(\xx) - \nabla f(\yy)\|^2$, which can be 
interpreted as the variance of $\nabla f_i(x) - \nabla f_i(y)$ where $i$ is selected uniformly at random.
 If each $f_i$ has $L_{\max}$-Lipschitz gradient, then we have 
$\Delta_1 \le 2L_{\max}$ and $\delta \le L_{\max}$. 
Therefore, both conditions are weaker than assuming each $f_i$ is Lipschitz-smooth. 
We refer to discussions in~\cite{s-dane} for more properties and details.

The previous two quantities $\delta$ and $\Delta_1$ will only affect the communication complexity of our algorithms, while the local complexity additionally depends on $L_1$ which is defined as follows.

\begin{assumption}
\label{assump:L1}
There exists $L_1 > 0$ such that for any $\xx, \yy \in \R^d$, we have:
\begin{equation}
    \|\nabla f_1(\xx) - \nabla f_1(\yy)\| \le L_1 \|\xx - \yy\| \;.
    \label{eq:L1-smooth}
\end{equation}%
%where 
\vspace{-1.4\baselineskip}
\end{assumption}

\section{Inexact Composite Gradient Method} 
\label{sec:ICGM-General}

\subsection{Inexact Composite Gradient Method}
\renewcommand{\algorithmicendfor}{}
  
We first introduce the Inexact Composite Gradient Method (\algoname{I-CGM}), which serves as the backbone of our approach. 
Consider the composite reformulation of the problem~\ref{eq:problem}: $f = f_1 + [f - f_1] = f_1 + h_1$.
Let $\lambda > 0$ and $\xx^0 \in \R^d$ be the initial point. 
At each iteration $t \ge 0$, \algoname{I-CGM} computes an approximation of the gradient $\gg^t \approx \nabla f(\xx^t)$ and defines the next iterate as: 
\begin{equation}
\boxed{
    \xx^{t+1}
    \approx
    \argmin_{\xx \in \R^d}
    \Bigl\{ F_{t}(\xx) := f_1(\xx) + h_1(\xx^t) + \langle \gg^t - \nabla f_1(\xx^t), \xx - \xx^t \rangle 
    + \frac{\lambda}{2}\| \xx-\xx^t \|^2 \Bigr\}, 
    \tag{I-CGM}}
    \label{Alg:PP}
\end{equation}
where both the inaccuracy in solving the subproblem and the approximation error (defined below) are assumed to be sufficiently small (to be specified later):

\begin{equation}
    \label{eq:Condition-SD}
    F_t(\xx^{t+1}) \le F_t(\xx^t),
    \quad
    e_t \defeq 
    \|\nabla F_t(\xx^{t+1})\|, 
    \quad
    \hat{\Sigma}_t^2 \defeq \norm{\gg^t-\nabla f(\xx^t)}^2 \;.
\end{equation}

In the following statement,
we provide the general convergence guarantee for \algoname{I-CGM}. 
The proof can be found in Section~\ref{sec:Proof-PP-Main} in the Appendix.

\begin{theorem} 
\label{thm:IterationCMGMain}
    Let \ref{Alg:PP} be applied to Problem~\eqref{eq:problem}.
    Suppose Assumption~\ref{assump:ED} and condition~\eqref{eq:Condition-SD} are satisfied.
    Let $\lambda > \Delta_1$.
    Then for any $T \ge 1$, we have:
    \begin{align*}
    \sum_{t=1}^T 
    &\| \nabla f(\xx^t) \|^2
    +(\lambda + \Delta_1)^2 \sum_{t=1}^T \| \xx^t - \xx^{t-1}\|^2
    \\
    &\le \frac{12(\lambda + \Delta_1)^2}{\lambda - \Delta_1} F^0 + 
    \Bigl( 
    \frac{12(\lambda + \Delta_1)^2}{(\lambda - \Delta_1)^2}
    +4
    \Bigr)
    \sum_{t=0}^{T-1} 
    \hat{\Sigma}_t^2
    + 4\sum_{t=0}^{T-1} e_t^2 
    \;.
    \end{align*}
\end{theorem}
We see that each subproblem can be solved inexactly without affecting the convergence rate (up to absolute constants), provided that the error term $\sum_{t=0}^{T-1} e_t^2$ is of the same order as the first two terms on the right-hand side. Moreover, if the approximation errors $\sum_{t=0}^{T-1} \hat{\Sigma}_t^2$ can also be bounded by
the first two terms on the left-hand side, then the convergence of the gradient norm is guaranteed. 
If there exists randomness either in solving the subproblems or in constructing the estimators, then these conditions are required to hold in expectation. Specifically, we obtain the following corollary.
\begin{corollary} 
\label{thm:CGM-Main-Corollary}
    Following the same settings as in Theorem~\ref{thm:IterationCMGMain}. If the inaccuracies in solving the subproblems satisfy:
    \begin{equation}
    \label{eq:AccuracyCondition}
    F_t(\xx^{t+1}) \le F_t(\xx^t),
    \quad 
    \sum_{t=0}^{T-1}\E[e_t^2]
    \le 
    \frac{(\lambda + \Delta_1)^2}{\lambda - \Delta_1} F^0 + \sum_{t=0}^{T-1} \Sigma_t^2
    \;,
    \end{equation}
    and the approximation errors satisfy:
    \begin{equation}
        \Bigl( 
    \frac{12(\lambda + \Delta_1)^2}{(\lambda - \Delta_1)^2}
    +8
    \Bigr)
    \sum_{t=0}^{T-1} \Sigma_t^2
    \le 
    \frac{1}{2}\sum_{t=1}^T G_t^2 
    + (\lambda + \Delta_1)^2 
    \sum_{t=1}^T \chi_t^2 \;,
    \label{eq:EstErrorCondition}
    \end{equation}
    then for any $T \ge 1$, we have:
    \[
    \E[\|\nabla f(\bar{\xx}^T)\|^2]   
    \le \frac{32(\lambda + \Delta_1)^2}{\lambda - \Delta_1} \frac{F^0}{T}
    \;.
    \]
    where 
    $G_t^2 \defeq \E[\|\nabla f(\xx^t) \|^2]$,
    $
    \chi_t^2 \defeq \E[\|\xx^t - \xx^{t-1}\|^2]
    $,  
    $\Sigma_t^2 \defeq \E[\| \gg^t - \nabla f(\xx^t) \|^2]$,
    and $\Bar{\xx}^T$ is  uniformly sampled from $(\xx^t)_{t=1}^T$.
\end{corollary}

When $\gg^t$ is the exact gradient $\nabla f(\xx^t)$ for all $t \ge 0$, then \ref{Alg:PP} is reduced to \algoname{CGM} that is widely used for solving Problem~\eqref{eq:problem}, particularly because of its ability to exploit functional similarity and reduce communication costs~\cite{spag,fedred,AccSVRS,svrp,s-dane,saber,grad-sliding}.
Indeed, 
if $\lambda \simeq \Delta_1$ and the accuracy condition~\eqref{eq:AccuracyCondition} is satisfied,
then $\E[\| \nabla f(\bar{\xx}^T) \|^2] \leq \epsilon^2$ after 
$T = \cO ( \frac{\Delta_1 F^0}{\epsilon^2} )$ iterations.
In contrast, the iteration complexity of Gradient Descent depends on $L_f$ which can be larger than $\Delta_1$~\eqref{assump:ED} when $f_1$ is similar to $f$. However,
\algoname{CGM} has
sub-optimal communication complexity in terms of $n$. 
Indeed, let us assume, for simplicity, that $m=1$. Then each iteration involves: 1) computing the full gradient $\nabla f(\xx^t)$, which requires $n$ sequential communication rounds using $\ASS$, and 2) an additional round using $\DSS$ for solving the subproblem. Consequently, the total number of communication rounds with $\ASS$ and $\DSS$ is $N_A = nT$ and $N_D = T$, respectively.
The communication complexity of \algoname{CGM} is thus:
$C_A N_A + N_D =
C_A nT + T= \cO(C_An T)
=
\cO(C_A\frac{n\Delta_1 F^0}{\epsilon^2} )$. This linear dependency on $n$ can be prohibitive in large-scale federated learning settings and is worse than the complexity of stochastic methods such as \algoname{ProxSARAH}~\cite{proxsarah}, 
\algoname{SpiderBoost}~\cite{spiderboost}, and \algoname{PAGE}~\cite{page}, each of them achieving:
$\cO(
C_A\frac{\sqrt{n}\Bar{L} F^0}{\epsilon^2} )$, although they rely on a slightly different assumption of 
average smoothness~\footnote{$\forall \xx,\yy \in \R^d$, it holds that $\Avg\norm{\nabla f_i(\xx) - \nabla f_i(\yy)}^2 \le \Bar{L}^2 \norm{\xx - \yy}^2$ and we have $\delta \le \Bar{L}$.}. 
Moreover, the dependence on $\frac{C_A}{\epsilon^2}$ can become significantly large in scenarios where using $\ASS$ is costly.

\subsection{Solving Auxiliary Subproblems}
\label{sec:localGD-main}
In this section, we 
assume that $f_1$ is $L_1$-smooth and
study how to achieve the accuracy condition~\eqref{eq:AccuracyCondition}. 
Recall that each subproblem $F_t$ consists of a smooth function $\phi (\xx) = f_1(\xx)$ and a quadratic regularizer $\psi_t(\xx) = \lin{\gg^t - \nabla f_1(\xx^t), \xx - \xx^t} + \frac{\lambda}{2} \norm{\xx - \xx^t}^2$.
Let us solve it using the standard
composite gradient method (\algoname{CGM}), which proceeds as follows: For $k=0,1,...,K_t-1$,
\begin{equation}
\boxed{
\begin{split}
\yy_{k+1}^t 
&= \argmin_{\yy \in \R^d}
\Bigl\{
\phi(\yy_k^t) +\lin{\nabla \phi(\yy_k^t), \yy - \yy_k^t} + \frac{L_1}{2}\norm{\yy - \yy_k^t}^2 + \psi_t(\yy) 
\Bigr\}
\\
&= \frac{1}{\lambda + L_1}
\bigl(
L_1 \yy_k^t + \lambda \xx^t + 
\nabla f_1(\xx^t) - \gg^t
- \nabla f_1(\yy_k^t) \bigr) 
\;.
\end{split}
} 
\label{Alg:LocalGD} 
\end{equation}
Each \algoname{CGM} step monotonically decreases the function value of $F_t$ (see Lemma~\ref{thm:CGM}). Therefore, we can initialize $\yy_0^t = \xx^t$ and choose $\xx^{t+1}$ to be a certain iterate of $(\yy_k)_{k=0}^{K}$.
Then the condition on $F_t(\xx^{t+1}) \le F_t(\xx^t)$ is satisfied. We next study the number of local steps $K_t$ required to achieve the second inequality in condition~\eqref{eq:AccuracyCondition}.

\subsubsection{Fixed Number of Local Steps}
Let $K_t \equiv K \ge 1$ be a constant number and let $\xx^{t+1}$ be the iterate with the minimum gradient norm of $F_t$ among $\{\yy_k^t\}_{k=1}^K$. We use the notation $\xx^{t+1} = \operatorname{CGM}_{\operatorname{const}}(\lambda,K,\xx^t,\gg^t)$ for this process. 

The goal is to upper bound $\sum_{t=0}^{T-1} e_t^2$ where 
$e_t \defeq \|\nabla F_t(\xx^{t+1})\|$.
For each $t \ge 0$, we have: $
    e_t^2 
    \lesssim \frac{L_1 ( F_t(\yy_0^t) - F_t(\yy_K^t) )}{K}
    \lesssim
    \frac{L_1(f(\xx^t) - f(\yy_K^t) + \frac{1}{\lambda}\hat{\Sigma}_t^2}{K}
    $ (see Lemma~\ref{thm:CGM} and~\ref{thm:FrDifference}).
However, since $\yy_K^t$ and $\xx^t$ are not necessarily the same, we cannot telescope $f(\xx^t) - f(\yy_K^t)$ when we sum up $e_t^2$. Instead, the ''best'' we can do is to upper bound $f(\xx^t) - f(\yy_K^t)$ by $f(\xx^t) - f^\star$. Then by further upper-bounding the summation of $\sum_{t=0}^{T-1} [f(\xx^t) - f^\star]$ in terms of $F^0$ and $\hat{\Sigma}_t^2$, it can be shown that we need $K \simeq \frac{L_1 T}{\lambda}$ local steps to achieve the desired accuracy condition~\eqref{eq:AccuracyCondition}. The proof can be found in Section~\ref{sec:localstep-const}.
\begin{lemma}
\label{thm:UpperboundKr-const}
Consider~\ref{Alg:PP} with $\xx^{t+1} = \operatorname{CGM}_{\operatorname{const}}(\lambda, K,\xx^t,\gg^t)$ under Assumption~\ref{assump:ED} and~\ref{assump:L1}.
Let $T \ge 1$ be the fixed number in condition~\eqref{eq:AccuracyCondition}.
Then by choosing $\lambda > \Delta_1$ and $K = K_T \defeq \lceil \frac{8 L_1 T}{\lambda - \Delta_1} \rceil$,
the accuracy condition~\eqref{eq:AccuracyCondition} is satisfied. 
\end{lemma}

\subsubsection{Random Number of Local Steps.}
We now allow the number of local steps 
$K_t$ to follow a geometric distribution—a common technique used to derive last-iterate recurrences~\cite{katyushaX}. When applied to solve the subproblems in~\ref{Alg:PP}, this approach yields an algorithm that is efficient in local computation. 

Let us consider \algoname{CGM}~\eqref{Alg:LocalGD} 
with $K_t = \hat{K}_t + 1$ iterations 
where $\hat{K}_t \sim \operatorname{Geom}(p)$, that is $\mathbb{P}(\hat{K}_t=k)=(1-p)^kp$ for each $k \in \{0,1,2,...\}$. The solution is set to be $\xx^{t+1} = \yy_{K_t}$. We use the notation $\xx^{t+1} = \operatorname{CGM}_{\operatorname{rand}}(\lambda,\hat{K}_t,\xx^t,\gg^t)$ for this process. 

In contrast to the convergence rate of using a deterministic $K$, 
we can now show that
$
    \E_{\hat{K}_t}[e_t^2]   
    \lesssim
    \frac{L_{1}^2 p}{L_{1}+\lambda} 
    \E_{\hat{K}_t}[F_t(\xx^t) - F_t(\xx^{t+1})] 
$. Using 
$
    \E_{\hat{K}_t}[F_t(\xx^t) - F_t(\xx^{t+1})] \lesssim \E_{\hat{K}_t}[ f(\xx^t) - f(\xx^{t+1})+
    \frac{1}{\lambda} \hat{\Sigma}_t^2] 
$, we get the telescoping term 
$\E[ f(\xx^t) - f(\xx^{t+1})]$ after passing to the full expectation, which allows to improve the total amount of local computations. 

\begin{lemma}
\label{thm:UpperboundKr-random}
Consider~\ref{Alg:PP} with $\xx^{t+1} = \operatorname{CGM}_{\operatorname{rand}}(\lambda, \hat{K}_t,\xx^t,\gg^t)$ where $\hat{K}_t \sim \operatorname{Geom}(p)$ under Assumption~\ref{assump:ED} and~\ref{assump:L1}. 
Let $T \ge 1$ be the fixed number in condition~\eqref{eq:AccuracyCondition}. Then by choosing
$\lambda > \Delta_1$
and
$p = \frac{\lambda - \Delta_1}{8(L_1 + \lambda)} < 1$, the accuracy condition~\eqref{eq:AccuracyCondition} is satisfied. 
\end{lemma}
    
To achieve the accuracy condition~\eqref{eq:AccuracyCondition}, the number of local first-order oracle queries required by using the random $\hat{K}_t$ at each iteration $t$ in expectation is $\E_{\hat{K}_t}[K_t] = \frac{1}{p} \simeq \frac{L_1}{\lambda}$, which improves upon the previous result of $\frac{L_1 T}{\lambda}$ obtained  by using a fixed number of $K$.

So far, we have studied how to solve the subproblems of \algoname{I-CGM} such that the accuracy condition~\eqref{eq:AccuracyCondition} is satisfied. 
We now turn to constructing the gradient estimator $\gg^t$ that has the desired approximation error~\eqref{eq:EstErrorCondition}. Meanwhile, we aim to improve both the dependence on $n$ and $C_A$ in the communication complexity of \algoname{CGM}. The main strategy is to design a gradient estimator 
whose approximation error depends only on the similarity constant $\delta$ while  avoiding periodic full synchronizations.  

\section{Basic Application Examples: SAGA + SVRG}
\label{sec:BasicExamples}
In this section, we present two algorithms that maintain an approximation of the gradient, $\mG^t \approx \nabla f(\xx_t)$ for $t \ge 0$. 
Each algorithm starts with an initial point $\xx^0$. Then at each iteration $t \ge 0$, $\mG^t$ is computed first, after which the next iterate $\xx^{t+1}$ is computed. In what follows, for a set $S \in \binom{[n]}{m}$ and $m \in [n]$, we use $f_S \defeq \frac{1}{m} \sum_{i \in S} f_i$ to denote the average function over this set.

For convenience of presentation, we use the following notations throughout the rest of the paper: 
\begin{equation}
n_m \defeq \frac{n}{m}, 
\quad
q_m \defeq \frac{n-m}{n-1}, 
\quad
\text{and} 
\quad 
\delta_m^2 \defeq \frac{q_m}{m}\delta^2 \;.
\end{equation}
\subsection{SAGA Estimator}
\algoname{SAGA} estimator is a variance-reduction technique based on incremental gradient updates, originally designed for centralized finite-sum minimization~\cite{saga}. In this section, we adapt this estimator to the federated optimization scenario and study its properties. 

The \algoname{SAGA} estimator defines:
\begin{equation}
\boxed{
\begin{split}
    \mG^0 = \nabla f(\xx^0),
    \quad
    \mG^1 = \nabla f(\xx^1),
    \quad
    \mG^t = \bb_{S_t}^t 
    - \bb_{S_t}^{t-1}
    + \bb^{t-1},
    \;
    t \ge 2
    \;,
\end{split}
}
\tag{SAGA}
\label{Alg:SAGA-update}
\end{equation}
where $S_{t} \in \binom{[n]}{m}$ is uniformly sampled  at random without replacement, 
$\bb_{S_t}^t \defeq \frac{1}{m}\sum_{i \in S_t}\bb_i^t$,
$\bb_{S_t}^{t-1} \defeq \frac{1}{m}\sum_{i \in S_t}\bb_i^{t-1}$,
$\bb^t \defeq \Avg \bb_i^t$,
and for any $i \in [n]$, $\bb_i^t$ is recurrently defined as: 
\[
\bb_i^0 = \nabla f_i(\xx^0),
\quad
\bb_i^1 = \nabla f_i(\xx^1),
\quad 
\bb_{i}^t =
   \begin{cases}
    \nabla f_i(\xx^t) & \text{if } i \in S_t , \\
   \bb_{i}^{t-1} & \text{otherwise},
   \end{cases},
   \quad 
   t \ge 2 \;.
\]

We have the following recurrence for $\bb^t$ (the derivation can be found in Lemma~\ref{thm:bt}):
\begin{equation}
\label{eq:br}
\bb^{t} = \bb^{t-1} + \frac{1}{n_m}[ \nabla f_{S_t}(\xx^t)- \bb_{S_t}^{t-1}],
\quad t \ge 2
\;.
\end{equation}

\textbf{Implementation.} 
At the beginning, when $t = 0$ and $1$, each client $i = 1, \ldots, n$ computes $\nabla f_i(\xx^t)$ and initializes $\bb_i^t$ and sends the result to the server; the server then aggregates these results computing $\nabla f(\xx^t)$ to initialize $\mG^t$ and $\bb^t$. This requires two full synchronizations ($2 \lceil n_m \rceil$ communications rounds using~$\ASS$). At each iteration $t \ge 2$, the server contacts the randomly selected set of clients $S_t$ using~$\RSS$ and sends $\xx^t$ to them. Each client $i \in S_t$ computes $\bb_i^t = \nabla f_i(\xx^t)$ and sends $\bb_i^{t} - \bb_i^{t-1}$ back to the server. The server then updates $\bb^t$ according to~\eqref{eq:br} and constructs the gradient estimator $\mG^t$ using the stored $\bb^{t-1}$ according to~\eqref{Alg:SAGA-update}.

Each client $i$ thus needs to store a single vector $\bb_{i}^t$. On the server side, only the aggregated vector $\bb^t$ and the iterate $\xx^t$ need to be maintained. The memory overhead of the \algoname{SAGA} estimator is thus very small in the federated learning setting,
similarly to the \algoname{SAG} estimator~\cite{sag} used in \algoname{Scaffold}~\cite{scaffold}. 

\textbf{Properties of \algoname{SAGA}.} It is not difficult to show that $\mG^t$ 
is a conditionally unbiased estimator of $\nabla f(\xx^t)$, namely, $\E_{S_t}[\mG^t] = \nabla f(\xx^t)$.
We next present a new variance bound for SAGA that is controlled by the constant $\delta$. 
The proof can be found in Section~\ref{sec:ProofVarianceSAGAMain}.

\begin{lemma}
\label{thm:VarianceSAGAMain}
    Consider the \algoname{SAGA} estimator~\eqref{Alg:SAGA-update} under Assumption~\ref{assump:SOD}.
    Then for any $t \ge 2$,
    $\E_{S_t}[\mG^t] = \nabla f(\xx^t)$ and
    for any $T \ge 1$,
    we have :
    \begin{equation*}
    \sum_{t=0}^{T} \sigma_t^2
    \le
    \frac{2n_mq_m}{m} G_1^2
    +
    \frac{n_m-1+\sqrt{n^2_m-n_m}}{(n-1)}
    \sum_{t=2}^{T-1} 
    G_t^2
    +4n_m^2\delta_m^2
    \sum_{t=2}^T
    \chi_{t}^2 \;,
    \end{equation*}
    where 
    $\sigma_t^2 \defeq \E_{S_{[t]}}[\| \mG^t - \nabla f(\xx^t) \|^2]$,
      $G_t^2 = \E_{S_{[t-1]}}[\| \nabla f(\xx^t) \|^2]$,
      $\chi_t^2 \defeq \E_{S_{[t-1]}}[\| \xx^{t} - \xx^{t-1}\|^2]$,
      and
     $S_{[t]} \defeq (S_2,...,S_t)$.
\end{lemma}
Note that this variance bound depends on 
$G_1^2,\ldots,G_{T-1}^2$ and $\chi_2^2,\ldots,\chi_T^2$, which aligns with the terms on the right-hand side of the desired error bound~\eqref{eq:EstErrorCondition}.
However, the coefficient in front of $\mG_1^2$ in this bound can be larger than $1$, whereas~\eqref{eq:EstErrorCondition} requires it to be strictly less than $1$. Consequently, 
the requirement is not met and
we cannot directly incorporate the \algoname{SAGA} estimator into~\ref{Alg:PP} by setting $\gg^t = \mG^t$. 
We will show in Section~\ref{sec:RG-main} that this error bound can be significantly improved by using the recursive gradient estimation technique.

\begin{remark}
Instead of computing the exact gradients $\nabla f(x^0)$ and $\nabla f(\xx^1)$ at the beginning which requires full synchronizations, it is possible to start with an approximation 
$\mG^0 \approx \nabla f(\xx^0)$. 
This requires only one communication round using~$\RSS$.
The resulting communication-complexity estimate will now additionally depend on the inexactness of the initial approximation but this strategy often works well in practice.
\end{remark}

\subsection{SVRG Estimator}
Another possible choice of the gradient estimator is the \algoname{SVRG} estimator~\cite{svrg}.
There are different variants of \algoname{SVRG}, and here we consider the so-called
loopless-SVRG estimator~\cite{loopless-svrg} for simplicity. 

The \algoname{SVRG} estimator defines:
\begin{equation}
\boxed{
    \mG^0 = \nabla f(\xx^0), 
    \quad
   \mG^t = \nabla f_{S_t} (\xx^t) + \nabla f(\ww^t) - \nabla f_{S_t}(\ww^t), 
   \; t \ge 1
   \;,}
   \label{Alg:SVRG-update}
   \tag{SVRG}
\end{equation}
where $S_{t} \in \binom{[n]}{m}$ is uniformly sampled  at random without replacement, 
\[
\ww^0 = \xx^0,
\quad 
\ww^{t}=
\begin{cases}
\xx^{t} & \text{if} \; \omega_t = 1, \\
\ww^{t-1} & \text{otherwise},
\end{cases}
\quad 
t \ge 1 \;,
\]
and $\omega_t$ is a Bernoulli random variable with parameter $p_B$, i.e.,
$P(\omega_t=1)=p_B \in (0,1)$.

\textbf{Implementation.}
At the beginning when $t = 0$, each client $i = 1, \ldots, n$ computes $\nabla f_i(\xx^0)$ and sends the result to the server; the server then aggregates these results, computing $\nabla f(\xx^0)$ to initialize $\mG^0$. This requires one full synchronization. At each iteration 
$t \ge 1$, the server uses $\RSS$ and sends $\ww^t$ and $\xx^t$ to the clients in $S_t$ and then receives the gradient difference $\nabla f_i(\xx^t) - \nabla f_i(\ww^t)$ from them. 
If $\omega_t = 1$, the server computes the new gradient $\nabla f(\ww^t)$ performing one full synchronization and stores it in memory; otherwise, it continues with $\nabla f(\ww^t) = \nabla f(\ww^{t - 1})$ which is already stored in memory. In total, the server needs to maintain two points $\xx^t$ and $\ww^t$
and one vector $\nabla f(\ww^t)$
and the clients are so-called stateless.

\textbf{Properties:} It is not difficult to show that the SVRG estimator $\mG^t$ 
is a conditionally unbiased estimator of $\nabla f(\xx^t)$, namely, $\E_{S_t}[\mG^t] = \nabla f(\xx^t)$. Moreover, the variance is controlled by $\delta$. 
The proof can be found in Section~\ref{sec:VarianceSVRGMain}.

\begin{lemma}
\label{thm:VarianceSVRGMain}
    Consider the \algoname{SVRG} estimator~\eqref{Alg:SVRG-update} under Assumption~\ref{assump:SOD}.
    Then for any $t \ge 1$, 
    $\E_{S_t}[\mG^t] = \nabla f(\xx^t)$ and 
    for any $T \ge 1$, we have:
    $
    \sum_{t=0}^T \sigma_{t}^2 
    \le 
    \frac{4\delta_m^2}{p_B^2} \sum_{t=1}^{T}  \chi_t^2
    $,
    where $\sigma_t^2 \defeq \E_{S_{t},\omega_{[t]}}[\norm{\mG^t - \nabla f(\xx^t)}^2]$, $\chi_t^2 \defeq \E_{\omega_{[t-1]}}[\| \xx^{t} - \xx^{t-1}\|^2]$, and $\omega_{[t]} \defeq (\omega_1,...,\omega_t)$.
\end{lemma}

We can incorporate the \ref{Alg:SVRG-update} estimator into \ref{Alg:PP} by setting $\gg^t = \mG^t$. 
This requires setting
$p_B \simeq \frac{1}{n_m}$ and
$\lambda \simeq \Delta_1 + n_m\delta_m$  
to achieve the error condition~\eqref{eq:EstErrorCondition}. The resulting communication complexity of the method is
$\cO(C_An_m + \frac{(C_R\Delta_1 + C_A n_m \delta_m)F^0}{\epsilon^2})$,
which still has a linear dependence on $n_m$.
(See Theorem~\ref{thm:I-CGM-SVRG-Main} with the proof that the reader can inspect if interested). 
Note that unbiasedness is not needed to incorporate \algoname{SVRG} directly into \algoname{I-CGM}. However, it becomes necessary later for the recursive gradient technique, which we discuss in the next section.

\section{Recursive Gradient Estimator + Examples (SAGA and SVRG)}
\label{sec:RG-main}
In this section, we present a general formular of the recursive gradient estimator that can potentially improve the error bound for a
given conditionally unbiased gradient estimator $\mG^t \approx \nabla f(\xx^t)$. Formally, we consider the following setting.

\begin{assumption}
\label{assump:unbiased-G}
For any $t \ge 0$, it holds that: 1) $S_t$ is independent of $\xx^0, \ldots, \xx^{t + 1}$, $\mG^0, \ldots, \mG^{t - 1}$; 2) $\E_{S_t}[\mG^t] = \nabla f(\xx^t)$.
\end{assumption}

The recursive gradient estimator (\algoname{RG}) defines:
\begin{equation}
    \label{Alg:RG-update}
    \tag{RG}
    \boxed{
    \gg^0 = \nabla f(\xx^0), 
    \quad 
    \gg^{t+1} 
    =(1-\beta) \gg^{t}
    + \beta \mG^{t} 
    + \nabla f_{S_{t}} (\xx^{t+1}) - \nabla f_{S_{t}}(\xx^{t}), \; 
    t \ge 0\;,}
\end{equation}
where $\beta \in (0,1]$ and $S_{t} \in \binom{[n]}{m}$ is uniformly sampled  at random without replacement.

Note that the indexing here differs from the previous ones. The algorithm starts with an initial point $\xx^0$. At each iteration $t \ge 0$, the estimator $\gg^t \approx \nabla f(\xx^t)$ is computed first and it depends only on $\mG^{t-1}$ and $S_{t-1}$.  
After that, the next iterate $\xx^{t+1}$ is computed. Therefore, $\xx^{t + 1}$ is independent from $S_t$
while previously it was dependent on it (if we use the SAGA/SVRG estimator).

Inspired by previous works, the expression of $\gg^t$ incorporates both recursive gradient update and momentum~\cite{pmlr-v258-chayti25a,yuan-polyak}. 
This expression unifies several existing methods: When $\beta = 0$, the estimator reduces to the SARAH update rule~\cite{sarah}. When $\mG^{t}$ is replaced with the \ref{Alg:SAGA-update} estimator, then $\gg^t$ recovers the structure of \algoname{ZeroSARAH}~\cite{zerosarah}. When $\mG^t = \nabla f_{S_t}(\xx^{t+1})$ and $\nabla f_{S_t}(\xx^{t+1}) - \nabla f_{S_t}(\xx^t)$ is multiplied by $1-\beta$, then it becomes STORM~\cite{storm}. In our formulation, 
$\mG^{t}$ is a general similarity-aware estimator of $\nabla f(\xx^t)$ that satisfies Assumption~\ref{assump:unbiased-G}, allowing us to flexibly instantiate the framework with various variance-reduction techniques. 

For instance, we can combine \ref{Alg:RG-update} with \ref{Alg:SAGA-update} or \ref{Alg:SVRG-update}. We refer to the resulting estimators as \ref{Alg:RG-update}-\ref{Alg:SAGA-update} and \ref{Alg:RG-update}-\ref{Alg:SVRG-update}. Note that $S_t$ in the formulas for \ref{Alg:SAGA-update} and \ref{Alg:SVRG-update} is exactly the same random index set that is used in the ~\ref{Alg:RG-update} -- they share the same randomness for the sake of efficiency.

\textbf{Implementation.}
At the beginning, each client $i = 1, \ldots, n$ computes $\nabla f_i(\xx^0)$ and sends the result to the server; the server then aggregates these results, computing $\nabla f(\xx^0)$ to initialize $\gg^0$. This requires one full synchronization.
Then $\xx^1$ is computed based on $\gg^0$.
At each iteration $t \ge 0$, 
the server uses $\RSS$ which generates a random subset $S_{t}$. For \ref{Alg:RG-update}-\ref{Alg:SAGA-update}, the server sends $\xx^{t+1}$, $\xx^{t}$ to the clients in $S_{t}$. Each client $i \in S_{t}$ updates $\bb_i^{t} = \nabla f_i(\xx^{t})$ and sends
$\nabla f_i(\xx^{t+1})$ along with
$\bb_i^{t} - \bb_i^{t-1}$ (when $t \ge 2$) or $\bb_i^{t}$ (when $t = 1$) to the server. 
For \ref{Alg:RG-update}-\ref{Alg:SVRG-update}, the server sends $\xx^{t+1}$, $\xx^{t}$ and $\ww^{t}$ to the clients which then return the gradients evaluated at these three points.
If $\omega_t = 1$, the server additionally computes the new gradient $\nabla f(\ww^t)$ performing one full synchronization. 
After receiving all the vectors, the server can compute $\nabla f_{S_{t}}(\xx^{t+1})$, $\nabla f_{S_{t}}(\xx^{t})$, $\mG^{t}$ and $\gg^{t+1}$. 

For \ref{Alg:RG-update}-\ref{Alg:SAGA-update}, each client $i$ needs to store a single vector $\bb_{i}^t$ and the server 
needs to maintain two points $\xx^{t+1}$ and $\xx^t$, and two vectors $\bb^t$ and $\gg^t$.
For \ref{Alg:RG-update}-\ref{Alg:SVRG-update}, clients are stateless and the server is required to maintain three points $\xx^{t+1}$, $\xx^t$, $\ww^t$ and one vector $\nabla f(\ww^t)$.

\begin{lemma}[Error bound for \ref{Alg:RG-update}]
\label{thm:VarianceQr}
    Consider the RG estimator~\eqref{Alg:RG-update} under
    Assumptions~\ref{assump:unbiased-G} and~\ref{assump:SOD}. Then for any $T \ge 1$, we have: 
    \begin{equation*}
    \sum_{t=0}^T \Sigma_{t}^2
    \le
    \frac{2\beta}{2-\beta} \sum_{t=0}^{T-1}
    \sigma_t^2 
    + \frac{2\delta_m^2}{2\beta-\beta^2} \sum_{t=1}^{T}
    \chi_{t}^2 \;.
    \end{equation*}
    where $\Sigma_t^2 \defeq \E_{S_{[t-1]}}[\| \gg^t - \nabla f(\xx^t) \|^2]$, $\sigma_t^2 \defeq 
    \E_{S_{[t]}}
    [\| \mG^t - \nabla f(\xx^t)\|^2]
    $, $\chi_{t}^2 \defeq \E_{S_{[t-2]}}[\| \xx^t - \xx^{t-1} \|^2]$, and $S_{[t]} \defeq (S_0, \ldots, S_t)$.
\end{lemma}
The proof can be found in Section~\ref{sec:thm:VarianceQr}. We next show that the error bound of both \ref{Alg:SAGA-update} and \ref{Alg:SVRG-update} can be improved by combining them with~\ref{Alg:RG-update} and adjusting the parameter $\beta$. For instance, by combining Lemma~\ref{thm:VarianceQr} and Lemma~\ref{thm:VarianceSAGAMain}, we obtain the following result for \ref{Alg:RG-update}-\ref{Alg:SAGA-update}.
\begin{corollary}
\label{thm:Variance-SAGA-RG}
    Consider 
    the 
    \ref{Alg:RG-update}-\ref{Alg:SAGA-update} estimator
    under Assumptions~\ref{assump:unbiased-G} and ~\ref{assump:SOD}.
    Then for any $T \ge 1$, it holds that:
    \[
    \sum_{t=0}^T \Sigma_t^2 
    \le 
    \frac{4\beta n_m q_m}{(2-\beta)m} G_1^2 + 
    \frac{2\beta( n_m-1+\sqrt{n^2_m-n_m} )}{(2-\beta)(n-1)}
    \sum_{t=2}^{T-1} 
    G_t^2
    +\frac{8\beta^2n_m^2\delta_m^2 +2\delta_m^2}{2\beta - \beta^2} \sum_{t=1}^T \chi_t^2
    \;,
    \]
    where $\Sigma_t^2 \defeq \E_{S_{[t-1]}}[\|\gg^t - \nabla f(\xx^t)\|^2]$, 
    $G_t^2 \defeq 
    \E_{S_{[t-2]}}
    [\norm{\nabla f(\xx^t)}^2]
    $, $\chi_{t}^2 \defeq \E_{S_{[t-2]}}[\| \xx^t - \xx^{t-1} \|^2]$ and $S_{[t]}\defeq (S_0,\ldots,S_t)$.
\end{corollary}
By choosing $\beta \simeq \frac{1}{n_m}$, we get $\sum_{t=0}^T \Sigma_t^2 
    \lesssim 
    \frac{q_m}{m}
    G_1^2 + \frac{1}{n}\sum_{t=2}^{T-1} G_t^2
     + n_m \delta_m^2 \sum_{t=1}^T \chi_t^2$. 
Compared with the original variance bound for \ref{Alg:SAGA-update} (Lemma~\ref{thm:VarianceSAGAMain}), the  bound with \algoname{RG} achieves an improvement by a factor of $n_m$.

The error bound for the \ref{Alg:SVRG-update} estimator can be improved in a similar way.
\begin{corollary}
\label{thm:Variance-SVRG-RG}
    Consider 
    the 
    \ref{Alg:RG-update}-\ref{Alg:SVRG-update} estimator
    under Assumptions~\ref{assump:unbiased-G} and ~\ref{assump:SOD}.
    Then for any $T \ge 1$, it holds that:
    \[
    \sum_{t=0}^T \Sigma_t^2 
    \le 
    \frac{8\beta^2\delta_m^2/p_B^2 +2\delta_m^2}{2\beta - \beta^2}
    \sum_{t=1}^T \chi_t^2 \;,
    \]
    where $\Sigma_t^2 \defeq \E_{S_{[t-1]},\omega_{[t-1]}}[\| \gg^t - \nabla f(\xx^t) \|^2]$, $\chi_{t}^2 \defeq \E_{S_{[t-2]},\omega_{[t-2]}}[\| \xx^t - \xx^{t-1} \|^2]$, 
    $S_{[t]} \defeq (S_0,\ldots,S_t)$ and $\omega_{[t]} \defeq (\omega_1,\ldots,\omega_t)$.
\end{corollary}
Compared with the original variance bound for \ref{Alg:SVRG-update} (Lemma \ref{thm:VarianceSVRGMain}), the new bound achieves an improvement by a factor of $1 / p_B$ by choosing $\beta \simeq p_B$.

We can now incorporate both enhanced estimators into~\ref{Alg:PP}.
It can be shown that the iterates $\{\xx^t\}_{t=0}^{\infty}$ generated by \ref{Alg:PP}-\ref{Alg:RG-update}-\ref{Alg:SAGA-update} or 
\ref{Alg:PP}-\ref{Alg:RG-update}-\ref{Alg:SVRG-update} and the corresponding sequence $\{\mG^t\}_{t=0}^{\infty}$
satisfy Assumption~\ref{assump:unbiased-G}. (See Lemma~\ref{thm:random-dependence-saga} and~\ref{thm:random-dependence-svrg}).

\section{Communication and Local Complexity of I-CGM-RG}
\label{sec:Complexity}

We are ready to establish the complexity of \ref{Alg:PP} equipped with the \ref{Alg:RG-update}-\ref{Alg:SAGA-update} and  \ref{Alg:RG-update}-\ref{Alg:SVRG-update} estimator. 
We first present the result for \algoname{RG-SAGA}. The proof can be found in Section~\ref{sec:thm:PP-SAGA-main-paper}.

\begin{theorem}[\algoname{I-CGM-RG-SAGA}]
\label{thm:PP-SAGA-main-paper}
    Let \ref{Alg:PP} be applied to Problem~\ref{eq:problem} under Assumptions~\ref{assump:ED},~\ref{assump:SOD} and~\ref{assump:L1}, where 
    $\xx^{t+1} = \operatorname{CGM}_{\operatorname{rand}}(\lambda, \hat{K}_t, \xx^t, \gg^t)$ with $\hat{K}_t \sim \operatorname{Geom}(p)$
    and $\gg^t$ is generated by the \ref{Alg:RG-update}-\ref{Alg:SAGA-update} estimator. 
    Then by choosing $\lambda = 3\Delta_1 + 113\sqrt{n_m}\delta_m$, $\beta = \frac{1}{112 n_m}$ and $p = \frac{\lambda - \Delta_1}{8(L_1 + \lambda)}$,
    after $T = \lceil \frac{(256(\Delta_1 + 38 \sqrt{n_m}\delta_m)F^0}{\epsilon^2} \rceil$ iterations, 
    we have $\E[\|\nabla f(\bar{\xx}^T)\|^2] \le \epsilon^2$, where $\Bar{\xx}^T$ is is uniformly sampled from $(\xx^t)_{t=1}^T$.
    The communication complexity is at most $2 C_A \lceil n_m \rceil
    + (C_R+1) \lceil \frac{(256(\Delta_1 + 38 \sqrt{n_m}\delta_m)F^0}{\epsilon^2} \rceil$ and the local complexity is bounded by $14 + 2\lceil n_m \rceil
    + \frac{512(7\Delta_1 + 283\sqrt{n_m}\delta_m + 2L_1)F^0}{\epsilon^2} + \frac{4L_1}{\Delta_1 + 28 \sqrt{n_m}\delta_m}$.
\end{theorem}

The communication complexity of \algoname{I-CGM-RG-SAGA} is of order 
$C_A n_m + C_R \frac{\Delta_1 + (\sqrt{n_m}\delta_m)F^0}{\epsilon^2}$ and the local complexity is of order $ n_m + \frac{(\Delta_1 + \sqrt{n_m}\delta_m + L_1)F^0}{\epsilon^2}$ when $\frac{(\Delta_1 + \sqrt{n_m}\delta_m)F^0}{\epsilon^2} \gtrsim 1$. 
The $n_m$ term comes from $n_m$ sequential rounds with $\ASS$ for computing the full gradients in the beginning.

\textbf{Comparison: \ref{Alg:RG-update}-\ref{Alg:SVRG-update} Estimator.} 
The communication complexity 
of \algoname{I-CGM-RG-SVRG} is  $\cO\big(
        C_An_m 
        + 
        \frac{(
        C_R \Delta_1 + \sqrt{C_A C_R n_m}\delta_m
        )F^0}{\epsilon^2}  \bigr)$,
        where $C_A$ also affects the term involving $\epsilon$ 
        (see Appendix~\ref{sec:proof-PP-SVRG} for details and the result of the local complexity).

\section{Numerical experiments}
In this section, we verify the theory of the proposed methods in numerical experiments. We set $C_A=C_R=1$ in the definition of communication complexity for all the experiments. We choose this case to demonstrate that even when $\ASS$ and $\RSS$ are equally cheap, our proposed methods already outperform several commonly used algorithms. (The study of the scenario when $C_A > C_R$ can be found in Appendix~\ref{sec:ablation}.)

\textbf{Quadratic minimization with log-sum penalty.} Consider the problem of
$f(\xx)=\Avg f_i(\xx)$ with 
$f_i(\xx) := \frac{1}{b} \sum_{j=1}^{b} 
\frac{1}{2}
\langle \mA_{i,j} (\xx - \bb_{i,j}), \xx - \bb_{i,j} \rangle  
+ \sum_{k=1}^d\log\bigl(1+\alpha | \xx_k |\bigr)$, where $\alpha > 0$, $\bb_{i,j} \in \R^d$, $\mA_{i,j} \in \R^{d \times d}$ is a diagonal matrix, and $\cdot_k$ is an indexing operation of a vector. 
We set $\alpha = 10$,
$b = 5$, $n=100$ and $d = 1000$. Each coordinate of $\bb_{i,j}$ is uniformly sampled from $[0,10]$.
To generate $\mA_{i,j}$, we first sample a diagonal matrix $\bar{\mA}$ with entries uniformly distributed in $[0, 110]$, and then add $bn$ diagonal noise matrices whose entries are sampled from $[0, 18]$. Each resulting $\mA_{i,j}$ is clipped to the interval $[1, 100]$ on the diagonal, and some eigenvalues are further set close to zero. Consequently, the dataset satisfies
$\mathbf{0} \preceq \mA_{i,j} \preceq 100\mathbf{I}$ for any $i,j$, with $\Delta_1 \approx \delta \approx 5$ and $L_{\max} \approx 100$. We set $m=\sqrt{n}$. For \algoname{I-CGM-RG},
we set $p=\frac{\delta}{L}$, $\lambda =  \frac{\sqrt{n}}{m}\delta + \Delta_1$, $\eta =  2L_{\max}$ and $\beta = \frac{m}{n}$. We compare two proposed algorithms against \algoname{Scaffold}~\cite{scaffold}, \algoname{FedAvg}~\cite{fedavg} (with sampling), \algoname{SABER-full}~\cite{saber} (with PAGE),
\algoname{SABER-partial}~\cite{saber} (only compute full gradient once) and \algoname{GD} (running directly on $f$). 
%We run 700 iterations for all methods (except for GD). 
For SVRG-based methods, the expected number of communication rounds at each iteration is roughly $\frac{m}{n}n + m$, which is twice as large as other methods. From Figure~\ref{fig:quadratics}, we observe that: 1) \algoname{I-CGM-RG-SAGA} is the most efficient in both communication and local computation. 2) \algoname{Scaffold} cannot fully exploit $\delta$-similarity as its local complexity is comparable to \algoname{GD} (the theoretical local complexity of both methods depends on $L_{\max}$). Finally, \algoname{I-CGM-RG-SAGA} with different initialization strategies can be found in Figure~\ref{fig:SAGA_r0_quadratics}.

\textbf{Logistic regression with nonconvex regularizer.} We now experiment with the binary classification task on two
real-world LIBSVM datasets~\cite{libsvm}. We use the 
standard regularized logistic loss: 
$f(\xx)=\Avg f_i(\xx)$ with 
$
f_i(\xx) 
:= 
\frac{n}{M}\sum_{j=1}^{m_i} 
\log(1 + \exp(-y_{i,j} \lin{\aa_{i,j}, \xx}))
+
\alpha \sum_{k=1}^d \frac{[\xx]_k^2}{1 + [\xx]_k^2 }
$
where $\alpha > 0$, 
$(\aa_{i,j},y_{i,j})\in\R^{d+1}$ 
are feature and labels and $M := \sum_{i=1}^n{m_i}$ is the total number of data points. We use $m=1$ and $n=10$. We plot the local $L_1$ and $\delta$ by computing $\norm{\nabla f_1(\xx^t) - \nabla f_1 (\xx^{t+1})}/\norm{\xx^t - \xx^{t+1}}$ and
$\sqrt{\Avg \norm{\nabla h_i(\xx^t) - \nabla h_i (\xx^{t+1})}^2 / \norm{\xx^t - \xx^{t+1}}^2}$
along the iterates of \algoname{I-CGM-RG-SAGA}. We observe that $\delta$ is much smaller than $L_1$ for the mushrooms dataset, while being comparable for the duke dataset. However, for both cases, \algoname{I-CGM-RG-SAGA} remains the most efficient in communication complexity.

\begin{figure}[tb!]
    \centering
    \includegraphics[width=0.8\linewidth]{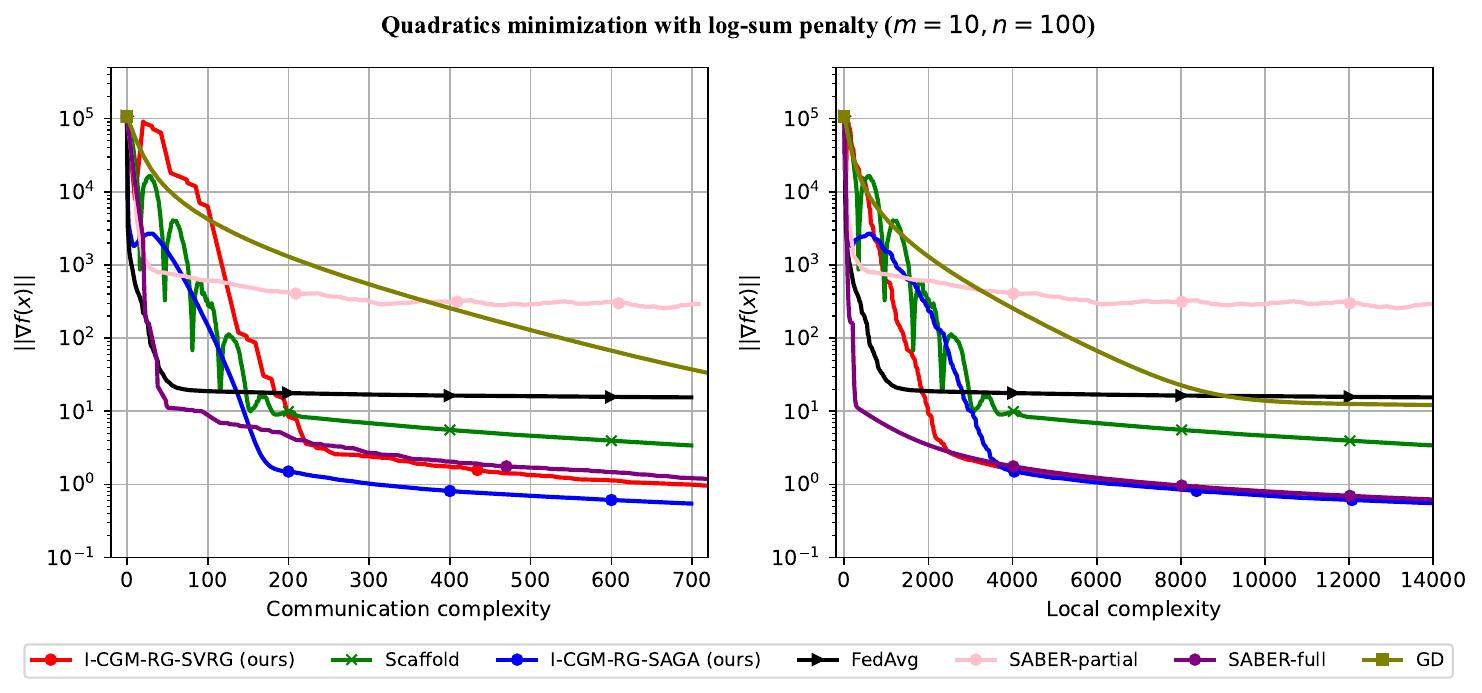}
    %\vskip-2mm
    \caption{Comparisons of different algorithms for solving the quadratic minimization problems with non-convex log-sum penalty.
    }
    \label{fig:quadratics}
\end{figure}

\begin{figure}[tb]
    \centering
    \vskip-2mm
    \includegraphics[width=1.0\linewidth]{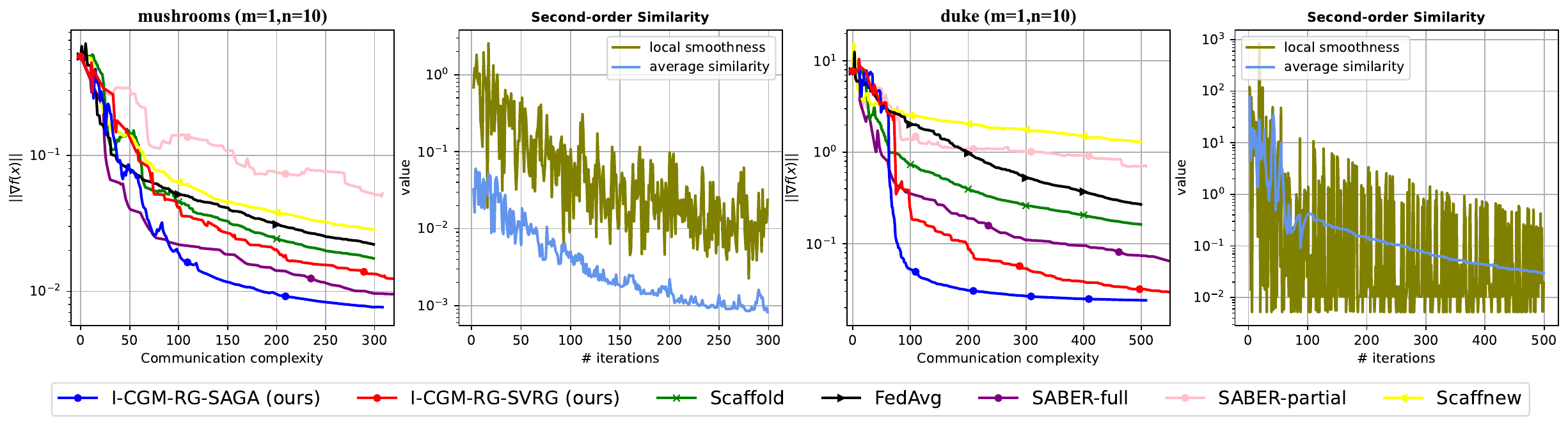}
    \vskip-2mm
    \caption{Comparisons of different algorithms on two LIBSVM datasets using logistic loss with non-convex regularizer. 
    }
    \label{fig:logistic-regression}
    \vskip - 0.1 in
\end{figure}

\paragraph{Deep learning tasks.}
We defer the study of neural network training to Appendix~\ref{sec:additional-exp}, where more experiments and details can be found.

\section{Conclusion}
\label{sec:Conclusion}
We introduced a new simple model  
for comparing federated optimization algorithms, where different client-selection strategies are associated with non-uniform costs. 
Within this model, we developed a new family of algorithm based on inexact composite gradient method with recursive gradient estimator. This design enables us to exploit functional similarity among clients while supporting partial client participation---a key requirement in practical FL systems.
It is efficient when full synchronizations (requiring sequential communications with all clients) are costly compared to client sampling. The key technical contribution of this work is a new variance bound for the SAGA estimator, which depends on the functional dissimilarity constant~$\delta$ rather than individual smoothness. This allows the SAGA-based variant of \algoname{I-CGM-RG} to outperform the previously best-known communication complexity of \algoname{SARAH}-based methods.

\textbf{Limitations and future extensions}.
1) In this work, we have assumed that there exists one delegate client that is reliable for communication. If we modify the setting and remove this delegate client, then we can still guarantee similar complexity with minor modifications. Specifically, instead of fixing the index $1$ in~\ref{Alg:PP},
we can sample $i_t \in [n]$ uniformly at random and define the updates as $\xx^{t+1}
    \approx
    \argmin_{\xx \in \R^d}
    \bigl\{ F_{t}(\xx) := f_{i_t}(\xx) + h_{i_t}(\xx^t) + \langle \gg^t - \nabla f_{i_t}(\xx^t), \xx - \xx^t \rangle 
    + \frac{\lambda}{2}\norm{ \xx-\xx^t }^2 \bigr\}$. This variant uses $\RSS$ instead of $\DSS$ at each iteration.  
    To ensure the convergence rate of $\frac{\lambda F^0}{T}$, we need to choose $\lambda \simeq \Delta_{\max}$~\cite{fedred}, where $\Delta_{\max} \lesssim \Delta_1$ is defined in~\ref{assump:SD}.
However, suppose there exists more than one delegate client, then it is interesting to check if we can further improve the current complexity.
2) We have shown that the variance of the SAGA estimator is bounded by the function dissimilarity constant $\delta$. An interesting question is whether something similar can be done for another closely related popular gradient estimator, SAG~\cite{sag}, used in Scaffold~\cite{scaffold}. It turns out that the answer is negative (see Section~\ref{sec:SAG}).
3) Our analysis focuses on the deterministic first-order oracle
$\OO_{f_i} = \OO_{\FO_i}$. It is interesting to develop efficient algorithms with stochastic, zero-order, or higher-order oracles. 
4) The current model does not impose constraints on the size of information that is transmitted between the server and clients. 
A promising direction is to incorporate communication compression and study how such constraints affect the algorithm design and overall complexity.

\bibliographystyle{unsrt}
{\small
\bibliography{reference}

\begin{thebibliography}{10}

\bibitem{fedavg}
Brendan McMahan, Eider Moore, Daniel Ramage, Seth Hampson, and Blaise~Aguera y~Arcas.
\newblock {Communication}-{Efficient} {Learning} of {Deep} {Networks} from {Decentralized} {Data}.
\newblock In {\em Artificial intelligence and statistics}, pages 1273--1282. PMLR, 2017.

\bibitem{kairouz2021advances}
Peter Kairouz, H.~Brendan McMahan, Brendan Avent, Aurélien Bellet, Mehdi Bennis, Arjun~Nitin Bhagoji, Keith Bonawitz, Zachary Charles, Graham Cormode, Rachel Cummings, Rafael G.~L. D'Oliveira, Salim~El Rouayheb, David Evans, Josh Gardner, Zachary Garrett, Adrià Gascón, Badih Ghazi, Phillip~B. Gibbons, Marco Gruteser, Zaid Harchaoui, Chaoyang He, Lie He, Zhouyuan Huo, Ben Hutchinson, Justin Hsu, Martin Jaggi, Tara Javidi, Gauri Joshi, Mikhail Khodak, Jakub Konečný, Aleksandra Korolova, Farinaz Koushanfar, Sanmi Koyejo, Tancrède Lepoint, Yang Liu, Prateek Mittal, Mehryar Mohri, Richard Nock, Ayfer Özgür, Rasmus Pagh, Mariana Raykova, Hang Qi, Daniel Ramage, Ramesh Raskar, Dawn Song, Weikang Song, Sebastian~U. Stich, Ziteng Sun, Ananda~Theertha Suresh, Florian Tramèr, Praneeth Vepakomma, Jianyu Wang, Li~Xiong, Zheng Xu, Qiang Yang, Felix~X. Yu, Han Yu, and Sen Zhao.
\newblock {Advances} and {Open} {Problems} in {Federated} {Learning}.
\newblock {\em Foundations and Trends® in Machine Learning}, 14(1--2):1--210, 2021.

\bibitem{konevcny2016communication}
Jakub Kone{\v{c}}n{\`y}, H~Brendan McMahan, Felix~X Yu, Peter Richt{\'a}rik, Ananda~Theertha Suresh, and Dave Bacon.
\newblock {Federated} {Learning}: {Strategies} for {Improving} {Communication} {Efficiency}.
\newblock {\em arXiv preprint arXiv:1610.05492}, 2016.

\bibitem{woodworth2018graph}
Blake~E Woodworth, Jialei Wang, Adam Smith, Brendan McMahan, and Nati Srebro.
\newblock {Graph} {Oracle} {Models}, {Lower} {Bounds}, and {Gaps} for {Parallel} {Stochastic} {Optimization}.
\newblock {\em Advances in neural information processing systems}, 31, 2018.

\bibitem{korhonen2021towards}
Janne~H Korhonen and Dan Alistarh.
\newblock {Towards} {Tight} {Communication} {Lower} {Bounds} for {Distributed} {Optimisation}.
\newblock {\em Advances in Neural Information Processing Systems}, 34:7254--7266, 2021.

\bibitem{celgd}
Kumar~Kshitij Patel, Lingxiao Wang, Blake~E Woodworth, Brian Bullins, and Nati Srebro.
\newblock {Towards} {Optimal} {Communication} {Complexity} in {Distributed} {Non}-{Convex} {Optimization}.
\newblock In {\em Advances in Neural Information Processing Systems}, volume~35, 2022.

\bibitem{NIPS2013_d6ef5f7f}
Yuchen Zhang, John Duchi, Michael~I Jordan, and Martin~J Wainwright.
\newblock {Information}-{Theoretic} {Lower} {Bounds} for {Distributed} {Statistical} {Estimation} with {Communication} {Constraints}.
\newblock In {\em Advances in Neural Information Processing Systems}, volume~26. Curran Associates, Inc., 2013.

\bibitem{davies2020new}
Peter Davies, Vijaykrishna Gurunathan, Niusha Moshrefi, Saleh Ashkboos, and Dan Alistarh.
\newblock {New} {Bounds} for {Distributed} {Mean} {Estimation} and {Variance} {Reduction}.
\newblock {\em arXiv preprint arXiv:2002.09268}, 2020.

\bibitem{JMLR:v20:19-543}
Kevin Scaman, Francis Bach, S{\'e}bastien Bubeck, Yin~Tat Lee, and Laurent Massouli{\'e}.
\newblock {Optimal} {Convergence} {Rates} for {Convex} {Distributed} {Optimization} in {Networks}.
\newblock {\em Journal of Machine Learning Research}, 20(159):1--31, 2019.

\bibitem{sarah}
Lam~M Nguyen, Jie Liu, Katya Scheinberg, and Martin Tak{\'a}{\v{c}}.
\newblock {SARAH}: {A} {Novel} {Method} for {Machine} {Learning} {Problems} using {Stochastic} {Recursive} {Gradient}.
\newblock In {\em International conference on machine learning}, pages 2613--2621. PMLR, 2017.

\bibitem{page}
Zhize Li, Hongyan Bao, Xiangliang Zhang, and Peter Richt{\'a}rik.
\newblock {PAGE}: {A} {Simple} and {Optimal} {Probabilistic} {Gradient} {Estimator} for {Nonconvex} {Optimization}.
\newblock In {\em International conference on machine learning}, pages 6286--6295. PMLR, 2021.

\bibitem{saber}
Konstantin Mishchenko, Rui Li, Hongxiang Fan, and Stylianos Venieris.
\newblock {Federated} {Learning} {Under} {Second}-{Order} {Data} {Heterogeneity}, 2024.

\bibitem{svrp}
Ahmed Khaled and Chi Jin.
\newblock {Faster} {Federated} {Optimization} under {Second}-{Order} {Similarity}.
\newblock In {\em The Eleventh International Conference on Learning Representations}, 2023.

\bibitem{chayti2022optimization}
El~Mahdi Chayti and Sai~Praneeth Karimireddy.
\newblock {Optimization} with {Access} to {Auxiliary} {Information}.
\newblock {\em arXiv preprint arXiv:2206.00395}, 2022.

\bibitem{mime}
Sai~Praneeth Karimireddy, Martin Jaggi, Satyen Kale, Mehryar Mohri, Sashank Reddi, Sebastian~U. Stich, and Ananda~Theertha Suresh.
\newblock {Breaking} the {Centralized} {Barrier} for {Cross}-{Device} {Federated} {Learning}.
\newblock In {\em Advances in Neural Information Processing Systems}, 2021.

\bibitem{sag}
Mark Schmidt, Nicolas Le~Roux, and Francis Bach.
\newblock {Minimizing} {Finite} {Sums} with the {Stochastic} {Average} {Gradient}.
\newblock {\em Mathematical Programming}, 162:83--112, 2017.

\bibitem{saga}
Aaron Defazio, Francis Bach, and Simon Lacoste-Julien.
\newblock {SAGA}: {A} {Fast} {Incremental} {Gradient} {Method} with {Support} for {Non}-{Strongly} {Convex} {Composite} {Objectives}.
\newblock {\em Advances in neural information processing systems}, 27, 2014.

\bibitem{reddi2016fast}
Sashank~J Reddi, Suvrit Sra, Barnab{\'a}s P{\'o}czos, and Alex Smola.
\newblock {Fast} {Incremental} {Method} for {Nonconvex} {Optimization}.
\newblock {\em arXiv preprint arXiv:1603.06159}, 2016.

\bibitem{zerosarah}
Zhize Li, Slavom{\'\i}r Hanzely, and Peter Richt{\'a}rik.
\newblock {ZeroSARAH}: {Efficient} {Nonconvex} {Finite}-{Sum} {Optimization} with {Zero} {Full} {Gradient} {Computation}.
\newblock {\em arXiv preprint arXiv:2103.01447}, 2021.

\bibitem{scaffold}
Sai~Praneeth Karimireddy, Satyen Kale, Mehryar Mohri, Sashank Reddi, Sebastian~U. Stich, and Ananda~Theertha Suresh.
\newblock {SCAFFOLD}: {Stochastic} {Controlled} {Averaging} for {Federated} {Learning}.
\newblock In {\em International conference on machine learning}, pages 5132--5143. PMLR, 2020.

\bibitem{fedred}
Xiaowen Jiang, Anton Rodomanov, and Sebastian~U Stich.
\newblock {Federated} {Optimization} with {Doubly} {Regularized} {Drift} {Correction}.
\newblock In {\em Proceedings of the 41st International Conference on Machine Learning}, volume 235 of {\em Proceedings of Machine Learning Research}, pages 21912--21945. PMLR, 21--27 Jul 2024.

\bibitem{feddyn}
Durmus Alp~Emre Acar, Yue Zhao, Ramon~Matas Navarro, Matthew Mattina, Paul~N Whatmough, and Venkatesh Saligrama.
\newblock {Federated} {Learning} {Based} on {Dynamic} {Regularization}.
\newblock {\em arXiv preprint arXiv:2111.04263}, 2021.

\bibitem{loopless-svrg}
Dmitry Kovalev, Samuel Horv{\'a}th, and Peter Richt{\'a}rik.
\newblock {Don}'t {Jump} {Through} {Hoops} and {Remove} {Those} {Loops}: {SVRG} and {Katyusha} are {Better} {Without} the {Outer} {Loop}.
\newblock In {\em Proceedings of the 31st International Conference on Algorithmic Learning Theory}, volume 117 of {\em Proceedings of Machine Learning Research}, pages 451--467. PMLR, 2020.

\bibitem{nemirovskij1994complexity}
A.~S. Nemirovsky.
\newblock {Information}-{Based} {Complexity} of {Convex} {Programming}.
\newblock 1994.

\bibitem{nemirovskij1983problem}
A.~S. Nemirovsky and D.~B. Yudin.
\newblock {Problem} {Complexity} and {Method} {Efficiency} in {Optimization}.
\newblock 1983.

\bibitem{arjevani2020complexity}
Yossi Arjevani, Amit Daniely, Stefanie Jegelka, and Hongzhou Lin.
\newblock On the complexity of minimizing convex finite sums without using the indices of the individual functions.
\newblock {\em arXiv preprint arXiv:2002.03273}, 2020.

\bibitem{s-dane}
Xiaowen Jiang, Anton Rodomanov, and Sebastian~U Stich.
\newblock {Stabilized} {Proximal}-{Point} {Methods} for {Federated} {Optimization}.
\newblock In {\em The Thirty-eighth Annual Conference on Neural Information Processing Systems}, 2024.

\bibitem{AccSVRS}
Dachao Lin, Yuze Han, Haishan Ye, and Zhihua Zhang.
\newblock {Stochastic} {Distributed} {Optimization} under {Average} {Second}-{Order} {Similarity}: {Algorithms} and {Analysis}.
\newblock {\em Advances in Neural Information Processing Systems}, 36, 2024.

\bibitem{takezawa2025exploiting}
Yuki Takezawa, Xiaowen Jiang, Anton Rodomanov, and Sebastian~U Stich.
\newblock {Exploiting} {Similarity} for {Computation} and {Communication}-{Efficient} {Decentralized} {Optimization}.
\newblock 2025.

\bibitem{spag}
Hadrien Hendrikx, Lin Xiao, Sebastien Bubeck, Francis Bach, and Laurent Massoulie.
\newblock {Statistically} {Preconditioned} {Accelerated} {Gradient} {Method} for {Distributed} {Optimization}.
\newblock In {\em International conference on machine learning}, pages 4203--4227. PMLR, 2020.

\bibitem{grad-sliding}
Dmitry Kovalev, Aleksandr Beznosikov, Ekaterina Borodich, Alexander Gasnikov, and Gesualdo Scutari.
\newblock {Optimal} {Gradient} {Sliding} and its {Application} to {Optimal} {Distributed} {Optimization} under {Similarity}.
\newblock {\em Advances in Neural Information Processing Systems}, 35:33494--33507, 2022.

\bibitem{proxsarah}
Nhan~H Pham, Lam~M Nguyen, Dzung~T Phan, and Quoc Tran-Dinh.
\newblock {ProxSARAH}: {An} {Efficient} {Algorithmic} {Framework} for {Stochastic} {Composite} {Nonconvex} {Optimization}.
\newblock {\em Journal of Machine Learning Research}, 21(110):1--48, 2020.

\bibitem{spiderboost}
Zhe Wang, Kaiyi Ji, Yi~Zhou, Yingbin Liang, and Vahid Tarokh.
\newblock {SpiderBoost} and {Momentum}: {Faster} {Variance} {Reduction} {Algorithms}.
\newblock {\em Advances in Neural Information Processing Systems}, 32, 2019.

\bibitem{katyushaX}
Zeyuan Allen-Zhu.
\newblock {Katyusha} x: {Practical} {Momentum} {Method} for {Stochastic} {Sum}-of-{Nonconvex} {Optimization}.
\newblock {\em arXiv preprint arXiv:1802.03866}, 2018.

\bibitem{svrg}
Rie Johnson and Tong Zhang.
\newblock {Accelerating} {Stochastic} {Gradient} {Descent} using {Predictive} {Variance} {Reduction}.
\newblock In {\em Advances in Neural Information Processing Systems}, volume~26. Curran Associates, Inc., 2013.

\bibitem{pmlr-v258-chayti25a}
El~Mahdi Chayti, Nikita Doikov, and Martin Jaggi.
\newblock {Improving} {Stochastic} {Cubic} {Newton} with {Momentum}.
\newblock In {\em Proceedings of The 28th International Conference on Artificial Intelligence and Statistics}, volume 258 of {\em Proceedings of Machine Learning Research}, pages 1441--1449. PMLR, 03--05 May 2025.

\bibitem{yuan-polyak}
Yuan Gao, Anton Rodomanov, and Sebastian~U. Stich.
\newblock {Non}-{Convex} {Stochastic} {Composite} {Optimization} with {Polyak} {Momentum}.
\newblock In {\em Proceedings of the 41st International Conference on Machine Learning}, ICML'24. JMLR.org, 2024.

\bibitem{storm}
Ashok Cutkosky and Francesco Orabona.
\newblock {Momentum}-{Based} {Variance} {Reduction} in {Non}-{Convex} {SGD}.
\newblock {\em Advances in neural information processing systems}, 32, 2019.

\bibitem{libsvm}
Chih-Chung Chang and Chih-Jen Lin.
\newblock {LIBSVM}: {A} {Library} for {Support} {Vector} {Machines}.
\newblock {\em ACM Transactions on Intelligent Systems and Technology}, 2:27:1--27:27, 2011.
\newblock Software available at \url{http://www.csie.ntu.edu.tw/~cjlin/libsvm}.

\bibitem{braverman2016communication}
Mark Braverman, Ankit Garg, Tengyu Ma, Huy~L Nguyen, and David~P Woodruff.
\newblock {Communication} {Lower} {Bounds} for {Statistical} {Estimation} {Problems} via a {Distributed} {Data} {Processing} {Inequality}.
\newblock In {\em Proceedings of the forty-eighth annual ACM symposium on Theory of Computing}, pages 1011--1020, 2016.

\bibitem{garg2014communication}
Ankit Garg, Tengyu Ma, and Huy~L Nguyen.
\newblock {On} {Communication} {Cost} of {Distributed} {Statistical} {Estimation} and {Dimensionality}.
\newblock {\em Advances in Neural Information Processing Systems}, 27, 2014.

\bibitem{arjevani2015communication}
Yossi Arjevani and Ohad Shamir.
\newblock {Communication} {Complexity} of {Distributed} {Convex} {Learning} and {Optimization}.
\newblock {\em Advances in neural information processing systems}, 28, 2015.

\bibitem{JMLR:v18:16-640}
Jason~D. Lee, Qihang Lin, Tengyu Ma, and Tianbao Yang.
\newblock {Distributed} {Stochastic} {Variance} {Reduced} {Gradient} {Methods} by {Sampling} {Extra} {Data} with {Replacement}.
\newblock {\em Journal of Machine Learning Research}, 18(122):1--43, 2017.

\bibitem{woodworth2021minimax}
Blake Woodworth.
\newblock {The} {Minimax} {Complexity} of {Distributed} {Optimization}.
\newblock {\em arXiv preprint arXiv:2109.00534}, 2021.

\bibitem{dane}
Ohad Shamir, Nati Srebro, and Tong Zhang.
\newblock {Communication}-{Efficient} {Distributed} {Optimization} using an {Approximate} {Newton}-{Type} {Method}.
\newblock In {\em International conference on machine learning}, pages 1000--1008. PMLR, 2014.

\bibitem{gao2025a}
Yuan Gao, Yuki Takezawa, and Sebastian~U Stich.
\newblock {A} {Bias} {Correction} {Mechanism} for {Distributed} {Asynchronous} {Optimization}.
\newblock {\em Transactions on Machine Learning Research}, 2025.

\bibitem{nesterov-book}
Yurii Nesterov.
\newblock {\em {Lectures} on {Convex} {Optimization}}.
\newblock Springer Publishing Company, Incorporated, 2nd edition, 2018.

\bibitem{allen2018katyusha}
Zeyuan Allen-Zhu.
\newblock {Katyusha} x: {Practical} {Momentum} {Method} for {Stochastic} {Sum}-of-{Nonconvex} {Optimization}.
\newblock {\em arXiv preprint arXiv:1802.03866}, 2018.

\bibitem{emnist}
Gregory Cohen, Saeed Afshar, Jonathan Tapson, and Andre Van~Schaik.
\newblock {EMNIST}: {Extending} {MNIST} to {Handwritten} {Letters}.
\newblock In {\em 2017 international joint conference on neural networks (IJCNN)}, pages 2921--2926. IEEE, 2017.

\bibitem{yin2025a}
Yida Yin, Zhiqiu Xu, Zhiyuan Li, Trevor Darrell, and Zhuang Liu.
\newblock {A} {Coefficient} {Makes} {SVRG} {Effective}.
\newblock In {\em The Thirteenth International Conference on Learning Representations}, 2025.

\bibitem{scaffnew}
Konstantin Mishchenko, Grigory Malinovsky, Sebastian Stich, and Peter Richt{\'a}rik.
\newblock {ProxSkip}: {Yes}! {Local} {Gradient} {Steps} {Provably} {Lead} to {Communication} {Acceleration}! {Finally}!
\newblock In {\em International Conference on Machine Learning}, pages 15750--15769. PMLR, 2022.

\bibitem{cifar10}
Alex Krizhevsky, Vinod Nair, and Geoffrey Hinton.
\newblock {CIFAR}-10 ({Canadian} {Institute} for {Advanced} {Research}).

\bibitem{resnet}
Kaiming He, Xiangyu Zhang, Shaoqing Ren, and Jian Sun.
\newblock {Deep} {Residual} {Learning} for {Image} {Recognition}.
\newblock In {\em 2016 IEEE Conference on Computer Vision and Pattern Recognition (CVPR)}, pages 770--778, 2016.

\end{thebibliography}
}
\appendix

\newpage
\numberwithin{equation}{section}
\numberwithin{figure}{section}
\numberwithin{table}{section}

\newpage
\small
{\Huge\textbf{Appendix}}
\vspace{0.5cm}
\makeatletter
\let\addcontentsline\orig@addcontentsline
\makeatother

{\small
\tableofcontents
}

\newpage
\section{Related Work}
\label{sec:related work}
\subsection{Formalization of Federated Optimization Algorithms and Their Complexity}

Several prior works have proposed oracle models and complexity metrics for distributed and federated optimization. These works typically consider solving the same problem as in~\eqref{eq:problem}, where each of the $n$ clients or workers only has access to its own local function. The primary distinctions among these models lie in how communication and computation are formalized. One line of work focuses on the settings where 
each worker can compute arbitrary information about
its local objective, but only a limited number of
bits is allowed to transmitted during each communication round
\cite{braverman2016communication,garg2014communication,NIPS2013_d6ef5f7f}. In this setting, complexity is often defined as the number of rounds required to reach a target accuracy. Alternative models remove this constraint and instead measure the total number of bits communicated over the entire optimization process~\cite{korhonen2021towards}. 
Other works impose structural restrictions on the communicated information, such as requiring exchanged vectors to lie in a certain subspace (e.g., linear combinations of local gradients)~\cite{arjevani2015communication, JMLR:v18:16-640}.

The closest related model to ours is the Graph Oracle Model (GOM)~\cite{woodworth2018graph,celgd,woodworth2021minimax}.
GOM introduces a computation and communication graph that determines how each device queries its local oracle and how the computed information propagates through the devices during optimization. Once the oracle and the graph structure are fixed, we obtain a specific model that allows to define the corresponding optimization algorithms.
A commonly studied setting is the intermittent communication model, where $n$ devices work in parallel and synchronize after every $K$ local oracle queries. This setting becomes
conceptually equivalent to our model when 1) all clients participate in every round, 2) the number of local oracle queries $K_r$ is uniformly bounded across all communication rounds, and 3)
the server and the clients are allowed to send its entire accumulated information. 

Beyond this scenario, there are several main differences between GOM and our proposed model. 
1) GOM does not distinguish between different client-selection strategies that might have non-uniform associated costs. 
2) Even when all the strategies have the same cost, 
GOM fixes the maximum number of local oracle queries for each round, whereas our model allows $K_r$ to vary across rounds. This flexibility enables modeling algorithms such as \algoname{DANE} \cite{dane,fedred}, which needs to solve each local subproblem sufficiently accurately, making $K_r$ dependent on the round $r$. 
3) While partial client participation can be modeled in GOM by generating a random graph, the resulting algorithms are generally restricted to using pre-specified groups of clients at each round. This effectively enforces an offline client-selection strategy for GOM, since it may not account for the need to know the past responses of specific clients before deciding which clients to contact next. In contrast, our model fully supports online and adaptive client-selection strategies.

We believe that our model is reasonably simple and it appears to sufficiently capture how federated optimization algorithms work in practice. Even in the cases where our model is mathematically equivalent to existing ones, our model could still be more convenient to be used.

\subsection{Comparison with Existing Federated Optimization Methods}
\textbf{Notations in Table~\ref{tab:method_comparison}.}
We denote $F^0 \defeq f(\xx^0) - f^\star$, $n_m \defeq \frac{n}{m}$,
$\delta_m^2 \defeq \frac{n-m}{n-1}\frac{\delta^2}{m}$,
$\zeta_m^2 \defeq \frac{\zeta^2}{m}$, and
$1 \lesssim C_R \lesssim C_A$ are the costs of communicating with a random set of $m$ clients and a specific set of $m$ clients, respectively.

In this section, we compare our proposed methods with several popular federated optimization algorithms in terms of their communication and local complexities (Section~\ref{sec:Model}).
For simplicity and to ensure fair comparisons across algorithms, we assume that all methods use the deterministic first-order oracle locally, i.e., $\OO_{f_i} = \OO_{\FO_i}$, for all $i \in [n]$.
We first state and discuss the assumptions under which each algorithm was analyzed in the literature. 
The abbreviations used in Table~\ref{tab:method_comparison} are defined as follows. 

\begin{assumption}[FS (Function Smoothness)]
There exists $L_f > 0$ such that 
for any $\xx,\yy \in \R^d$, we have:
$\|\nabla f(\xx) - \nabla f(\yy)\| \le L_f \|\xx - \yy \|$.
\end{assumption}

\begin{assumption}[IS (Individual Smoothness)]
There exists $L_{\max} > 0$ such that 
for any $\xx,\yy \in \R^d$ and any $i \in [n]$, we have:
$\|\nabla f_i(\xx) - \nabla f_i(\yy)\| \le L_{\max} \|\xx - \yy\|$.
\end{assumption}

\begin{assumption}[SD (Second-order Dissimilarity)~\cite{scaffold,fedred,gao2025a}]
\label{assump:SD}
There exists $\Delta_{\max} > 0$ such that 
for any $\xx,\yy \in \R^d$ and any $i \in [n]$, we have:
$\|\nabla h_i(\xx) - \nabla h_i(\yy)\| \le \Delta_{\max} \|\xx - \yy \|$.
\end{assumption}

\begin{assumption}[BGD (Bounded Gradient Dissimilarity)]
There exists $\zeta > 0$ such that 
for any $\xx \in \R^d$, we have:
$\Avg \|\nabla f_i(\xx) - \nabla f_i(\yy)\|^2 \le \zeta^2 $.
\end{assumption}

Note that the problem class of IS belong to SD, and 
SD implies
Assumption~\ref{assump:ED}      and~\ref{assump:SOD}. Moreover, any functions that satisfy Assumption~\ref{assump:ED} and~\ref{assump:L1} belong to the class of FS. 
Finally, the class of FS partially overlaps with the problem class defined by Assumptions \ref{assump:ED} and \ref{assump:SOD}. In Table~\ref{tab:method_comparison}, the assumption under which Centralized GD is analyzed is the most general one, and those for I-CGM-RG are the second most general. Finally, the smoothness constants satisfy the following relations:
\[
\delta,\Delta_1 \lesssim \Delta_{\max} \lesssim L_{\max},  \quad L_1,L_f \lesssim L_{\max} \;.
\]

We next briefly describe each method in Table~\ref{tab:method_comparison} and discuss how the communication and local complexities are computed for each method. There are two operations that are commonly used in these methods. We describe them here to avoid repetition later. Denote $n_m \defeq n / m$.
The first operation is to compute the full gradient $\nabla f$ at the server at a certain point $\xx$. 
As discussed in Section~\ref{sec:Model}, this can be implemented with $\lceil n_m \rceil$ successive communication rounds, each involving the use of $\ASS$ and one local gradient computation. Each such operation adds therefore $N_{\nabla f} \defeq C_A \lceil n_m \rceil$ to the total communication complexity and $K_{\nabla f} \defeq \lceil n_m \rceil$ to the total local complexity of an algorithm.  

Another commonly used operation is to compute several mini-batch gradients $\nabla f_{S}$ at $b \ge 1$ points where $S \in \binom{[n]}{m}$ is sampled uniformly at random. This requires one communication round using $\RSS$. The communication complexity of this operation is $N_{\nabla f_S,b} \defeq C_R$ and the local complexity is $K_{\nabla f_S, b} \defeq b$. 

In what follows, we omit the subscript $\cF$ in the notation for the complexities $N_{\cF}(\epsilon)$ and $K_{\cF}(\epsilon)$; the corresponding problem class is specified in Table~\ref{tab:method_comparison} for each method.

\textbf{Centralized \algoname{GD}}. The method iterates: 
\[ \xx^{t+1} = \xx^t - \frac{1}{L_f} \nabla f(\xx^t) \;. \] 
The iteration complexity of \algoname{GD} is $T = \cO( \frac{L_f F^0}{\epsilon^2} )$~\cite{nesterov-book}, implying that the communication complexity is $N(\epsilon) = 
T N_{\nabla f} = 
\cO( C_A n_m \frac{L_f F^0}{\epsilon^2} )$ and
the local complexity is $K(\epsilon)
=  T K_{\nabla f} =\cO(n_m \frac{L_fF^0}{\epsilon^2})$.

\textbf{\algoname{FedRed}}~\cite{fedred}. We consider \algoname{FedRed-GD}, which initializes 
$\tilde{\xx}_0 = \xx_0$ and iterates: 
\[ 
\xx_{t+1} = \argmin_{
\xx \in \R^d}\Bigl\{
f_1(\xx_t)
+ \lin{\nabla f_1(\xx_t) + \nabla f(\tilde{\xx}_t) - \nabla f_1(\tilde{\xx}_t), \xx}
+\frac{\eta}{2} \norm{\xx - \xx_t}^2 + \frac{\lambda}{2}\norm{\xx - \tilde{\xx}_t}^2
\Bigr\} \;,
\] 
where $\tilde{\xx}_{t+1} = \xx_{t+1}$ w.p. $p$ and 
$\tilde{\xx}_{t+1} = \tilde{\xx}_{t}$ w.p. $1-p$. The solution of the subproblem can be computed in a closed-form. For $p \simeq \frac{\Delta_1}{L_1+\Delta_1}$, 
$\eta \simeq L_1$ and $\lambda \simeq \Delta_1$,
the iteration complexity of the method is $T = \cO(\frac{L_1 F^0}{\epsilon^2})$~\cite{fedred}. 
In expectation, once every $1/p$ iterations, the server computes the full gradient $\nabla f(\tilde{\xx}_t)=\nabla f(\xx_t)$, which adds $N_{\nabla f}$ to the total communication complexity and $K_{\nabla f}$ to the total local complexity.
Then the server makes another communication round with $\DSS$, sends $\nabla f(\xx_t)$ to client $1$ which then performs $1/p$ local steps in expectation and sends the result back to the server.
The expected number of 
times the full gradient $\nabla f(\tilde{\xx}_t)$ is computed is $pT=\cO( \frac{\Delta_1 F^0}{\epsilon^2} )$. 
The expected number of communication rounds where $\DSS$ is used is $\E[N_D] = pT$. 
Therefore, the communication complexity is 
\[ 
N(\epsilon) = \E[C_AN_A + N_D]
=
\cO( pT N_{\nabla f} + pT )= \cO\Bigl(
 C_A n_m \frac{\Delta_1 F^0}{\epsilon^2} 
 \Bigr) \;. 
 \]
 The local complexity is bounded by 
 \[
 K(\epsilon)
 =
 \cO( pT K_{\nabla f} + T) = \cO
 (pTn_m + T) = \cO\Bigl( 
 n_m \frac{\Delta_1 F^0}{\epsilon^2} + \frac{L_1 F^0}{\epsilon^2} 
 \Bigr) \;.
 \]

\textbf{\algoname{FedAvg}}~\cite{fedavg}. At each communication round $r \ge 0$, the server uses $\RSS$ to select clients, sends $\xx^r$ to each client $i \in S_r$, which then returns an approximation solution $\xx_{i}^{r+1} \approx \argmin_{\xx} f_i(\xx) $ by running local GD starting at $\xx^r$ for $K_r$ steps. Then the next iterate is defined as $\xx^{r+1} = \frac{1}{m}\sum_{i\in S_r} \xx_{i}^{r+1}$. When using local-GD, the optimal number of local steps is of order $1$~\cite{scaffold}. Therefore, the local complexity is of the same order as the iteration complexity $T = \cO(  \frac{\zeta_m^2 F^0}{\epsilon^4}+\frac{\sqrt{L_{\max}}\zeta}{\epsilon^{3}}+\frac{L_{\max}F^0}{\epsilon^2}  )$
~\cite{scaffold}, and the communication complexity is $C_R T$.

\textbf{\algoname{MimeMVR}}~\cite{mime}. At each iteration $t \ge 0$, the server first uses $\RSS$ to get a random client set $S_t$ and 
computes the mini-batch gradient $\nabla f_{S_t}(\xx^t)$. Then 
the server uses $\ASS$ to establish communication with the same set of clients $S_t$, sends $\nabla f_{S_t}(\xx^t)$ to them which then updates:
\[ 
\xx_{i}^{t+1} \approx \argmin_{\xx \in \R^d} \bigl\{ f_i(\xx) + \lin{
\nabla f_{S_t}(\xx^t) - \nabla f_i(\xx^t), \xx} \bigr\}
\]
by running momentum-based first-order methods locally for $\Theta(\frac{L_{\max}}{\Delta_{\max}})$ steps. The next iterate is defined as: $\xx^{t+1} = \frac{1}{m}\sum_{i \in S_t} \xx_{i}^{t+1}$.
The communication complexity is thus $N(\epsilon) = \E[C_A N_A + C_R N_R] = (C_A + C_R)T$, where $T = \cO(\frac{\zeta^2_m F^0}{\epsilon^2}
+ \frac{\zeta_m \Delta_{\max} F^0}{\epsilon^{3}}
+\frac{\Delta_{\max}F^0}{\epsilon^2})$ is the iteration complexity~\cite{mime}.
The local complexity is 
$\cO(\frac{L_{\max}}{\Delta_{\max}} T)$.

\textbf{\algoname{CE-LGD}}~\cite{celgd}. 
The method initializes $\xx^{-1}=\xx^{0}$ and $\vv^{-1} \in \R^d$. 
At each iteration $t \ge 0$, the server first uses $\RSS$ to select clients $S_t$, sends $\xx^{t-1}$ and $\xx^t$ to the client $i \in S_t$ which then computes 
$\nabla f_{i}(\xx^t)$ and 
$\nabla f_{i}(\xx^{t-1})$ and sends them back to the server. 
The server computes $\vv^t = \nabla f_{S_t}(\xx^t) + (1-\rho)(\vv^{t-1} - \nabla f_{S_t}(\xx^{t-1}))$
where $\rho \in [0,1]$.
%(at iteration $r=0$, the procedure is repeated at most $T$ times to ensure the low variance of $\vv^0$.) 
Then the server uses $\RSS$ again to communicate with a random client. The client returns $\xx^{t+1}$ by running the local \algoname{SARAH} method using $\vv^t$ for $\Theta(\frac{L{\max}}{\Delta_{\max}})$ local steps. 
The iteration complexity is 
$T = \cO(\frac{\zeta_m^2 F^0}{\epsilon^2}
+ \frac{\zeta_m \Delta_{\max} F^0}{\sqrt{m}\epsilon^{3}}
+\frac{\Delta_{\max}F^0}{\epsilon^2})$~\cite{celgd}.
The communication complexity is  $N(\epsilon) = \cO(C_R T)$ and the local complexity is $K(\epsilon) = \cO(\frac{L_{\max}}{\Delta_{\max}}T)$.

\textbf{\algoname{Scaffold}}~\cite{scaffold}.
At the beginning,
each client $i=1,\ldots,n$ computes $\bb_i^0 = \nabla f_i(\xx^0)$ and sends the result to the server; the server then computes $\bb^0 = \nabla f(\xx^0)$, which adds 
$N_{\nabla f}$ and $K_{\nabla f}$ to the total communication and local complexities, respectively.
At each iteration $t \ge 0$, the server uses $\RSS$ to generate the client set $S_t$ and sends $\xx^t$ to each client $i \in S_t$, which then computes $\bb_i^t = \nabla f_i(\xx^t)$ and sends $\bb_i^t - \bb_i^{t-1}$ back to the server. The server then updates $\bb^t$ (SAG~\cite{sag}) according to~\eqref{eq:br} (for $t \ge 1$).
Then the server uses $\ASS$ to contact the clients in $S_t$ again and sends $\bb^t$ to them. Each client $i \in S_t$ computes
\[ 
\xx_{i}^{t+1} \approx \argmin_{\xx \in \R^d} \bigl\{
f_i(\xx) + \lin{\bb^t - \nabla f_i(\xx^t), \xx}
\bigr\}
\] 
by running local GD for $K \simeq 1$ steps and sends the result back to the server. The server computes the next iterate as $\xx^{t+1} = \frac{1}{m}\sum_{i \in S_t} \xx_{i}^{t+1}$. The iteration complexity is $T=\cO((\frac{n}{m})^{\frac{2}{3}}\frac{L_{\max}F^0}{\epsilon^2})$~\cite{scaffold}. 
The communication complexity $N(\epsilon)
=
\cO( C_A \lceil n_m \rceil
+ (C_A + C_R) T ) 
$. The local complexity $K(\epsilon) = \lceil n_m \rceil
+ (1+ K) T = \cO(n_m + T)$.

The final three algorithms do not strictly satisfy our definition of an algorithm because they do not clearly specify when to terminate the local method. 
However, we still present the conceptual methods and their communication complexity estimates, assuming (rather informally) that certain "local" operations can be implemented by running a certain local method for a sufficiently long time.

\textbf{\algoname{FedDyn}}~\cite{feddyn}.
During initialization, 
the server needs to collect $\xx_{i}^0$ that satisfies 
$\nabla f_i(\xx_{i}^0)=0$ from all clients. Therefore, the communication complexity of this operation is $C_A \lceil n_m \rceil$.
At each communication round $r \ge 0$, the server uses $\RSS$ to select clients $S_r$ and sends $\xx^r$ to each client $i \in S_r$ which then sends
\[ 
\xx_{i}^{r+1} = \argmin_{\xx}
\bigl\{ f_i(\xx) - \lin{\nabla f_i(\xx_{i,r}), \xx} + \frac{\lambda}{2} \norm{\xx - \xx^r}^2 \bigr\}, \quad 
i \in S_r \;,
\]
back to the server.
For $i \notin S_r$, $\xx_{i}^{r+1} = \xx_{i}^r$. Then the next iterate is updated as: 
\[ 
\xx^{r+1} = \frac{1}{m}\sum_{i \in S_r} \xx_{i}^{r+1} - \frac{1}{\lambda}\hh^{r+1}, \quad
\hh^{r+1} = \hh^{r} - \lambda \frac{1}{n}(\sum_{i \in S_r}\xx_{i}^{r+1} - \xx^r) \;.
\] 
The iteration complexity is $T = \cO( n_m\frac{L_{\max} F^0}{\epsilon^2})$~\cite{feddyn} and the communication complexity is $N(\epsilon) = \cO(C_A n_m + C_R T)$.

\textbf{\algoname{SABER-full}}~\cite{saber}. 
The method initializes $\xx^{-1}=\xx^{0}$ and $\vv^{-1}=\vv^0 = \nabla f(\xx^0)$. 
At each iteration $t \ge 0$,
w.p. $\frac{1}{ n_m }$, the server updates $\vv^t = \nabla f(\xx^t)$ which adds $N_{\nabla f}$ and $K_{\nabla f}$ to the total communication and local complexity, respectively. 
With probability $1- \frac{1}{n_m }$, the server 
uses $\RSS$, obtains the random set $S_t$, and computes two mini-batch gradients $\nabla f_{S_t}(\xx^t)$ and $\nabla f_{S_t}(\xx^{t-1})$. This operation adds $N_{\nabla f_{S_t}}^2$ and $K_{\nabla f_{S_t}}^2$ to the communication and local complexity, respectively. Then the server updates  $\vv^t = \vv^{t-1} + \nabla f_{S_t}(\xx^t) - \nabla f_{S_t}(\xx^{t-1})$. (The original method samples a single index $m_t$. Here we extend it to $S_t$.)
After the computation of $\vv^t$, the server uses $\RSS$, samples a random index $\tilde{m}_t \in [n]$, and sends $\vv^t$ and $\xx^t$ to the client $\tilde{m}_t$, which then returns 
\[
\xx^{t+1} \approx \argmin_{\xx \in \R^d}\Bigl\{f_{\tilde{m}_t}(\xx) + \lin{\vv^t - \nabla f_{\tilde{m}_t}(\xx^t), \xx} + \frac{\lambda}{2}\norm{\xx - \xx^t}^2 \Bigr\} 
\]
back to the server. 
The iteration complexity of the method is $T=\cO(\frac{\sqrt{n_m}\Delta_{\max}F^0}{\epsilon^2})$ if each subproblem is solved exactly. 
The total communication complexity is 
\[
N(\epsilon) = \E[C_A N_A + C_R N_R] 
= \cO\Bigl(
C_A n_m + \frac{1}{n_m}C_A Tn_m + (1-\frac{1}{n_m}) C_R T + C_R T \Bigr) = \cO(C_A n_m + C_A T) \;.
\]

\textbf{\algoname{SABER-partial}}~\cite{saber}. 
We refer to Algorithm 2 in the original paper as \algoname{SABER-partial}. By Theorem 3 in that paper, the best $p$ is 1.
The algorithm initializes $\vv^0 = \nabla f(\xx^0)$, which adds $N_{\nabla f}$ and $K_{\nabla f}$ to the total communication and local complexities, respectively. 
At each iteration $t \ge 1$, the method also needs to compute a mini-batch gradient $\vv^t = \frac{1}{s}\sum_{i \in S_t'} \nabla f_i(\xx^t)$ where $S_t'$ is sampled uniformly at random with replacement and $|S_t'| = s$ where $s$ is a parameter of the method. Since we assume that the server can communicate with at most $m$ clients at each round, implementing this operation requires $\lceil s/m \rceil$ sequential communication rounds with $\RSS$.
After that, the method needs to choose another random set $S_t$ with $|S_t| = s$. 
This adds $C_R \lceil s/m \rceil$ to the total communication complexity. The server 
sends $\xx^t$ and $\vv^t$ to the client $i \in S_t$, which then computes
\[
\xx_{i}^{t+1} \approx \argmin_{\xx \in \R^d}\Bigl\{f_i(\xx)
+ \lin{\vv^t - \nabla f_i(\xx^t),\xx} + \frac{\lambda}{2}\norm{\xx - \xx^t}^2 \Bigr\}
\;,
\]
and sends the result back to the server. The next iterate is  updated as $\xx^{t+1} = \frac{1}{s}\sum_{i\in S_t} \xx_{i}^{t+1}$.
If $\frac{\zeta^2}{\epsilon^2} \lesssim n$ and $s$ is chosen as $\Theta(\frac{\zeta^2}{\epsilon^2})$, then the method can output an $\epsilon$-approximate stationary point after 
$T = \cO(\frac{\Delta_{\max} F^0}{\sqrt{p}\epsilon^2})$ iterations if each subproblem is solved exactly. 
The communication complexity is $N(\epsilon) = \E[C_R N_{R} + C_A N_{A}] = \cO(C_R T \lceil \frac{\zeta^2}{m\epsilon^2} \rceil + C_A n_m )$.

\textbf{Discussions}. 
\algoname{FedDyn}, \algoname{SABER-FULL} and \algoname{SABER-Partial} do not strictly satisfy our definition of an algorithm since they do not precisely specify when to terminate the local methods.
The problem class for which 
\algoname{FedAvg}, \algoname{MimeMVR},  
\algoname{CE-LGD}, and
\algoname{SABER-partial}
are analyzed is the smallest among all methods. Specifically, in addition to IS, they also assume BGD, which can be restrictive and exclude simple quadratics. Among these four, \algoname{MimeMVR} and \algoname{CE-LGD} improve upon \algoname{FedAvg} and \algoname{SABER-partial} in terms of their dependence on the target accuracy~$\epsilon$. 
Except for \algoname{FedAvg}, the remaining three methods replace the dependence on $L_{\max}$ with $\Delta_{\max}$ in the communication complexity. Furthermore, compared to \algoname{MimeMVR}, \algoname{CE-LGD} achieves a better dependence on $m$ in the term involving $\epsilon^{-3}$.

For the remaining methods, \algoname{Scaffold} improves the dependence on $n$ from $n_m$ (as in \algoname{FedDyn}) to $n_m^{2/3}$.
\algoname{SAVER-full} further reduces the dependence on $n$ to $\sqrt{n_m}$ and simultaneously improves the smoothness dependence from $L_{\max}$ to $\Delta_{\max}$. 
\algoname{I-CGM-RG-SVRG} achieves a tighter bound of $C_R \Delta_1 + \sqrt{C_A C_R n_m}\delta_m \lesssim C_A \sqrt{n_m}\Delta_{\max}$. Finally, \algoname{I-CGM-RG-SAGA} improves the communication cost constant from $C_A$ to $C_R$, compared to \algoname{I-CGM-RG-SVRG}, while maintaining the same local complexity—the best among all existing methods.

\begin{remark}
According to~\cite{arjevani2020complexity}, one may alternatively compute the full gradient using only $\RSS$. 
Let $m=1$ for simplicity.
Lemma 2 in~\cite{arjevani2020complexity} shows that w.p. $1-\delta$, we can recover the full gradient $\nabla f(\xx)$ at a given point $\xx$ by making $2n^2\log(\frac{2n}{\delta})$ communication rounds with $\RSS$. This can be helpful when the cost $C_A$ is extremely large. Indeed, the current communication complexity of \algoname{I-CGM-RG-SAGA} is of order $C_A n + C_R( (\Delta_1 + \sqrt{n}\delta) F^0 / \epsilon^2 )$. As soon as $C_A \gtrsim C_R( (\Delta_1 + \sqrt{n}\delta) F^0 / \epsilon^2 )/n$, the complexity is dominated by $C_A n$. This term arises from $2n$ sequential communication rounds with $\ASS$ for computing the full gradients. Now if we replace these operations with $4n^2\log(\frac{2n}{\delta})$ sequential rounds with $\RSS$, the total complexity might be reduced to $C_R\bigl( n^2\log(n) + (\Delta_1 + \sqrt{n}\delta) F^0 / \epsilon^2 \bigr)$. We leave a full theoretical development of this direction as interesting future work.
\end{remark}

\section{Technical Preliminaries}
\label{sec:TechnicalPreliminaries}

We frequently use the following lemmas for the proofs. 
\begin{lemma}
    For any $\xx,\yy\in \R^d$ and any $\gamma > 0$, we have:
    \begin{equation}
        \abs{\lin{\xx, \yy}} \le \frac{\gamma}{2} \norm{\xx}^2 
        + \frac{1}{2 \gamma}\norm{ \yy}^2 ,
        \label{eq:BasicInequality1}
    \end{equation}
    \begin{equation}
        \norm{ \xx + \yy}^2 \le (1 + \gamma) \norm{ \xx}^2 
        + \Bigl( 1 + \frac{1}{\gamma} \Bigr) \norm{ \yy}^2 \;.
        \label{eq:BasicInequality2}
    \end{equation}
\end{lemma}

\begin{lemma}[\cite{s-dane}, Lemma 13]
\label{thm:Variance-SampleWithoutReplacement}
      Let $\{\gg_i\}_{i=1}^n$ be vectors in $\R^d$
      with $n \ge 2$. 
      Let $m \in [n]$ and 
      let $S \in \binom{[n]}{m}$ be sampled uniformly at random without replacement. 
      Let $\Bar{\gg} \defeq \Avg \gg_i$, $\sigma^2 \defeq \Avg \norm{\gg_i - \Bar{\gg}}^2$,
      and $\Bar{\gg}_S \defeq \frac{1}{m}\sum_{j \in S} \gg_j$.
      Then,
      \begin{equation}
            \E_S[\Bar{\gg}_S] = \Bar{\gg} \qquad
            \text{and} \qquad
          \E_S[\norm{ \Bar{\gg}_S - \Bar{\gg} }^2] 
          = \frac{n - m}{n - 1} \frac{\sigma^2}{m}.
          \label{eq:Variance-SampleWithoutReplacement}
      \end{equation}
\end{lemma}

\begin{lemma}[\cite{allen2018katyusha}, Fact 2.3]
\label{thm:Geom}
    Let $A_0,A_1,...$ be reals and let $K \sim \operatorname{Geom}(p)$ with $p \in (0,1]$, that is $\mathbb{P}(K=k)=(1-p)^kp$ for each $k \in \{0,1,2,...\}$. Then it holds that:
    $\E[K] = \frac{1}{p}-1$ and
    \begin{equation}
    \E[A_{K}] = (1-p)\E[A_{K+1}] + pA_0 \;.
    \end{equation}
\end{lemma}

\begin{proof}
    Using the identity $\sum_{k \ge 0}k q^k = \frac{q}{(1-q)^2}$ for any $|q| < 1$, we have:
    \[
    \E[K] = p\sum_{k \ge 0} k (1-p)^k 
    = p \frac{1-p}{p^2} = \frac{1}{p} - 1 \;.
    \]
    To prove the second part, using the definition of $K$,
    \[
    \E[A_{K+1}] = p\sum_{k \ge 0} A_{k+1} (1-p)^k 
    = \frac{p}{1-p}\sum_{k\ge 1} A_k (1-p)^{k}
    =\frac{1}{1-p}(\E[A_K] - pA_0)
    \;.
    \]
    Rearranging gives the claim.
\end{proof}

\begin{lemma}
\label{thm:SimpleRecurrence1}
    Let $(A_t)_{t=0}^{\infty}$, $(B_t)_{t=0}^{\infty}$ and be two non-negative sequences such that
    \[
    A_{i+1} 
    \le (1-\alpha) A_i 
    + B_i 
    \]
    for any $i \ge 0$ with $\alpha \in (0, 1]$. 
    Then for any $t \ge 1$,
    \[
    A_t \le (1-\alpha)^t A_0 + \sum_{i=1}^{t} (1-\alpha)^{t-i} B_{i-1} \;,
    \]
    and for any $T \ge 1$,
    \[
    \sum_{t=1}^T A_t \le 
    \frac{(1-\alpha)(1-(1-\alpha)^T)}{\alpha} A_0
        + \sum_{t=0}^{T-1} \frac{1-(1-\alpha)^{T-t}}{\alpha} B_t \le 
        \frac{1-\alpha}{\alpha} A_0
        + \frac{1}{\alpha} \sum_{t=0}^{T-1} B_t
        \;.
    \]
\end{lemma}
\begin{proof}
    When $\alpha = 1$, the claim clearly holds. 
    Let $0 < \alpha < 1$.
    Dividing both sides of the main recurrence by $(1-\alpha)^{i+1}$, we have for any $i \ge 0$:
    \[
    \frac{A_{i+1}}{(1-\alpha)^{i+1}}
    \le 
    \frac{A_i}{(1-\alpha)^i}
    + \frac{B_i}{(1-\alpha)^{i+1}}  \;.
    \]
    Summing up from $i=0$ to $i=t-1$, we get,
    for any $t \ge 1$:
    \begin{align*}
    \frac{A_t}{(1-\alpha)^t}
    \le A_0 + \sum_{i=0}^{t-1} \frac{B_i}{(1-\alpha)^{i+1}} 
    =
    A_0 
    + \sum_{i=1}^{t} \frac{B_{i-1}}{(1-\alpha)^{i}}
    \;.
    \end{align*}
    This proves the first claim. To prove the second part, we sum up the first claim from $t=1$ to $t=T$,
    {\allowdisplaybreaks
    \begin{align*}
        \sum_{t=1}^T A_t
        &\le \sum_{t=1}^T (1-\alpha)^t A_0
        +\sum_{t=1}^T \sum_{i=1}^t (1-\alpha)^{t-i} B_{i-1}
        \\
        &=
        \frac{(1-\alpha)(1-(1-\alpha)^T)}{\alpha} A_0
        + \sum_{i=1}^{T} \sum_{t=i}^{T}(1-\alpha)^{t-i} B_{i-1}
        \\
        &=
        \frac{(1-\alpha)(1-(1-\alpha)^T)}{\alpha} A_0
        + \sum_{i=1}^{T} 
        \frac{1-(1-\alpha)^{T-i+1}}{\alpha}
        B_{i-1}
        \\
        &=
        \frac{(1-\alpha)(1-(1-\alpha)^T)}{\alpha} A_0
        + \sum_{t=0}^{T-1} \frac{1-(1-\alpha)^{T-t}}{\alpha} B_t
        \;.
        \qedhere
    \end{align*}
    }
\end{proof}

\begin{lemma}
\label{thm:MinimizerOfGamma}
Let $p \in (0,1)$.
The minimizer of the problem
$\min_{\gamma \in (0, 1)} \{ f(\gamma) \defeq \frac{1 - \gamma p}{\gamma (1 - \gamma)} \}$ 
is attained at $\gamma^\star = \frac{1-\sqrt{1-p}}{p}$.
\end{lemma}
\begin{proof}
    Differentiate $f(\gamma)$, we have $f'(\gamma) = \frac{-1+2\gamma-p\gamma^2}{\gamma^2(1-\gamma)^2}$. 
    Setting $f'(\gamma)=0$ with $\gamma \in (0,1)$ gives
    $\gamma^\star = \frac{1-\sqrt{1-p}}{p}$. 
    Since $f(\gamma) \to \infty$ as $\gamma \to 0^+$ or $\gamma \to 1^-$, the critical point $\gamma^\star$ is the minimizer over $(0,1)$.
\end{proof}

\section{Proofs for \algoname{I-CGM}}
\begin{lemma}
\label{thm:FrDifference}
Let \ref{Alg:PP} be applied to Problem~\eqref{eq:problem} under Assumption~\ref{assump:ED}.
Let $\lambda > \Delta_1$.
Then for any $t \ge 0$, 
\begin{equation*}
    \norm{\nabla f(\xx^{t+1})}
    \le 
    (\lambda + \Delta_1)
    \hat{\chi}_{t+1}
    +
    \hat{\Sigma}_t
    + e_t \;,
    \end{equation*}
where $\hat{\chi}_t \defeq \| \xx^t - \xx^{t-1} \|$. 
For any $\xx \in \R^d$, we have:
\[
F_t(\xx^t) - F_t(\xx) \le f(\xx^t) - f(\xx)
+\frac{\hat{\Sigma}_t^2}{2(\lambda - \Delta_1)}  
\;.
\]
Suppose the iterates satisfy~\eqref{eq:Condition-SD}, then the function value decreases as:
\[
    f(\xx^{t+1})
    \le f(\xx^t)
    - \frac{\lambda - \Delta_1}{4}
    \hat{\chi}_{t+1}^2
    + \frac{\hat{\Sigma}_t^2}{\lambda - \Delta_1} \;.
\]
\end{lemma}
\begin{proof}
By the definition of $F_t$, we have:
    \begin{align*}
    \nabla F_t(\xx^{t+1}) &= \nabla f_1 (\xx^{t+1})
    + \gg^t - \nabla f_1(\xx^t) + \lambda (\xx^{t+1} - \xx^t) 
    \\
    &=
    \nabla f(\xx^{t+1}) + 
    \bigl(
    \gg^t - 
    \nabla f(\xx^t) \bigr)
    +
    \bigl(
    \nabla h_1(\xx^t) - \nabla h_1 (\xx^{t+1}) 
    \bigr) + \lambda (\xx^{t+1} - \xx^t)
    \;.
    \end{align*}
    It follows that:
    \begin{align*}
    \norm{\nabla f(\xx^{t+1})}
    &\le 
    \lambda \hat{\chi}_{t+1}
    +
    \| \nabla h_1(\xx^t) - \nabla h_1 (\xx^{t+1}) \|
    +\hat{\Sigma}_t
    + e_t
    \\
    &\stackrel{\eqref{eq:ED}}{\le}
    (\lambda + \Delta_1)
    \hat{\chi}_{t+1}
    + \hat{\Sigma}_t
    + e_t \;,
    \end{align*}
    which proves the first inequality.
    Using the definition of $F_t$,
    for any $\xx \in \R^d$, we have:
    \begin{align*}
    &\quad F_t(\xx^t) - F_t(\xx)
    \\
    &=f_1(\xx^t)
    + h_1(\xx^t)
    - f_1(\xx) - h_1(\xx^t)
    - \langle \gg^t - \nabla f_1(\xx^t), \xx - \xx^t \rangle
    - \frac{\lambda}{2} \|\xx - \xx^t\|^2 
    \\
    &= f(\xx^t) - f(\xx)
    + f(\xx) - f_1(\xx) 
    - h_1(\xx^t)- \langle \gg^t - \nabla f_1(\xx^t), \xx - \xx^t \rangle
    - \frac{\lambda}{2} \|\xx - \xx^t\|^2 
    \\
    &= f(\xx^t) - f(\xx)
    +\bigl(
    h_1 (\xx) - h_1(\xx^t)
    -\lin{\nabla h_1(\xx^t), \xx - \xx^t}
    \bigr)
    - \frac{\lambda}{2} \| \xx - \xx^t \|^2
    -
    \lin{\gg^t - \nabla f(\xx^t), \xx - \xx^t}
    \\
    &\stackrel{\eqref{assump:ED}}{\le} f(\xx^t) - f(\xx) - 
    \frac{\lambda - \Delta_1}{2} \| \xx-\xx^t \|^2
    -\lin{\gg^t - \nabla f(\xx^t), \xx - \xx^t}
    \;.
    \end{align*}    
    Using~\eqref{eq:BasicInequality1}, we can bound the last two terms by:
    $
    \frac{\| \gg^t - \nabla f(\xx^t)\|^2 } {2(\lambda-\Delta_1)} 
    $, which proves the second claim.
    Substituting $\xx = \xx^{t+1}$ and using~\eqref{eq:Condition-SD}, we get:
    \begin{align*}
    f(\xx^{t+1})
    &\le f(\xx^t)
    - 
    \frac{\lambda - \Delta_1}{2} \hat{\chi}_{t+1}^2
    -\lin{\gg^t - \nabla f(\xx^t), \xx^{t+1} - \xx^t}
    \\
    &\stackrel{\eqref{eq:BasicInequality1}}{\le} f(\xx^t)
    - \frac{\lambda - \Delta_1}{4}
    \hat{\chi}_{t+1}^2
    + \frac{\hat{\Sigma}_t^2}{\lambda - \Delta_1} \;.
    \qedhere
    \end{align*}
\end{proof}

\subsection{Proof for Theorem~\ref{thm:IterationCMGMain}.}
\label{sec:Proof-PP-Main}

\begin{proof}
    Let $t \ge 0$. By Lemma~\ref{thm:FrDifference}, we have:
    \[
    \frac{\lambda - \Delta_1}{4}
    \hat{\chi}_{t+1}^2 
    \le f(\xx^t)
    - f(\xx^{t+1})
    + \frac{\hat{\Sigma}_t^2}{\lambda - \Delta_1} \;.
    \]
    Using the first claim of Lemma~\ref{thm:FrDifference}, we have:
    \begin{align*}
    \| \nabla f(\xx^{t+1}) \|^2
    \le 
    \bigl (\lambda + \Delta_1)
    \hat{\chi}_{t+1}
    +
    \hat{\Sigma}_t
    + e_t \bigr)^2 
    \le 2 (\lambda + \Delta_1)^2 
    \hat{\chi}_{t+1}^2 
    + 2 (
    \hat{\Sigma}_t + e_t)^2   
    \;,
    \end{align*}
    Adding $(\lambda + \Delta_1)^2 
    \hat{\chi}_{t+1}^2$ to both sides of this inequality and substituting the first display, we have:
    \begin{align*}
        &\quad
        \| \nabla f(\xx^{t+1}) \|^2
        + (\lambda + \Delta_1)^2 
        \hat{\chi}_{t+1}^2
        \\
        &\le 3(\lambda + \Delta_1)^2
        \Bigl(
        \frac{4}{\lambda - \Delta_1} \bigl( 
        f(\xx^t) - f(\xx^{t+1})
        \bigr)
        + \frac{4}{\lambda - \Delta_1} \frac{\hat{\Sigma}_t^2}{\lambda - \Delta_1}
        \Bigr)
        + 2 (
        \hat{\Sigma}_t + e_t)^2  
        \\
        &\le 
        \frac{12 (\lambda + \Delta_1)^2}{\lambda - \Delta_1}\bigl(
        f(\xx^t) - f(\xx^{t+1})
        \bigr)
        +\Bigl(
        \frac{12(\lambda + \Delta_1)^2}{(\lambda-\Delta_1)^2} + 4
        \Bigr)\hat{\Sigma}_t^2
        +4e_t^2 \;.
    \end{align*}
    Summing up from $t=0$ to $T-1$, we get the claim.
\end{proof}

\subsection{Proofs of Local CGM for Solving the Subproblems}
\begin{lemma}[Composite gradient method]
\label{thm:CGM}
    Consider the composite problem:
    \[ \min_{\xx \in \R^d} \bigl\{ F(\xx) \defeq \phi(\xx) + \psi(\xx) \bigr\}
    \;, 
    \]
    where $\phi$ is $L_{\phi}$-smooth and $\psi$ is $\lambda_{\psi}$-strongly convex and simple with $\lambda_{\psi} \ge 0$. 
    Let $\eta = L_{\phi}$. 
    Consider the composite gradient method: 
    \[ \xx_{k+1} = \argmin_{\xx \in \R^d} \bigl\{L_k(\xx) \defeq \phi(\xx_k) +\lin{\nabla \phi(\xx_k), \xx - \xx_k} + \frac{\eta}{2}\| \xx - \xx_k \|^2 + \psi(\xx) \bigr\}  \;. \] 
    Then for any $k \ge 0$, 
    $F(\xx_{k+1}) \le F(\xx_k)$.
    For any $K \ge 1$, it holds that:
    \[
    \norm{\nabla F(\xx_K^\star)}^2
    \le \frac{8 L_\phi^2 \bigl[ F(\xx_0) - F(\xx_K)\bigr]}{(L_{\phi} + \lambda_{\phi})K} \;,
    \]
    where $\xx_K^\star = \argmin_{(\xx_k)_{k=1}^K} \norm{\nabla F(\xx_k)}$. Furthermore, if $\hat{K} \sim \operatorname{Geom}(p)$ with $p \in (0,1]$, then we also have:
        \[
    \E_{\hat{K}}\bigl[\| \nabla F(\xx_{\hat{K}+1}) \|^2 \bigr]
    \le\frac{8L_{\phi}^2 p}{L_{\phi}+\lambda_{\psi}} 
    \bigl[
    F(\xx_0) - \E_{\hat{K}}[F(\xx_{\hat{K}+1})] \bigr] \;.
    \]
\end{lemma}
\begin{proof}
    Let $k \ge 0$. 
    By $(L_{\phi}+\lambda_{\psi})$-strong convexity of $L_k$, for any $\xx \in \R^d$, we have,
    \[
    L_k(\xx) \ge L_k(\xx_{k+1})
    + \frac{L_{\phi} + \lambda_{\psi}}{2} \|\xx - \xx_{k+1}\|^2 \;.
    \]
    Substituting $\xx = \xx_k$, 
    it follows that,
    \begin{align*}
    F(\xx_k)
    &\ge \phi(\xx_k) +\lin{\nabla \phi(\xx_k), \xx_{k+1} - \xx_k} + \frac{L_{\phi}}{2}\norm{\xx_{k+1} - \xx_k}^2 + \psi(\xx_{k+1})  + \frac{L_{\phi} + \lambda_{\psi}}{2}\norm{\xx_{k+1} - \xx_k}^2 
    \\
    &\ge \phi(\xx_{k+1})
     + \psi(\xx_{k+1})  + \frac{L_{\phi} + \lambda_{\psi}}{2}\norm{\xx_{k+1} - \xx_k}^2 
     =
     F(\xx_{k+1})
     + \frac{L_{\phi} + \lambda_{\psi}}{2}\norm{\xx_{k+1} - \xx_k}^2 
    \;.
    \end{align*}
    This proves that the function value of $F$ monotinically decreases.
    By the definition of $\xx_{k+1}$, we get:
    \[
    \nabla \phi(\xx_k) + L_{\phi}(\xx_{k+1} - \xx_k) 
    + \nabla \psi(\xx_{k+1}) = 0 \;.
    \]
    It follows that:
    \[
    \nabla F (\xx_{k+1}) = 
    \nabla \phi(\xx_{k+1})
    + \nabla \psi(\xx_{k+1})
    = \nabla \phi(\xx_{k+1})
    - \nabla \phi(\xx_k)
    + L_{\phi}(\xx_k - \xx_{k+1}) \;,
    \]
    and hence,
    \[
    \norm{ \nabla F(\xx_{k+1}) } \le \norm{\nabla \phi(\xx_{k+1}) - \nabla \phi(\xx_{k})} + \eta \norm {\xx_{k+1} - \xx_k} \le 2L_{\phi} \norm{\xx_{k+1} - \xx_k} \;. 
    \]
    Substituting this inequality into the second display, we get, for any $k \ge 0$:
    \[
    \norm{\nabla F(\xx_{k+1})}^2 \le \frac{8 L_{\phi}^2}{L_{\phi} + \lambda_{\psi}} \bigl[ F(\xx_k) - F(\xx_{k+1}) \bigr] \;.
    \]
    Summing up from $k=0$ to $K-1$,
    we have:
    \[
    \sum_{k=1}^{K} 
    \| \nabla F(\xx_{k}) \|^2 
    \le 
    \frac{8 L_{\phi}^2}{L_{\phi} + \lambda_{\phi}}
    \bigl[ F(\xx_0) - F(\xx_K) 
    \bigr] \;.
    \]
    Dividing both sides by $K$, we get the first claim.
    
    For the second claim, substituting $k = \hat{K}$ with $\hat{K} \sim \operatorname{Geom}(p)$ into the last second display, passing to the expectations and applying Lemma~\ref{thm:Geom}, we have:
    \begin{align*}
    \E_{\hat{K}} [\| \nabla F(\xx_{\hat{K}+1})\|^2 ]
    &\le
    \frac{8 L_{\phi}^2}{L_{\phi}+\lambda_{\psi}}
    \E_{\hat{K}}[ F(\xx_{\hat{K}}) - F(\xx_{\hat{K}+1})  ]
    \\
    &\le \frac{8L_{\phi}^2}{L_{\phi}+\lambda_{\psi}}
    \bigl((1-p)\E_{\hat{K}}[F(\xx_{\hat{K}+1})] + p F(\xx_0) - \E_{\hat{K}}[F(\xx_{\hat{K}+1})] \bigr)
    \\
    &=\frac{8L_{\phi}^2 p}{L_{\phi}+\lambda_{\psi}} 
    \bigl[
    F(\xx_0) - \E_{\hat{K}}[F(\xx_{\hat{K}+1})] \bigr] \;.
    \qedhere
    \end{align*}
\end{proof}

\subsubsection{Proof of Lemma~\ref{thm:UpperboundKr-const}}
\label{sec:localstep-const}
\begin{proof}
Applying Lemma~\ref{thm:CGM} (with $\phi(\xx)=f_1(\xx)$ and $\psi(\xx)=\lin{\gg^t - \nabla f_1(\xx^t),\xx-\xx^t}+\frac{\lambda}{2}\norm{\xx - \xx^t}^2 $), we have for any $t \ge 0$:
    \[
    e_t^2 = 
    \norm{\nabla F_{t}(\xx^{t+1})}^2 \le \frac{8 L_1^2 ( F_t(\xx^t) - F_t^\star )}{(L_1 + \lambda) K} \;,
    \]
    where $F_t^\star \defeq \min_{\xx \in \R^d}\{F_t(\xx)\}$.
    Applying Lemma~\ref{thm:FrDifference}, we get:
    \[
    F_t(\xx^t) - F_t^\star \le f(\xx^t) - f^\star +
    \frac{\hat{\Sigma}_t^2}{2(\lambda - \Delta_1)} \;.
    \]
    It follows that:
    \[
    \sum_{t=0}^{T-1}
    e_t^2 
    \le \frac{8L_1^2}{(L_1 + \lambda)K}
    \Bigl( 
    \sum_{t=0}^{T-1} (f(\xx^t) - f^\star) + 
    \sum_{t=0}^{T-1}
    \frac{\hat{\Sigma}_t^2}{2(\lambda - \Delta_1)}
    \Bigr)
    \;.
    \]
    We next upper bound $\sum_{t=0}^{T-1} (f(\xx^t) - f^\star)$.
    Applying Lemma~\ref{thm:CGM}, we have for any $i \ge 0$:
    \[
    f(\xx^{i+1}) \le f(\xx^i) 
    + \frac{\hat{\Sigma}_i^2}{2(\lambda - \Delta_1)}
    \
    \;.
    \]
    Summing up from $i = 0$ to $i = t-1$, we have:
    \[
    f(\xx^t) \le f(\xx^0)
    + \sum_{i=0}^{t-1} \frac{\hat{\Sigma}_i^2}{2(\lambda - \Delta_1)} \;.
    \]
    Hence,
    \[
    \sum_{t=0}^{T-1} (f(\xx^t) - f^\star) \le 
    T(f(\xx^0) - f^\star)
    + \sum_{t=0}^{T-1}\sum_{i=0}^{t-1} \frac{\hat{\Sigma}_i^2}{2(\lambda - \Delta_1)} 
    \le T F^0
    + T \sum_{t=0}^{T-2} \frac{\hat{\Sigma}_t^2}{2(\lambda - \Delta_1)} \;. 
    \]
    It follows that:
    \begin{equation*}
    \sum_{t=0}^{T-1}
    e_t^2 
    \le \frac{8L_1^2}{(L_1 + \lambda)K}
    \Bigl( 
    T F^0
    + (T+1) \sum_{t=0}^{T-1} \frac{\hat{\Sigma}_t^2}{2(\lambda - \Delta_1)} 
    \Bigr)
    \;.
    \end{equation*}
    To achieve the accuracy condition~\ref{eq:AccuracyCondition}, 
    by the choice of $K$, we have
    \[
    \frac{8L_1^2 T}{(L_1 + \lambda)K} \le \frac{8L_1 T}{K} \le \lambda - \Delta_1 
    \le \frac{(\lambda + \Delta_1)^2}{\lambda - \Delta_1} \;, 
    \quad 
    \operatorname{and}
    \quad
    \frac{8L_1^2 }{(L_1 + \lambda)K} \frac{(T+1)}{2(\lambda - \Delta_1)} \le 
    \frac{8L_1 T}{(\lambda - \Delta_1)K}
    \le 1 \;.
    \]
    Passing to the full expectation, we get the claim.
\end{proof}

\subsubsection{Proof of Lemma~\ref{thm:UpperboundKr-random}}
\label{sec:localGD-K-Geom}

\begin{proof}
    Applying Lemma~\ref{thm:CGM}  and~\ref{thm:FrDifference}, we have for any $t \ge 0$:
    \[
    \E_{\hat{K}_t}
    [e_t^2]
    \le \frac{8L_1^2 p}{L_1+\lambda}
    \E_{\hat{K}_t}[F_t(\xx^t) - F_t(\xx^{t+1})] \le 
    \frac{8L_1^2 p}{L_1+\lambda}
    \Bigl(
    \E_{\hat{K}_t}\bigl[
    f(\xx^t) - f(\xx^{t+1})
    \bigr]
    +
    \frac{\hat{\Sigma}_t^2}{2(\lambda-\Delta_1)} 
    \Bigr)
    \;.
    \]
    Taking the full expectation and
    summing up from $t = 0$ to $t = T-1$, we have:
    \begin{align*}
    \sum_{t=0}^{T-1} \E[e_t^2] 
    &\le 
    \frac{8L_{1}^2 p}{L_{1}+\lambda} \Bigl( f(\xx^0) - f^\star 
    + \frac{1}{2(\lambda - \Delta_1)}
    \sum_{t=0}^{T-1}\Sigma_t^2
    \Bigr)
    \;.
    \end{align*}
    By the choice of $p$, it holds that:
    \[
    \frac{8L_{1}^2 p}{L_{1}+\lambda} 
    =\frac{L_1^2(\lambda - \Delta_1)}{(L_1+\lambda)^2} 
    \le \lambda - \Delta_1
    \le \frac{(\lambda + \Delta_1)^2}{\lambda - \Delta_1} \;,
    \]
    and 
    \[ 
    \frac{4L_1^2 p}{(L_1+\lambda)(\lambda - \Delta_1)} = \frac{L_1^2}{2(L_1+\lambda)^2} < 1 \;.
    \]
    Hence, condition~\eqref{eq:AccuracyCondition} is satisfied.
\end{proof}

\subsection{Proofs of Properties of the \algoname{SAGA} and \algoname{SVRG} Estimators}
\label{sec:SAGA-update}

\begin{lemma}
\label{thm:bt}
Consider the \ref{Alg:SAGA-update} estimator. Then for any $t \ge 2$, it holds that:   
\[
\bb^t = \bb^{t-1} + \frac{1}{n_m} [\nabla f_{S_t}(\xx^t)- \bb_{S_t}^{t-1}] 
\;.
\]
\end{lemma}
\begin{proof}
    Indeed,
    \begin{align*}
        \bb^t &= 
        \Avg \bb_i^t  
        =
        \frac{1}{n} \Bigl[
        \sum_{i \notin S_t} \bb_i^{t - 1} +
        \sum_{i \in S_t} \nabla f_i(\xx^t)\Bigr]
        =
        \frac{1}{n} \Bigl[
          \sum_{i = 1}^n \bb_i^{t - 1}
          +
          \sum_{i \in S_t} [\nabla f_i(\xx^t) - \bb_i^{t - 1}] \Bigr]
        \\
        &= \bb^{t-1} + \frac{1}{n_m} [\nabla f_{S_t}(\xx^t)- \bb_{S_t}^{t-1}] 
        \;.
        \qedhere
    \end{align*}
\end{proof}

\subsubsection{Proof of Lemma~\ref{thm:VarianceSAGAMain}}
\label{sec:ProofVarianceSAGAMain}

\begin{proof}
Let $t \ge 2$. 
By definition, 
$\mG^t = \frac{1}{m}\sum_{i \in S_t} \mG_i^t$, where 
$\mG_i^t \defeq \nabla f_i(\xx^t)-\bb_i^{t-1} + \bb^{t-1}$ and $S_t$ is independent from $\xx^t$ and $(\bb_i^{t - 1})_{i=1}^n$. 
Therefore, according to Lemma~\ref{thm:Variance-SampleWithoutReplacement}, 
we have:
    \begin{equation*}
    \E_{S_t}[\mG^t] = 
    \Avg \mG_i^t 
    =
    \nabla f(\xx^t)
    \quad 
    \text{and}
    \quad 
    \E_{S_t}[ 
        \| \mG^t 
        - \nabla f(\xx^t) \|^2 
        ]
        =
        \frac{n-m}{n-1}\frac{1}{m}
        \hat{\sigma}_{t,1}^2
        \;,
    \end{equation*}
    where $\hat{\sigma}^2_{t,1} \defeq 
    \Avg 
    \|
    (\nabla f_i(\xx^t) - \bb_i^{t-1}) - (\nabla f(\xx^t) - \bb^{t-1})
    \|^2
    $.
    Taking the expectation w.r.t. $S_{[t-1]}$ on both sides, we get:
    \[
    \sigma_t^2 = \frac{n-m}{n-1}\frac{1}{m}\E_{S_{[t-1]}}[\hat{\sigma}_{t,1}^2] 
    \defeq 
    \frac{q_m}{m} \sigma_{t,1}^2 \;.
    \]
    We next derive the recurrence for $\sigma_{t,1}^2$.
    Denote $\hat{\chi}_t \defeq 
    \|\xx^t - \xx^{t-1}\|$.
    For any $\alpha > 0$, we obtain: 
    \begin{equation*}
    \begin{split}
        \label{eq:SAGAVarianceMiddle1}
        &\quad 
        \hat{\sigma}_{t+1,1}^2
        =
        \Avg\| (\nabla f_i(\xx^{t+1}) - \bb_i^t) - (\nabla f(\xx^{t+1}) - \bb^t) \|^2
        \\
        &=
        \Avg \bigl\lVert
        (\nabla f_i(\xx^t) - \bb_i^t)
        - (\nabla f(\xx^t) - \bb^t)
        + [
        \nabla h_i(\xx^t) - \nabla h_i(\xx^{t+1}) 
        ]
        \bigr\rVert^2 
        \\
        &\stackrel{\eqref{eq:BasicInequality2},\eqref{eq:delta}}{\le} (1 + \alpha)
        \Avg \|(\nabla f_i(\xx^t) - \bb_i^t)
        - (\nabla f(\xx^t) - \bb^t)
        \|^2
        + \Bigl(1 + \frac{1}{\alpha}\Bigr)
        \delta^2
        \hat{\chi}_{t+1}^2
        \\
        &=(1 + \alpha)
        \Bigl[
        \Avg \|\nabla f_i(\xx^t) - \bb_i^t
        - \nabla f(\xx^t) 
        \|^2
        - \| \bb^t \|^2
        \Bigr]
        + \Bigl(1 + \frac{1}{\alpha}\Bigr)
        \delta^2
        \hat{\chi}_{t+1}^2
        \\
        &=(1 + \alpha)
        \Bigl[
        \frac{1}{n_m}
        \|\nabla f(\xx^t)\|^2
        + \frac{1}{n} \sum_{i \notin S_t}
        \|\nabla f_i(\xx^t)
        - \bb_i^{t-1} 
        - \nabla f(\xx^t)\|^2
        - \| \bb^t \|^2
        \Bigr]
        + \Bigl(1 + \frac{1}{\alpha}\Bigr)
        \delta^2
        \hat{\chi}_{t+1}^2
        \;,
    \end{split}
    \end{equation*}   
    where the last second equality follows from the identity $\Avg [\nabla f_i(\xx^t) - \bb_i^t
        - \nabla f(\xx^t)] = \bb^t$, and the last equality follows from the definition of $\bb^t_i$.
    Further note that
    \begin{align*}
        &\E_{S_t}
        \Bigl[ 
        \frac{1}{n} \sum_{i \notin S_t}
        \|\nabla f_i(\xx^t)
        - \bb_i^{t-1} 
        - \nabla f(\xx^t)\|^2
        \Bigr]
        = 
        \frac{1}{n} \sum_{i=1}^n \mathbb{P}(i \notin S_t)
        \|\nabla f_i(\xx^t)
        - \bb_i^{t-1} 
        - \nabla f(\xx^t)\|^2
        \\
        &=\Bigl(
        1-\frac{1}{n_m} \Bigr)\frac{1}{n}\sum_{i=1}^n \|\nabla f_i(\xx^t)
        - \bb_i^{t-1} 
        - \nabla f(\xx^t)\|^2
        =\Bigl(
        1-\frac{1}{n_m} \Bigr)
        [\hat{\sigma}_{t,1}^2 + \|\bb^{t-1}\|^2]
    \end{align*}
    Taking the expectation w.r.t. $S_t$ on both sides of the last second display and plugging in this identity, we obtain:
    \begin{align*}
    \E_{S_t}[\hat{\sigma}_{t+1,1}^2
    + (1+\alpha) \| \bb^t \|^2]
    \le (1+\alpha)
    \Bigl(
    1-\frac{1}{n_m}
    \Bigr)
    \bigl[
    \hat{\sigma}_{t,1}^2
    &+ \| \bb^{t-1} \|^2
    \bigr]
    + (1+\alpha)\frac{1}{n_m}\| \nabla f(\xx^t) \|^2
    \\
    &+ \Bigl(1+\frac{1}{\alpha}
    \Bigr)\delta^2 
    \E_{S_t}[\hat{\chi}_{t+1}^2] \;.
    \end{align*}
    Taking the expectation w.r.t. $S_{[t-1]}$ on both sides and denoting $\mB_t^2 \defeq  (1+\alpha)\E_{S_{[t]}}[\|\bb^t\|^2]$, we get:
    \[
    \sigma_{t+1,1}^2 + \mB_t^2
    \le (1+\alpha)
    \Bigl(1-\frac{1}{n_m}\Bigr)
    [\sigma_{t,1}^2 + \mB_{t-1}^2]
    +\frac{1+\alpha}{n_m}
    G_t^2
    +(1+\frac{1}{\alpha})\delta^2 
    \chi_{t+1}^2
    \;.
    \]
    Let $1-\gamma / n_m \defeq (1+\alpha)(1-1/n_m) \in (1-1/n_m,1)$.
    We then have:
    $\gamma \in (0, 1)$, 
    $1+\alpha = \frac{n_m - \gamma}{n_m - 1}$ and 
    $1 + \frac{1}{\alpha} = 
    \frac{n_m - \gamma}{1 - \gamma}$. 
    The previous display
    can thus be reformulated as:
    \[
    \sigma_{t+1,1}^2 + \mB_t^2
    \le 
    \Bigl(1-\frac{\gamma}{n_m}\Bigr)
    [\sigma_{t,1}^2 + \mB_{t-1}^2]
    +\frac{n_m - \gamma}{n_m(n_m-1)}
    G_t^2
    +\frac{n_m - \gamma}{1-\gamma}\delta^2 
    \chi_{t+1}^2
    \;.
    \] 
    Let $T \ge 3$.
    Applying Lemma~\ref{thm:SimpleRecurrence1} (starting from $t = 2$), we have:
    \[
    \sum_{t=3}^T [\sigma_{t,1}^2 + \mB_{t-1}^2 ] 
    \le \frac{1-\gamma / n_m}{\gamma / n_m}[\sigma_{2,1}^2 + \mB_{1}^2] + \frac{1}{\gamma / n_m}
    \sum_{t=2}^{T-1} 
    \Bigl[
    \frac{n_m - \gamma}{n_m(n_m-1)}
    G_t^2
    +\frac{n_m - \gamma}{1-\gamma}\delta^2
    \chi_{t+1}^2
    \Bigr]
     \;.
    \]
    Adding 
    $\sigma_{2,1}^2$ to both sides and dropping the non-negative $\mB_t^2$, we obtain:
    \[
    \sum_{t=2}^T \sigma_{t,1}^2 
    \le 
    \frac{n_m}{\gamma}
    \sigma_{2,1}^2
    +
    \frac{n_m-\gamma}{\gamma}\mB_{1}^2 
    +\frac{n_m - \gamma}{\gamma(n_m -1)}\sum_{t=2}^{T-1} G_t^2
    + \frac{n_m(n_m - \gamma)}{\gamma(1-\gamma)}\delta^2
    \sum_{t=2}^{T-1} 
    \chi_{t+1}^2
     \;.
    \]
    Recall that $\bb_i^1 = \nabla f_i(\xx^1)$ for all $i \in [n]$. 
    It holds that:
    \[
    \sigma_{2,1}^2 
    =  \hat{\sigma}_{2,1}^2
    = 
    \Avg [\|  
    (\nabla f_i(\xx^2) - \nabla f_i(\xx^1)) - (\nabla f(\xx^2) - \nabla f(\xx^1))
    \|^2 
    \stackrel{\eqref{assump:SOD}}{\le} 
    \delta^2 \hat{\chi}_2^2 = \delta^2 \chi_2^2, \;
    \mB_1^2 = \frac{n_m - \gamma}{n_m-1} G_1^2 \;.
    \]
    It follows that:
    \[
    \sum_{t=2}^T \sigma_{t,1}^2 
    \le 
    \frac{(n_m-\gamma)^2}{\gamma(n_m - 1)}G_{1}^2 
    +\frac{n_m - \gamma}{\gamma(n_m -1)}\sum_{t=2}^{T-1} G_t^2
    + \frac{n_m(n_m - \gamma)}{\gamma(1-\gamma)}\delta^2
    \sum_{t=1}^{T-1} 
    \chi_{t+1}^2
     \;.
    \]
    Let us choose $\gamma$ which minimizes the coefficient in front of $\sum_{t=1}^{T-1}\chi_{t+1}^2$ over $(0,1)$.
    By Lemma~\ref{thm:MinimizerOfGamma},
    we get $\gamma^\star = n_m - \sqrt{n_m^2-n_m}$. 
    Substituting $\gamma = \gamma^\star$, we have:
    \begin{align*}
    \sum_{t=2}^T \sigma_{t,1}^2 
    &\le 
    (\sqrt{n_m^2-n_m}+n_m)G_{1}^2 
    +\Bigl(
    1+\frac{\sqrt{n_m}}{\sqrt{n_m-1}} \Bigr)\sum_{t=2}^{T-1} G_t^2
    + n_m(\sqrt{n_m} + \sqrt{n_m-1})^2\delta^2
    \sum_{t=1}^{T-1} 
    \chi_{t+1}^2
    \\
    &\le 2 n_m G_1^2
    + \Bigl(
    1+\frac{\sqrt{n_m}}{\sqrt{n_m-1}} \Bigr)
    \sum_{t=2}^{T-1} G_t^2
    + 4 n_m^2\delta^2
    \sum_{t=1}^{T-1} 
    \chi_{t+1}^2
    \;.
    \end{align*}
    Multiplying both sides by $\frac{q_m}{m}$, substituting the identity 
    $\sigma_t^2 = \frac{q_m}{m}\sigma_{t,1}^2$ and  
    $
    \frac{q_m}{m}\bigl(
    1+\frac{\sqrt{n_m}}{\sqrt{n_m-1}} \bigr) 
    =\frac{n-m}{n-1}\frac{1}{m}
    +\frac{\sqrt{n-m}\sqrt{n}}{m(n-1)}
    $, we obtain:
    \[
    \sum_{t=2}^T \sigma_{t}^2 
    \le 
    \frac{2n_mq_m}{m} G_1^2
    +
    \frac{n_m-1+\sqrt{n^2_m-n_m}}{(n-1)}
    \sum_{t=2}^{T-1} 
    G_t^2
    +4n_m^2\delta_m^2
    \sum_{t=2}^T
    \chi_{t}^2
     \;.
    \]
    Adding $\sigma_0^2=0$ and $\sigma_1^2=0$ to both sides, we prove the variance bound for $T \ge 3$, since  $\mG^0 = \nabla f(\xx^0)$ and $\mG^1 = \nabla f(\xx^1)$. 
    The same inequality also holds for $T = 1$ and $T = 2$, since 
    $\sigma_0^2=\sigma_1^2=0$ and 
    $\sigma_2^2 = \frac{q_m}{m}\sigma_{2,1}^2 \le \frac{q_m}{m}\delta^2 \chi_2^2$.
    
\end{proof}

\subsubsection{Proof of Lemma~\ref{thm:VarianceSVRGMain}}
\label{sec:VarianceSVRGMain}

\begin{proof}
    Let $t \ge 1$.
    By definition, 
    $\mG^t = \frac{1}{m}\sum_{i \in S_t} \mG_i^t$, where 
    $\mG_i^t \defeq \nabla f_i(\xx^t)-\nabla f_i(\ww^t) + \nabla f(\ww^t)$ and $S_t$ is independent from $\xx^t$ and $\ww^t$. 
    Therefore, according to Lemma~\ref{thm:Variance-SampleWithoutReplacement}, 
    we have:
    \[
    \E_{S_t}[\mG^t] 
    = \Avg \mG_i^t
    = 
    \nabla f(\xx^t)
    \quad 
    \text{and}
    \quad 
    \E_{S_t}[ 
        \| \mG^t - \nabla f(\xx^t) \|^2] 
        = \frac{n-m}{n-1}\frac{1}{m} \hat{\sigma}_{t,1}^2
        \;,
    \]
    where $\hat{\sigma}_{t,1}^2 \defeq 
    \Avg \| \nabla h_i(\xx^t) - \nabla h_i(\ww^t) \|^2$.
    Since $\omega_{t+1}$ is independent of $\xx^{t+1}$ and $\ww^t$, we have for any $\alpha > 0$:
    \begin{align*}
        \E_{\omega_{t+1}}[ 
         \hat{\sigma}_{t+1,1}^2]
        &=(1-p_B) \Avg \| \nabla h_i (\xx^{t+1}) - \nabla h_i (\ww^{t}) \|^2
        \\
        &\stackrel{\eqref{eq:delta},\eqref{eq:BasicInequality2}}{\le}
        (1-p_B)(1+\alpha)
        \hat{\sigma}_{t,1}^2
        + (1-p_B)\Bigl(1+\frac{1}{\alpha}\Bigr)\delta^2
        \hat{\chi}_{t+1}^2 \;.
    \end{align*}
    where $\hat{\chi}_{t+1} \defeq \|\xx^{t+1} - \xx^t\|$.
    Let $1-\gamma p_B \defeq (1-p_B)(1+\alpha) \in (1-p_B, 1)$. We then have 
    $\gamma \in (0,1)$ and 
    $1+1/\alpha = \frac{1-p_B\gamma}{p_B(1-\gamma)}$.
    Therefore, the previous display can be reformulated as:
    \[
    \E_{\omega_{t+1}}[ 
         \hat{\sigma}_{t+1,1}^2]
         \le (1-\gamma p_B)
         \hat{\sigma}_{t,1}^2
         + \frac{(1-p_B)(1-p_B\gamma)}{p_B(1-\gamma)} \delta^2 \hat{\chi}_{t+1}^2 \;.
    \]
    Taking the expectation w.r.t,
    $\omega_{[t]}$ on both sides and denoting $\sigma_{t,1}^2 \defeq \E_{\omega_{[t]}}[\hat{\sigma}_{t,1}^2]$,
    we have:
    \[
    \sigma_{t+1,1}^2 
    \le (1-\gamma p_B)\sigma_{t,1}^2 + 
    \frac{(1-p_B)(1-p_B\gamma)}{p_B(1-\gamma)} \delta^2 \chi_{t+1}^2 \;.
    \]
   Let $T \ge 2$.
   Applying Lemma~\ref{thm:SimpleRecurrence1} (starting from $t = 1$), we obtain:
    \[
    \sum_{t=2}^T \sigma_{t,1}^2
    \le \frac{1-\gamma p_B}{\gamma p_B} \sigma_{1,1}^2
    + 
    \frac{(1-p_B)(1-p_B\gamma)}{p_B^2\gamma (1-\gamma)} \delta^2 
    \sum_{t=1}^{T-1}
    \chi_{t+1}^2 \;.
    \]
    Adding $\sigma_{1,1}^2$ to both sides and using 
    $\sigma_{1,1}^2 = \E_{\omega_1}[\hat{\sigma}_{1,1}^2]
    \le (1-p_B) \delta^2 \hat{\chi}_1^2 = (1-p_B) \delta^2 \chi_1^2$, we obtain:
    \begin{align*}
    \sum_{t=1}^T \sigma_{t,1}^2
    &\le \frac{1}{\gamma p_B} 
    (1-p_B) \delta^2 \chi_1^2
    + 
    \frac{(1-p_B)(1-p_B\gamma)}{p_B^2\gamma (1-\gamma)} \delta^2 
    \sum_{t=1}^{T-1}
    \chi_{t+1}^2 
    \\
    &\le \frac{(1-p_B)(1-p_B\gamma)}{p_B^2\gamma (1-\gamma)} \delta^2  \sum_{t=0}^{T-1}
    \chi_{t+1}^2 \;.
    \end{align*}
    According to Lemma~\ref{thm:MinimizerOfGamma},
    the minimizer of 
    $\frac{(1-p_B)(1-p_B\gamma)}{p_B^2\gamma (1-\gamma)}$
    over $\gamma \in (0,1)$ is 
    $\gamma^\star = \frac{1-\sqrt{1-p_B}}{p_B}$.
    Substituting $\gamma = \gamma^\star$, we get:
    \[
    \sum_{t=1}^T \sigma_{t,1}^2
    \le \frac{(1-p_B) \delta^2}{(1-\sqrt{1-p_B})^2}  \sum_{t=1}^{T}
    \chi_{t}^2 
    \;.
    \]
    Multiplying both sides by $\frac{q_m}{m}$ and 
    using the identity 
    $\sigma_t^2 = \E_{S_t, \omega_{[t]}}[\|\mG^t - \nabla f(\xx^t)\|^2] = \frac{q_m}{m}\E_{\omega_{[t]}}[\hat{\sigma}_{t,1}^2] = \frac{q_m}{m} \sigma_{t,1}^2$, we have:
    \[
    \sum_{t=1}^T \sigma_{t}^2
    \le \frac{(1-p_B) \delta^2_m}{(1-\sqrt{1-p_B})^2}  \sum_{t=1}^{T}
    \chi_{t}^2 
    \le \frac{4\delta_m^2}{p_B^2}
    \sum_{t=1}^T \chi_t^2
    \;.
    \]    
    Adding $\sigma_0^2 = \| \mG^0-\nabla f(\xx^0) \|^2=0$ to both sides, we get the variance bound for $T \ge 2$.
    The same bound holds for $T = 1$ since $\sigma_0^2 = 0$ and 
    $\sigma_1^2 = \frac{q_m}{m} \sigma_{1,1}^2 \le (1-p_B)\delta_m^2 \chi_1^2$.
\end{proof}

\begin{theorem}
\label{thm:I-CGM-SVRG-Main}
    Let~\ref{Alg:PP} be applied to Problem~\ref{eq:problem} with 
    the \ref{Alg:SVRG-update} estimator
    under Assumption~\ref{assump:ED} and~\ref{assump:SOD}. 
    Suppose the inaccuracies in solving the subproblems satisfy~\eqref{eq:AccuracyCondition}. Then by choosing $\lambda = 3\Delta_1 + 16\delta_m / p_B$, after $T = \lceil \frac{(256(\Delta_1 + 6 \delta_m / p_B)F^0}{\epsilon^2} \rceil$ iterations, we have 
    $\E[\|\nabla f(\bar{\xx}^T)\|^2] \le \epsilon^2$, where $\Bar{\xx}^T$ is uniformly sampled from $(\xx^t)_{t=1}^T$. By choosing $p_B = \frac{C_R}{C_A \lceil n_m \rceil}$
    The communication complexity  is at most $
    C_A \lceil n_m \rceil
    +
    (2C_R + 1) \bigl\lceil \frac{(256(\Delta_1 + 6 \delta_m C_A\lceil n_m \rceil/C_R) F^0}{\epsilon^2} \bigr\rceil$.
\end{theorem}
\begin{proof}
    The proof strategy is the same as the one for Theorem~\ref{thm:I-CGM-RG-SVRG-Main}.
\end{proof}

\subsection{Properties of the RG Estimator}

\subsubsection{Proof of Lemma~\ref{thm:VarianceQr}}
\label{sec:thm:VarianceQr}
\begin{proof}
    Let $t \ge 0$.
    By the definition of~\ref{Alg:RG-update}, we obtain:
    \begin{align*}
        \hat{\Sigma}_{t+1}^2 
        &= \|\gg^{t+1} - \nabla f(\xx^{t+1})\|^2
        \\
        &=
        \norm{(1-\beta) \gg^t + \beta \mG^t
        +\nabla f_{S_t} (\xx^{t+1}) - \nabla f_{S_t}(\xx^t) -\nabla f(\xx^{t+1})
        }^2
        \\
        &=
        \norm{(1-\beta) \bigl(
        \gg^t - \nabla f(\xx^t)
        \bigr) + \beta \bigl(
        \mG^t - \nabla f(\xx^t) \bigr)
        + \bigl( \nabla h_{S_t} (\xx^t)
        - \nabla h_{S_t} (\xx^{t+1}) \bigr)
        }^2 
        \\
        &= 
        (1-\beta)^2 \hat{\Sigma}_t^2
        + 
        \norm{\beta \bigl(
        \mG^t - \nabla f(\xx^t) \bigr)
        + \bigl( 
        \nabla h_{S_t} (\xx^t)
        - \nabla h_{S_t} (\xx^{t+1})\bigr) }^2 
        \\
        &\qquad 
        + 2(1-\beta)
        \lin{\gg^t - \nabla f(\xx^t), 
        \beta \bigl(
        \mG^t - \nabla f(\xx^t) \bigr)
        + \bigl(
        \nabla h_{S_t} (\xx^t)
        - \nabla h_{S_t} (\xx^{t+1}) \bigr) 
        }
        \;.
    \end{align*}
    By Assumption~\ref{assump:unbiased-G}, $S_t$ is independent of  $\xx^t$, $\xx^{t+1}$ and $\mG^{t-1}$. Furthermore, since $\gg^t$ is a deterministic function of $\mG^{t-1}$, $\xx^t$, $\xx^{t-1}$, $S_{t-1}$ and $\gg^{t-1}$, by induction, $S_t$ is also independent of $\gg^t$.
    Hence, it holds that:
    \begin{align*}
        &\quad 
        \E_{S_t}[
        \lin{\gg^t - \nabla f(\xx^t), 
        \beta \bigl(
        \mG^t - \nabla f(\xx^t) \bigr)
        + \nabla h_{S_t} (\xx^t)
        - \nabla h_{S_t} (\xx^{t+1})}
        ]
        \\
        &=\lin{\gg^t - \nabla f(\xx^t), 
        \beta \E_{S_t}[ 
        \mG^t - \nabla f(\xx^t)  ]
        + \E_{S_t}[ \nabla h_{S_t} (\xx^t)
        - \nabla h_{S_t} (\xx^{t+1}) ]}\;.
    \end{align*}
    By Assumption~\ref{assump:unbiased-G}, we have $\E_{S_t}[\mG^t] = \nabla f(\xx^t)$. 
    Using Lemma~\ref{thm:Variance-SampleWithoutReplacement}, we have
    $
    \E_{S_t}[ \nabla h_{S_t} (\xx^t)
        - \nabla h_{S_t} (\xx^{t+1}) ] 
    =\Avg [ \nabla h_i(\xx^t) - \nabla h_i (\xx^{t+1}) ] = 0 
    $ and
    \[
    \E_{S_t}[\|\nabla h_{S_t} (\xx^t)
        - \nabla h_{S_t} (\xx^{t+1})\|^2]
        \stackrel{\eqref{eq:Variance-SampleWithoutReplacement}}{=}
        \frac{q_m}{m}  
        \Avg \|\nabla h_i (\xx^t) - \nabla h_i(\xx^{t+1})\|^2 
        \stackrel{\eqref{assump:SOD}}{\le}  \delta_m^2 \hat{\chi}_{t+1}^2 \;.
    \]
    Taking the expectation w.r.t. $S_t$ on both sides of the first display, we get:
    \begin{align*}
        &\quad 
        \E_{S_{t}}\bigl[
        \hat{\Sigma}_{t+1}^2 
        \bigr]   
        \\
        &= 
        (1-\beta)^2  \hat{\Sigma}_t^2
        + 
        \E_{S_t}[\| \beta \bigl(
        \mG^t - \nabla f(\xx^t) \bigr)
        + \bigl( \nabla h_{S_t} (\xx^t)
        - \nabla h_{S_t} (\xx^{t+1})\bigr) \|^2 ]
        \\
        &\le 
        (1-\beta)^2 \hat{\Sigma}_t^2 
        + 2\beta^2 \E_{S_t}[  
        \| \mG^t - \nabla f(\xx^t) \|^2]
        + 2 \delta_m^2 \hat{\chi}_{t+1}^2 \;.
    \end{align*}
    Taking the expectation
    w.r.t. $S_{[t-1]}$
    on both sides and substituting the notations, we get:
    \[
    \Sigma_{t+1}^2 
    \le (1-\beta)^2 
    \Sigma_{t}^2
    + 2\beta^2 \sigma_t^2
    + 2\delta_m^2 
    \chi_{t+1}^2
    \;.
    \]
    
    Applying Lemma~\ref{thm:SimpleRecurrence1}, we get for any $T \ge 1$:
    \begin{align*}
    \sum_{t=1}^T \Sigma_{t}^2
    &\le
    \frac{(1-\beta)^2}{2\beta - \beta^2} \Sigma_0^2 
    + \frac{2\beta}{2-\beta} \sum_{t=0}^{T-1}
    \sigma_t^2 
    + \frac{2\delta_m^2}{2\beta-\beta^2} \sum_{t=0}^{T-1}
    \chi_{t+1}^2 \;.
    \end{align*}
    This proves the claim since $\gg^0 = \nabla f(\xx^0)$ and so $\Sigma_0^2 = 0$. 
\end{proof}

\subsubsection{Proof of Corollary~\ref{thm:Variance-SAGA-RG}.}

\begin{proof}
    Let $T \ge 1$. 
    Note that, under Assumption~\ref{assump:unbiased-G}, we have $ \E_{S_{[t-1]}}[\|\xx^t - \xx^{t-1}\|]^2 = 
    \E_{S_{[t-2]}}[\|\xx^t - \xx^{t-1}\|]^2 = \chi_t^2$ and 
    $\E_{S_{[t-1]}}[\|\nabla f(\xx^t)\|^2] = \E_{S_{[t-2]}}[\|\nabla f(\xx^t)\|^2] = G_t^2$.
    Applying Lemma~\ref{thm:VarianceSAGAMain}, we have:
    \begin{equation*}
    \sum_{t=0}^{T} \sigma_t^2
    \le
    \frac{2n_mq_m}{m} G_1^2
    +
    \frac{n_m-1+\sqrt{n^2_m-n_m}}{(n-1)}
    \sum_{t=2}^{T-1} 
    G_t^2
    +4n_m^2\delta_m^2
    \sum_{t=2}^T
    \chi_{t}^2 \;.
    \end{equation*}
    Applying Lemma~\ref{thm:VarianceQr}
    , we obtain:
    \begin{equation*}
        \sum_{t=0}^T \Sigma_{t}^2
        \le 
        \frac{2\beta}{2-\beta} \sum_{t=0}^{T-1}
        \sigma_t^2 
        + \frac{2\delta_m^2}{2\beta-\beta^2} \sum_{t=1}^{T}
        \chi_{t}^2
        \;.
    \end{equation*}
    Combining the previous two displays, we have:
    \[
    \sum_{t=0}^T \Sigma_t^2 
    \le 
    \frac{4\beta n_m q_m}{(2-\beta)m} G_1^2 + 
    \frac{2\beta( n_m-1+\sqrt{n^2_m-n_m} )}{(2-\beta)(n-1)}
    \sum_{t=2}^{T-1} 
    G_t^2
    +\frac{8\beta^2n_m^2\delta_m^2 +2\delta_m^2}{2\beta - \beta^2} \sum_{t=1}^T \chi_t^2
    \;.
    \qedhere
    \]
\end{proof}

\subsubsection{Proof of Corollary~\ref{thm:Variance-SVRG-RG}.}

\begin{proof}
    Let $T \ge 1$.
    Applying Lemma~\ref{thm:VarianceSVRGMain}, we get:
    \[
    \sum_{t=0}^T \E_{S_t, \omega_{[t]}}[\|\mG^t - \nabla f(\xx^t)\|^2] 
    \le \frac{4\delta_m^2}{p_B^2}
    \sum_{t=1}^T \E_{\omega_{[t-1]}}[\|\xx^t - \xx^{t-1}\|^2] \;.
    \]
    Note that, under Assumption~\ref{assump:unbiased-G}, $S_{t-1}$ is independent of $\xx^t$ and $\xx^{t-1}$. 
    Moreover, the first iterate that might depend on $\omega_{t-1}$ is $\xx^{t+1}$
    since $\gg^t$ is computed using $\xx^t$ and $\mG^{t-1}$ which is a function of $\omega_{t-1}$. Therefore,
    $\omega_{t-1}$ is also independent of $\xx^t$ and $\xx^{t-1}$. Hence, we have $ \E_{S_{[t-1]},\omega_{[t-1]}}[\|\xx^t - \xx^{t-1}\|^2] = 
    \E_{S_{[t-2]},\omega_{[t-2]}}[\hat{\chi}_t^2] = \chi_t^2$.
    Taking the expectation w.r.t. $S_{[t-1]}$ on both sides of the first display, we get:
    \[
    \sum_{t=0}^T \sigma_t^2
    \le \frac{4\delta_m^2}{p_B^2}
    \sum_{t=1}^T \chi_t^2 \;.
    \]
    where $\sigma_t^2 \defeq \E_{S_{[t]}, \omega_{[t]}}[\|\mG^t - \nabla f(\xx^t)\|^2] $.
    Applying Lemma~\ref{thm:VarianceQr}, taking the expectation w.r.t. $\omega_{[t]}$ and 
    substituting the identity 
    $\E_{S_{[t-1]},\omega_{[t]}}[\hat{\Sigma}_t^2] = \E_{S_{[t-1]},\omega_{[t-1]}}[\hat{\Sigma}_t^2] = \Sigma_t^2$ and 
    $\E_{S_{[t-2]},\omega_{[t]}}[\hat{\chi}_t^2] = \E_{S_{[t-2]},\omega_{[t-2]}}[\hat{\chi}_t^2]=\chi_t^2$,
    we obtain:
    \begin{equation*}
        \sum_{t=0}^T \Sigma_{t}^2
        \le 
        \frac{2\beta}{2-\beta} \sum_{t=0}^{T-1}
        \sigma_t^2 
        + \frac{2\delta_m^2}{2\beta-\beta^2} \sum_{t=1}^{T}
        \chi_{t}^2
        \;.
    \end{equation*}
    Combining the previous two displays, we have:
    \[
    \sum_{t=0}^T \Sigma_t^2 
    \le 
    \frac{8\beta^2\delta_m^2/p_B^2 +2\delta_m^2}{2\beta - \beta^2}
    \sum_{t=1}^T \chi_t^2 \;.
    \qedhere 
    \]
\end{proof}

\subsection{Proofs for~\ref{Alg:PP} with \ref{Alg:RG-update}-\ref{Alg:SAGA-update}}
\label{sec:proof-PP-SAGA}

\begin{lemma}
\label{thm:random-dependence-saga}
Let $\xx^t$ be the iterates of \ref{Alg:PP}-\ref{Alg:RG-update}-\ref{Alg:SAGA-update} and let $\mG^t$ be the-\ref{Alg:SAGA-update} estimator for all $t \ge 0$. Let $\zeta_t$ denote the randomness generated during the process of solving the subproblem $F_{t-1}$ in~\ref{Alg:PP} for any $t \ge 1$.
Assume that $\{\zeta_t\}_{t=1}^{\infty}$ are mutually independent  across $t$. 
Then the iterates $\{\xx^t\}_{t=0}^{\infty}$ and the estimators $\{ \mG^t \}_{t=0}^{\infty}$ satisfy Assumption~\ref{assump:unbiased-G}.
\end{lemma}
\begin{proof}
The equation $\E_{S_t}[\mG^t] = \nabla f(\xx^t)$ has been proved in Lemma~\ref{thm:VarianceSAGAMain}. We next verify the dependency of randomness. Let $t \ge 1$ and denote $S_{[t]} \defeq (S_0,\ldots,S_t)$ and $\zeta_{[t]}\defeq (\zeta_1,\ldots,\zeta_t)$.
Assume that $\xx^t$ is a deterministic function of  $(S_{[t-2]}, \zeta_{[t]})$. 
Then $\mG^t$ is a deterministic function of $(S_{[t]}, \zeta_{[t]})$ since $\mG^t$ depends only on $\xx_{[t]} \defeq (\xx^0,\xx^1,...,\xx^t)$ and $S_t$.
Next observe that $\gg^{t-1}$ is a function of $S_{t-2}$, $\xx^{t-1}$, $\xx^{t-2}$, $\mG^{t-2}$ and $\gg^{t-2}$.
Therefore, $\gg^{t-1}$ is a deterministic function of $S_{[t-2]}$ and $\zeta_{[t-1]}$. Finally, from the update rule of \ref{Alg:PP}, $\xx^{t}$ is a deterministic function of $\gg^{t-1}$, $\zeta_t$ and $\xx^{t-1}$. Therefore, the assumption that $\xx^t$ is determinsitic conditioned on $(S_{[t-2]},\zeta_{[t]})$ is satisfied. 
This implies that $S_t$ is independent of $\xx_{[t+1]}$, $\mG^{t-1},\ldots,\mG^0$.
\end{proof}

\subsubsection{Proof of Theorem~\ref{thm:PP-SAGA-main-paper}.}
\label{sec:thm:PP-SAGA-main-paper}

\begin{proof}
According to Lemma~\ref{thm:UpperboundKr-random}, by choosing $p = \frac{\lambda - \Delta_1}{8(L_1 + \lambda)}$, the accuracy condition~\eqref{eq:AccuracyCondition} for solving the subproblems is satisfied. 
Applying Corollary~\ref{thm:Variance-SAGA-RG} and taking the full expectation, for any $T \ge 1$, we have:
\begin{align*}
    \sum_{t=0}^T \Sigma_t^2 
    &\le 
    \frac{4\beta q_m n_m}{(2-\beta)m} G_1^2 + 
    \frac{2\beta( n_m-1+\sqrt{n^2_m-n_m} )}{(2-\beta)(n-1)}
    \sum_{t=2}^{T-1} 
    G_t^2
    +\frac{8\beta^2n_m^2\delta_m^2 +2\delta_m^2}{2\beta - \beta^2} \sum_{t=1}^T \chi_t^2
    \;,
\end{align*}
where $\Sigma_t^2$, $G_t^2$ and $\chi_t^2$ are defined in Corollary~\ref{thm:CGM-Main-Corollary}. 
Using $\frac{1}{2-\beta} \le 1$,
$\frac{q_m}{m} \le 1$,
$\frac{n_m-1}{n-1} \le 1 \le n_m$,
and
$\frac{\sqrt{n_m^2 - n_m}}{n-1}
\le n_m$ as $n \ge 2$, we get:
\[
\sum_{t=0}^T \Sigma_t^2 
    \le
    4\beta n_m \sum_{t=1}^{T-1} G_t^2 + \Bigl( 8\beta n_m^2 \delta_m^2 + \frac{2\delta_m^2}{\beta} \Bigr)
    \sum_{t=1}^T \chi_t^2 \;.
\]
Let $\lambda = \frac{1}{a} \Delta_1 + b \sqrt{n_m} \delta_m$ and $\beta = \frac{1}{c n_m}$ where $0 < a < 1$ and $b,c > 0$.
To achieve the error condition~\eqref{eq:EstErrorCondition}, the constants should satisfy:
\[
\Bigl( 
    \frac{12(\lambda + \Delta_1)^2}{(\lambda - \Delta_1)^2}
    +8
    \Bigr) 4 \beta n_m 
    \le 
    \Bigl( \frac{12(1+a)^2}{(1-a)^2} + 8 \Bigr)\frac{4}{c}
    \le
    \frac{1}{2}\;,
\]
and
\[
\Bigl( 
    \frac{12(\lambda + \Delta_1)^2}{(\lambda - \Delta_1)^2}
    +8
    \Bigr)
\Bigl(8\beta n_m^2\delta_m^2
    +
    \frac{2\delta_m^2}{\beta}
    \Bigr) \le 
    \Bigl( \frac{12(1+a)^2}{(1-a)^2} + 8 \Bigr)
    (8/c + 2c)n_m \delta_m^2
    \le b^2 n_m \delta_m^2
    \le 
    (\lambda + \Delta_1)^2 \;,
\] 
which gives:
\begin{equation}
\label{eq:Const-ICGM-SAGA}
\Bigl( \frac{12(1+a)^2}{(1-a)^2} + 8 \Bigr) \le \frac{c}{2}, 
\quad 
\Bigl( \frac{12(1+a)^2}{(1-a)^2} + 8 \Bigr) (8/c + 2c) \le b^2 \;.
\end{equation}
Let $a,b,c$ satisfy~\eqref{eq:Const-ICGM-SAGA}.
We can apply Corollary~\ref{thm:CGM-Main-Corollary} and obtain:
\[
    \E[\|\nabla f(\bar{\xx}^T)\|^2]   
    \le \frac{32(\lambda + \Delta_1)^2}{\lambda - \Delta_1} \frac{F^0}{T}
    \le \frac{32(1+a)^2}{1-a}
    \Bigl(
    \frac{1}{a}\Delta_1 
    +b\sqrt{n_m}\delta_m
    \Bigr)\frac{F^0}{T}
    \;.
\]
Minimizing the coefficient in front of $\Delta_1$ gives $a^\star = \frac{1}{3}$. Choosing 
$b = 113$ and $c = 112$, the condition~\eqref{eq:Const-ICGM-SAGA} is satisfied and we have:
\[
\E[\|\nabla f(\bar{\xx}^T)\|^2]   
    \le \frac{256 (\Delta_1 + 38 \sqrt{n_m}\delta_m)F^0}{T} \;.
\]
Therefore, to achieve $\E[\|\nabla f(\bar{\xx}^T)\|^2] \le \epsilon^2$, we need at most 
$T = \lceil \frac{(256(\Delta_1 + 38 \sqrt{n_m}\delta_m)F^0}{\epsilon^2} \rceil$ iterations.
We next compute the communication and local complexity. At the beginning when $t = 0$ and $t = 1$, we need $2 \lceil n_m \rceil$ communication rounds with $\ASS$ to compute two full gradients and the associated local complexity is $1$ for each round. Additionally, $2$ communication rounds with $\DSS$ are needed to compute $\xx^1$ and $\xx^2$, where the local complexity is $\frac{1}{p}$ for each round. For subsequent iterations $t \ge 2$, one communication round with $\RSS$ is needed for updating $\gg^t$
and its associated local complexity is $2$ since each client in $S_{t-1}$ needs to compute $\nabla f_i(\xx^t)$ and 
$\nabla f_i(\xx^{t-1})$. Then another round with $\DSS$ is required to compute the next iterate, where the local complexity is $\frac{1}{p}$.
    Therefore, the total communication complexity 
    is at most:
    \begin{align*}
    N(\epsilon) &= \E[C_A N_{A} + C_R N_{R} + N_{D}]
    \\
    &\le 
    2 C_A \lceil n_m \rceil + C_R T + T 
    \\
    &= 
    2 C_A \lceil n_m \rceil
    + (C_R+1) \biggl\lceil \frac{(256(\Delta_1 + 38 \sqrt{n_m}\delta_m)F^0}{\epsilon^2} \biggr\rceil \;.
    \end{align*}
The local complexity is bounded by: 
\begin{align*} 
K(\epsilon) 
&= \E[N_{A} + N_{D} / p + 2N_{R}] 
\\
&\le 2\lceil n_m\rceil + \frac{1}{p} T + 2T 
\\
&= 2\lceil n_m\rceil + \Bigl( 2+\frac{8(L_1+\lambda)}{\lambda - \Delta_1} \Bigr) T 
\\
&\le 
2\lceil n_m\rceil + \frac{28\Delta_1+ 1130\sqrt{n_m}\delta_m + 8L_1}{2\Delta_1 + 113\sqrt{n_m}\delta_m}  
\Bigl( \frac{(256(\Delta_1 + 38 \sqrt{n_m}\delta_m)F^0}{\epsilon^2} + 1 \Bigr)
\\
&\le 2\lceil n_m \rceil
+ \frac{512(7\Delta_1 + 283\sqrt{n_m}\delta_m + 2L_1)F^0}{\epsilon^2} + 14 + \frac{4L_1}{\Delta_1 + 28 \sqrt{n_m}\delta_m}
\;.
\qedhere
\end{align*}
\end{proof}

\subsection{Proofs for~\ref{Alg:PP} with \ref{Alg:RG-update}-\ref{Alg:SVRG-update}}
\label{sec:proof-PP-SVRG}

\begin{lemma}
\label{thm:random-dependence-svrg}
Let $\xx^t$ be the iterates of \ref{Alg:PP}-\ref{Alg:RG-update}-\ref{Alg:SVRG-update} and let $\mG^t$ be the-\ref{Alg:SVRG-update} estimator for all $t \ge 0$. Let $\zeta_t$ denote the randomness generated during the process of solving the subproblem $F_{t-1}$ in~\ref{Alg:PP} for any $t \ge 1$. 
Assume that $\{\zeta_t\}_{t=1}^{\infty}$ are mutually independent  across $t$. 
Then the iterates $\{\xx^t\}_{t=0}^{\infty}$ and the estimators $\{ \mG^t \}_{t=0}^{\infty}$ satisfy Assumption~\ref{assump:unbiased-G}.
\end{lemma}
\begin{proof}
    The equation $\E_{S_t}[\mG^t] = \nabla f(\xx^t)$ has been proved in Lemma~\ref{thm:VarianceSVRGMain}. We next verify the dependency of randomness. Let $t \ge 1$ and denote $\xx_{[t]}=(\xx^0,..,\xx^t)$,
    $\omega_{[t]}=(\omega_1,...,\omega_{t})$ 
    and 
    $\zeta_{[t]}=(\zeta_1,\ldots,\zeta_t)$.
Assume that $\xx^t$ is a deterministic function of  $(S_{[t-2]}, \omega_{[t-2]} ,\zeta_{[t]})$. 
It follows that $\ww^t$
is a deterministic function of $(S_{[t-2]},\omega_{[t]}, \zeta_{[t]})$ 
since $\ww^t$ depends only on $\xx^t$, $\ww^{t-1}$ and $\omega_t$.
Then $\mG^t$ is  deterministic conditioned on $(S_{[t]}, \omega_{[t]},\zeta_{[t]})$ since $\mG^t$ is a function of $\xx_t$, $\ww_t$ and $S_t$.
Next observe that $\gg^{t-1}$ is a function of $S_{t-2}$, $\xx^{t-1}$, $\xx^{t-2}$, $\mG^{t-2}$ and $\gg^{t-2}$.
Hence, $\gg^{t-1}$ is a deterministic function of $(S_{[t-2]}, \omega_{[t-2]},\zeta_{[t-1]})$. Finally, from the update rule of \ref{Alg:PP}, $\xx^{t}$ is a deterministic function of $\gg^{t-1}$, $\zeta_t$ and $\xx^{t-1}$. Therefore, the assumption that $\xx^t$ is deterministic conditioned on $(S_{[t-2]},\omega_{[t-2]},\zeta_{[t]})$ is satisfied. This implies that $S_t$ is independent of $\xx_{[t+1]}$, $\mG^0, \ldots, \mG^{t - 1}$.
\end{proof}

\begin{theorem}[\algoname{I-CGM-RG-SVRG}]
\label{thm:I-CGM-RG-SVRG-Main}
Let \ref{Alg:PP} be applied to Problem~\ref{eq:problem} under Assumptions~\ref{assump:ED},~\ref{assump:SOD} and~\ref{assump:L1}, where 
    $\xx^{t+1} = \operatorname{CGM}_{\operatorname{rand}}(\lambda, \hat{K}_t, \xx^t, \gg^t)$ with $\hat{K}_t \sim \operatorname{Geom}(p)$
    and $\gg^t$ is generated by the \ref{Alg:RG-update}-\ref{Alg:SVRG-update} estimator. 
    Then by choosing $\lambda = 3\Delta_1 + 22 \delta_m/\sqrt{p_B}$, $\beta = \frac{p_B}{2}$, and $p = \frac{\lambda - \Delta_1}{8(L_1 + \lambda)}$,
    after $T = \lceil \frac{(256(\Delta_1 + 8 \delta_m/\sqrt{p_B})F^0}{\epsilon^2} \rceil$ iterations, 
    we have $\E[\|\nabla f(\bar{\xx}^T)\|^2] \le \epsilon^2$, where $\Bar{\xx}^T$ is is uniformly sampled from $(\xx^t)_{t=1}^T$.
    Further let $p_B = \frac{C_R}{C_A \lceil n_m \rceil}$.
    The communication complexity is at most $
    C_A \lceil n_m \rceil
    +
    (2C_R + 1) \bigl\lceil \frac{(256(\Delta_1 + 8 \delta_m\sqrt{C_A\lceil n_m \rceil/C_R}F^0}{\epsilon^2} \bigr\rceil$ and the local complexity is bounded by $16 + \lceil n_m \rceil 
    + \frac{1024(L_1 + 4\Delta_1 + 33 \sqrt{C_A\lceil n_m \rceil / C_R}\delta_m)F^0}{\epsilon^2} + \frac{4L_1}{\Delta_1 + 11 \sqrt{C_A\lceil n_m \rceil / C_R} \delta_m }$.
\end{theorem}
\begin{proof}
    According to Lemma~\ref{thm:UpperboundKr-random}, by choosing $p = \frac{\lambda - \Delta_1}{8(L_1 + \lambda)}$, the accuracy condition~\eqref{eq:AccuracyCondition} for solving the subproblems is satisfied. 
    Applying Corollary~\ref{thm:Variance-SVRG-RG} and taking the full expectation, for any $T \ge 1$, we have:
    \[
    \sum_{t=0}^T \Sigma_t^2 
    \le 
    \frac{8\beta^2\delta_m^2/p_B^2 +2\delta_m^2}{2\beta - \beta^2}
    \sum_{t=1}^T \chi_t^2 
    \le \Bigl( 
    \frac{8\beta\delta_m^2}{p_B^2}
    + \frac{2\delta_m^2}{\beta}
    \Bigr) \sum_{t=1}^T \chi_t^2  
    \;.
    \]
    Let $\lambda = \frac{1}{a} \Delta_1 + b \delta_m / \sqrt{p_B}$ and $\beta = \frac{p_B}{c}$ where $0 < a < 1$ and $b,c > 0$.
To achieve the error condition~\eqref{eq:EstErrorCondition}, the constants should satisfy:
\[
\Bigl( 
    \frac{12(\lambda + \Delta_1)^2}{(\lambda - \Delta_1)^2}
    +8
    \Bigr)
    \Bigl(\frac{8\beta\delta_m^2}{p_B^2}
    + \frac{2\delta_m^2}{\beta} \Bigr) \le 
    \Bigl( \frac{12(1+a)^2}{(1-a)^2} + 8 \Bigr)
    (8/c + 2c) \delta_m^2 / p_B
    \le b^2 \delta_m^2 / p_B
    \le 
    (\lambda + \Delta_1)^2 \;,
\] 
which gives:
\begin{equation}
\label{eq:Const-ICGM-SVRG}
\Bigl( \frac{12(1+a)^2}{(1-a)^2} + 8 \Bigr) (8/c + 2c) \le b^2 \;.
\end{equation}
Let $a,b,c$ satisfy~\eqref{eq:Const-ICGM-SVRG}.
We can apply Corollary~\ref{thm:CGM-Main-Corollary} and obtain:
\[
    \E[\|\nabla f(\bar{\xx}^T)\|^2]   
    \le \frac{32(\lambda + \Delta_1)^2}{\lambda - \Delta_1} \frac{F^0}{T}
    \le \frac{32(1+a)^2}{1-a}
    \Bigl(
    \frac{1}{a}\Delta_1 
    +b \delta_m / \sqrt{p_B}
    \Bigr)\frac{F^0}{T}
    \;.
\]
Minimizing the coefficient in front of $\Delta_1$ gives $a^\star = \frac{1}{3}$. Choosing 
$b = 22$ and $c = 2$, the condition~\eqref{eq:Const-ICGM-SVRG} is satisfied and we have:
\[
\E[\|\nabla f(\bar{\xx}^T)\|^2]   
    \le \frac{256 (\Delta_1 + 8 \delta_m / \sqrt{p_B})F^0}{T} \;.
\]
Therefore, to achieve $\E[\|\nabla f(\bar{\xx}^T)\|^2] \le \epsilon^2$, we need at most 
$T = \lceil \frac{(256(\Delta_1 + 8 \delta_m / \sqrt{p_B})F^0}{\epsilon^2} \rceil$ iterations.
We next compute the communication and local complexity. 
At iteration $t = 0$, the full gradient $\nabla f(\xx^0)$ is computed which requires $\lceil n_m \rceil$ communication rounds with $\ASS$.
    At each iteration $t \ge 1$, with probability $p_B$, the full gradient is computed, which requires $\lceil n_m \rceil$ rounds with $\ASS$. The expected total number of rounds where $\ASS$ is used is thus bounded by:
    $
    \lceil n_m \rceil + \lceil n_m \rceil p_B T \;. 
    $
    The associated local complexity for each round with $\ASS$ is always $1$.
    For $t \ge 1$, one communication round with $\RSS$ is needed for updating $\gg^t$ and its associated local complexity is $3$ since the client $i \in S_{t-1}$ needs to compute $\nabla f_i (\xx^{t-1})$, 
        $\nabla f_i (\ww^{t-1})$
        and $\nabla f_i (\xx^{t})$. Then another round with $\DSS$ is established, which has the local complexity of $1/p$.
    Therefore, the communication complexity is bounded by:
    \[ 
    N(\epsilon) = \E[C_A N_{A} + C_R N_{R} + N_{D}] \le 
    C_A(\lceil n_m \rceil + \lceil n_m \rceil p_B T)  
    + C_R T + T = C_A\lceil n_m \rceil + (C_A \lceil n_m \rceil p_B + C_R + 1) T \;.
    \] 
    Let $C_A \lceil n_m \rceil p_B = C_R$. We have 
    \[ 
    N(\epsilon) \le 
    C_A \lceil n_m \rceil
    +
    (2C_R + 1) \biggl\lceil \frac{(256(\Delta_1 + 8 \delta_m\sqrt{C_A\lceil n_m \rceil / C_R})F^0}{\epsilon^2} \biggr\rceil \;. 
    \]
    The local complexity $K(\epsilon)$ is bounded by: 
    \begin{align*}
    \E[N_{A} + N_{D}/p + 3N_{R}] 
    &= \lceil n_m\rceil + \lceil n_m \rceil p_B T + T/p + 3T 
    \\
    &\le \lceil n_m \rceil + ( \frac{8 (L_1 + \lambda)}{\lambda - \Delta_1} +4)T 
    \\
    &=  
    \lceil n_m \rceil
    +
    \frac{8L_1 + 32\Delta_1 + 264 \delta_m/\sqrt{p_B}}{2\Delta_1 + 22\delta_m / \sqrt{p_B}}
    \bigl( \frac{(256(\Delta_1 + 8 \delta_m / \sqrt{p_B})F^0}{\epsilon^2} + 1 \bigr)
    \\
    &\le 
    \lceil n_m \rceil 
    + 128 \frac{(8L_1 + 32\Delta_1 + 264 \delta_m/\sqrt{p_B})F^0}{\epsilon^2} + 16 + \frac{4L_1}{\Delta_1 + 11 \delta_m / \sqrt{p_B}} \;.
    \qedhere
    \end{align*}
\end{proof}

\section{Discussion on the \algoname{SAG} estimator}
\label{sec:SAG}
\algoname{SAG} is another incremental gradient method~\citep{sag}. \algoname{Scaffold} has successfully applied it to the FL settings. 
Specifically, the local update rule of device $1$ at outer iteration $t$ (assuming no stochasticity for simplicity) is:
\begin{equation*}
        \yy_{k+1}^t = \yy_k^t - \frac{1}{\eta}\bigl( 
        \nabla f_1(\yy_k^t)
        + \bb^t - \nabla f_1(\xx^t)
        \bigr) \;.
\end{equation*}
Compared with the local CGM~\eqref{Alg:LocalGD}, 
Scaffold sets $\lambda = 0$ and uses $\bb^t$~\eqref{eq:br}(\algoname{SAG}) instead of $\mG^t$ (\algoname{SAGA}) in the control variate. we next show that the variance of $\bb^t$ cannot be controlled by $\delta$. Let $n=2$, $t=1$, $\bb_{1}^{0} = \nabla f_1 (\xx^0)$ and $\bb_{2}^{0} = \nabla f_2 (\xx^0)$. 
Then we get: $\bb^1 = \frac{1}{2} \bigl(\nabla f_1 (\xx^1) + \nabla f_2(\xx^0)\bigr)$,
if $S_1 = \{1\}$ and $\bb^1 = \frac{1}{2} \bigl(\nabla f_2 (\xx^1) + \nabla f_1(\xx^0)\bigr)$,
if $S_1 = \{2\}$.
Then the variance can be computed as:
\begin{align*}
&\E_{S_1}\bigl[ \bigl\lVert 
\bb^1 - \nabla f(\xx^1)
\bigr\rVert^2
\bigr]
=\frac{1}{8}\sum_{i=1}^2\norm{\nabla f_i(\xx^0) - \nabla f_i (\xx^1)}^2\;.
\end{align*}
While for \algoname{SAGA}, we have:
\begin{align*}
&\E_{S_1}\bigl[ \bigl\lVert 
\mG^1 - \nabla f(\xx^1)
\bigr\rVert^2
\bigr]
=\frac{1}{2}\sum_{i=1}^2\norm{\nabla h_i(\xx^0) - \nabla h_i (\xx^1)}^2\;,
\end{align*}
where $h_i \defeq f - f_i$. Therefore, the \algoname{SAG} estimator cannot fully exploit functional similarity as efficiently as \algoname{SAGA} in the worst case from a theoretical perspective.

\section{Additional details and experiments}
We simulate the deep learning experiments on one NVIDIA DGX A100. All the other experiments are run on a MacBook Pro laptop. 

\label{sec:additional-exp}
\begin{figure}
    \centering
    \includegraphics[width=0.9\linewidth]{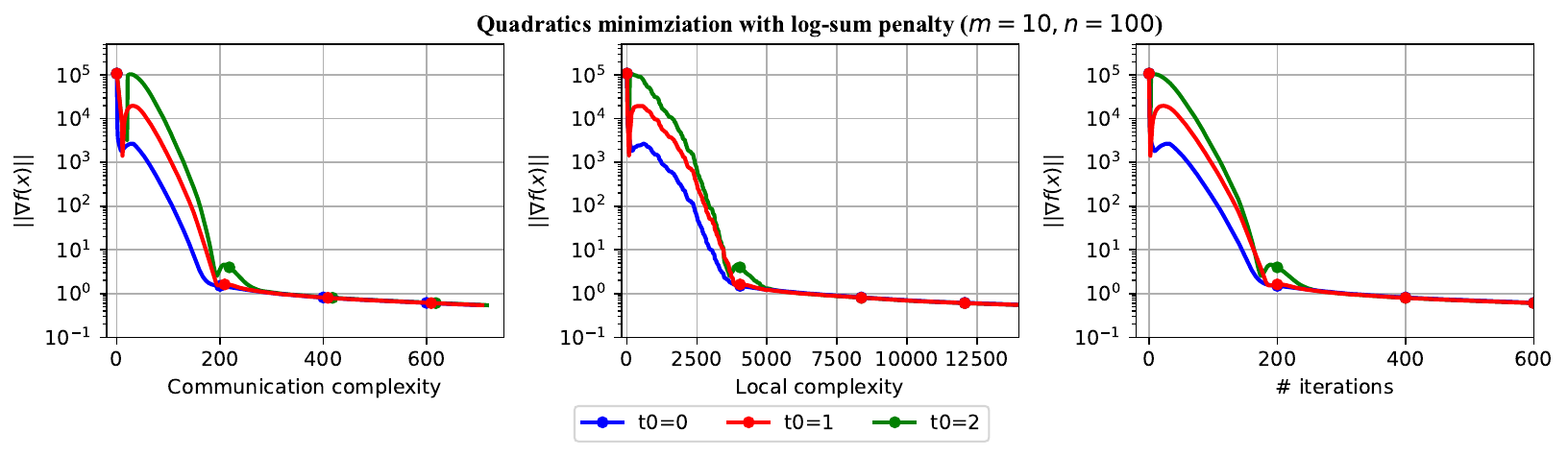}
    \caption{Comparisons of different initialization strategies of I-CGM-RG-SAGA for solving the quadratic minimization problems with non-convex
    log-sum penalty. }
    \label{fig:SAGA_r0_quadratics}
\end{figure}

\subsection{Quadratic minimization with log-sum penalty.} 
Everywhere in the paper, we use the first choice of the control variate for \algoname{Scaffold}~\cite{scaffold}. 
We set the number of local steps $K$ to be $20$ and the local learning rate to be $0.003$ ($0.005$ diverges at the beginning)
for \algoname{FedAvg} and \algoname{Scaffold}. 
For \algoname{SABER-full}, we use the standard gradient method as the local solver and 
set $K$ to be $20$, local learning rate to be $0.005$ and the probability for computing the full gradient to be $0.1$, matching \algoname{I-CGM-RG-SVRG}.
For \algoname{GD}, we run $14000 = 20 * 700$  iterations to match the local gradient computations of other algorithms. Finally, the comparisons of different initialization strategies for \algoname{I-CGM-RG-SAGA} can be found in Figure~\ref{fig:SAGA_r0_quadratics} ($t_0=0,1,2$ correspond to computing the full gradient $0,1,2$ times at the beginning).

\subsubsection{Ablation studies of I-CGM-RG-SAGA}
\label{sec:ablation}
\textbf{Initialization strategies}. 
The comparisons of different initialization strategies for \algoname{I-CGM-RG-SAGA} can be found in Figure~\ref{fig:SAGA_r0_quadratics} ($t_0=0,1,2$ correspond to computing the full gradient $0,1,2$ times at the beginning.  
The result shows that the method works well without any full gradient computations.

\textbf{Local steps}. We now compare the performance of I-CGM-RG-SAGA under different choices of the parameter 
$p$, which is defined in Local CGM~\eqref{Alg:LocalGD}. 
Theoretically,
$p \simeq \frac{\lambda}{\lambda + L_1}$. Since the expected number of local steps per iteration is $\frac{1}{p}$, a smaller $p$ corresponds to more local computations.
In the previous experiments, we used the default value $p = \frac{\delta}{L_1} \approx \frac{5}{100} = 0.05$. We now vary $p \in \{0.5,0.05,0.005\}$. 
From Figure~\ref{fig:SAGA_p_quadratics}, we observe that 1)  Large $p=0.5$ results in worse communication complexity since the local accuracy condition is not fully satisfied; 2) Small $p=0.005$ achieves similar performance to $p = 0.05$ in terms of communication complexity. This is expected, since communication complexity is determined by the fixed parameter $\lambda$. However, the local complexity becomes worse, as the total number of local steps increase and becomes unnecessarily large.

\begin{figure}
    \centering
    \includegraphics[width=0.9\linewidth]{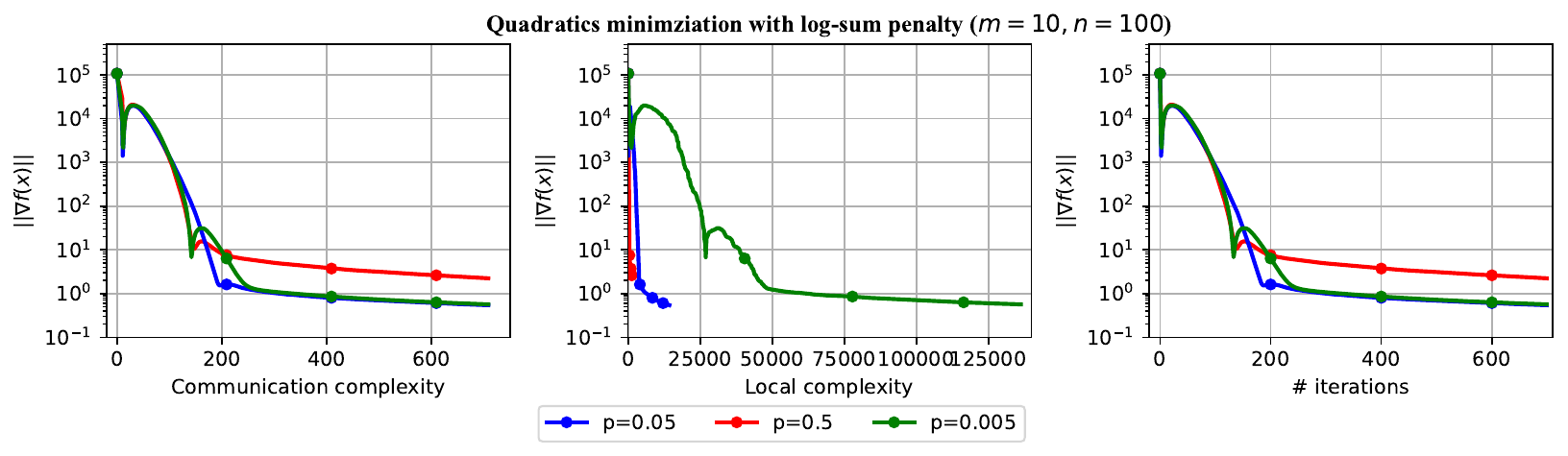}
    \caption{Comparisons of different $p$ (number of local steps) used in local CGM for I-CGM-RG-SAGA when solving the quadratic minimization problems with non-convex
    log-sum penalty. }
    \label{fig:SAGA_p_quadratics}
\end{figure}

\textbf{Constant $\lambda$}.
We now study the impact of the constant $\lambda$ on the performance of I-CGM-RG-SAGA. Note that $\lambda$ directly determines the iteration complexity. Theoretically the best $\lambda \simeq \Delta_1 + \sqrt{n_m}\delta$. In the previous experiments, we used the default value $\lambda = \sqrt{n_m}\delta \approx 15$. We now vary $\lambda \in \{1,10,100\}$.
From Figure~\ref{fig:SAGA_lambda_quadratics}, we observe that: 1) Large $\lambda=100$ results in worse communication complexity since it does not fully use the similarity structure; 
2) Small $\lambda=1$ does not converge as the theory requires $\lambda \gtrsim \Delta_1 + \sqrt{n_m}\delta_m$, all matching the theory.

\begin{figure}
    \centering
    \includegraphics[width=0.9\linewidth]{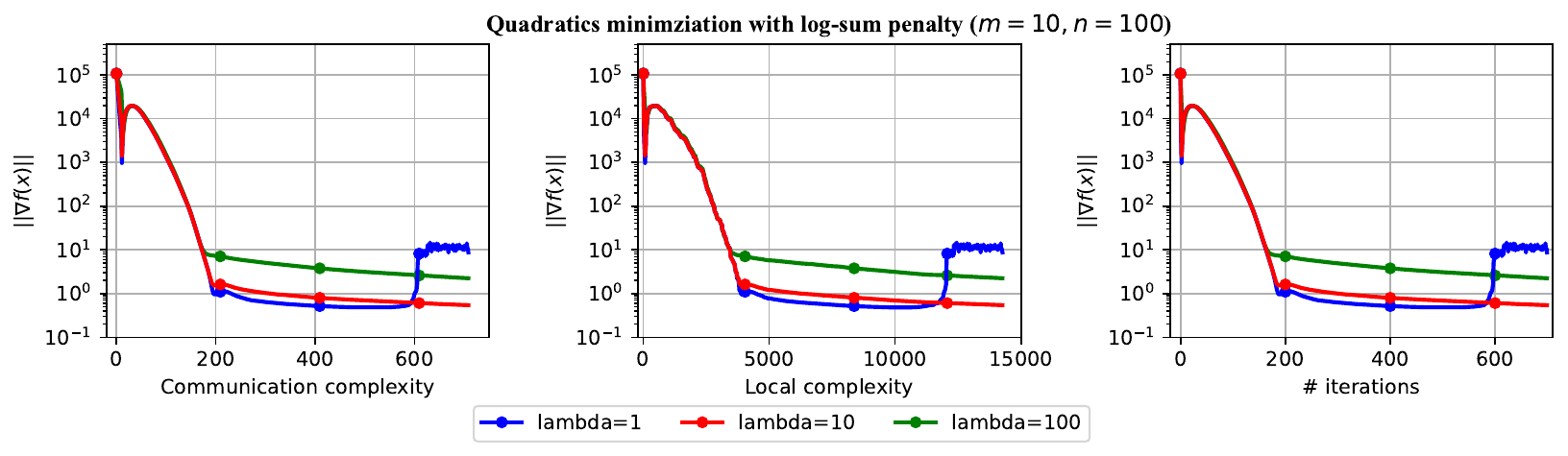}
    \caption{Comparisons of different $\lambda$ used I-CGM-RG-SAGA for solving the quadratic minimization problems with non-convex
    log-sum penalty.}
    \label{fig:SAGA_lambda_quadratics}
\end{figure}

\textbf{Constant $\beta$}.
We now test the effect of $\beta$ used in the~\ref{Alg:RG-update} estimator. 
Both larger or smaller $\beta$ can theoretically increase the variance bound (Lemma~\ref{thm:VarianceQr}).
Theoretically, the best $\beta \simeq \frac{1}{n_m}$. We now vary $\beta \in \{0.5,0.1,0.05,0.01,0.005,0.001\}$.
From Figure~\ref{fig:SAGA_beta_quadratics}, we see that $\beta \in [0.05,0.5]$ results in relatively better performance as $\frac{1}{n_m}=0.1$ and the values that fall outside this range lead to worse communication complexity.

\begin{figure}
    \centering
    \includegraphics[width=0.9\linewidth]{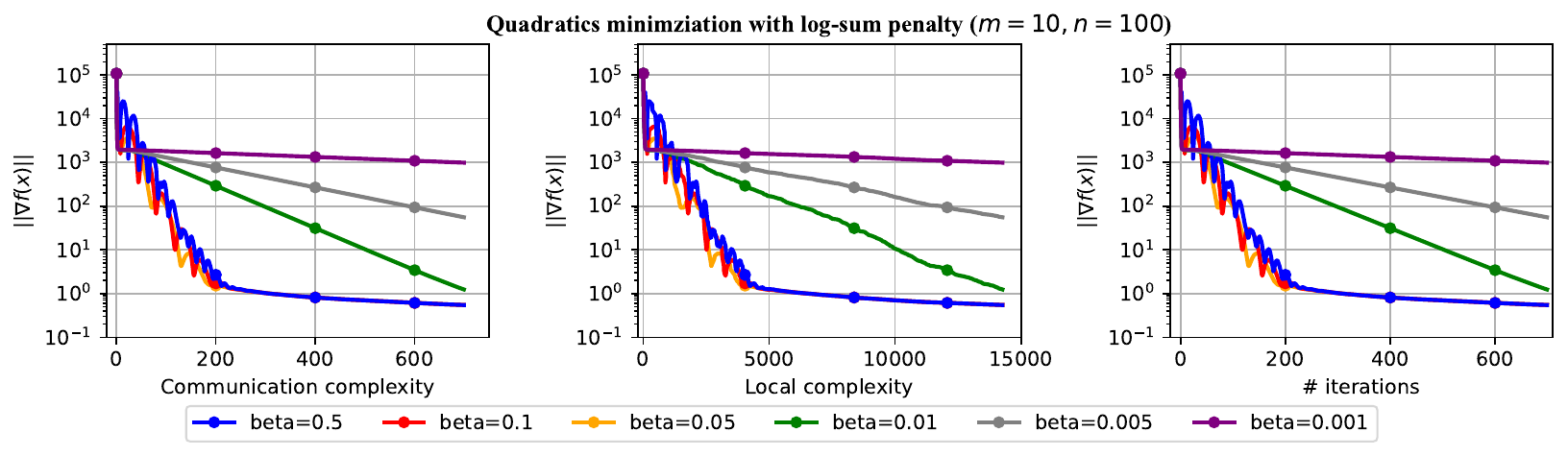}
    \caption{Comparisons of different $\beta$ used I-CGM-RG-SAGA for solving the quadratic minimization problems with non-convex
    log-sum penalty.}
    \label{fig:SAGA_beta_quadratics}
\end{figure}

\textbf{Ratio $\frac{C_A}{C_R}$}.
In the main text, we report results under the extreme setting where $C_A=C_R=1$. Now we test how increasing the ratio $C_A/C_R$ affects the performance. Specifically, we vary $C_A \in \{1,5,10,20\}$ while keeping $C_R=1$, and repeat the same experiments. From Figure~\ref{fig:SAGA_CA_quadratics}, we observe that the performance of I-CGM-RG-SVRG degrades as $C_A$ increases since each use of $\ASS$ becomes more costly. In contrast, I-CGM-RG-SAGA remains largely unaffected, as ASS is only used during initialization. This result further confirms the advantage of I-CGM-RG-SAGA in settings where full synchronization is costly.

\textbf{Ratio $\frac{n}{m}$}.
Finally, we examine how the ratio $\frac{n}{m}$ influences the performance of our method. Theoretically, both the communication and local complexities scale with $\sqrt{n_m} \delta_m F^0 / \epsilon^2$. We fix $m=1$ and vary $n \in \{10,100,1000\}$. 
The datasets are generated in a consistent manner so that the values of $\delta$ and $L_{\max}$ remain approximately unchanged.
We set $\lambda = \sqrt{n_m}\delta \approx 5 \sqrt{n_m}$, $\beta = \frac{1}{n_m}$ and $p = \frac{\lambda}{\lambda + L_{\max}} \approx \frac{\lambda}{\lambda + 100}$ with $n_m = n$. From Figure~\ref{fig:SAGA_n_quadratics}, we observe that increasing  $n_m$ indeed leads to higher communication complexity.
However, the growth is moderate: the additional cost scales by roughly
$\sqrt{100} / \sqrt{10} = \sqrt{1000} / \sqrt{100} \approx 3$ rather than linearly $100 / 10 = 100 / 10 = 10$, confirming that the dependence is on $\sqrt{n_m}$ instead of $n_m$.  

\begin{figure}
    \centering
    \includegraphics[width=0.9\linewidth]{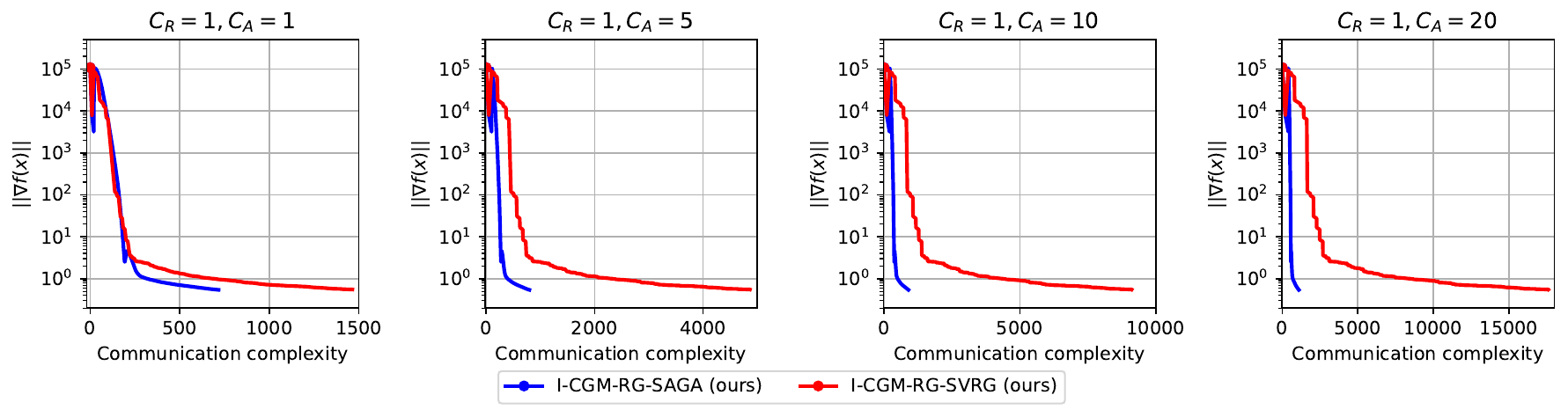}
    \caption{Comparisons of I-CGM-RG-SAGA against I-CGM-RG-SVRG under different $C_A/C_R$ for solving the quadratic minimization problems with non-convex
    log-sum penalty.}
    \label{fig:SAGA_CA_quadratics}
\end{figure}
    
\begin{figure}
    \centering
    \includegraphics[width=0.9\linewidth]{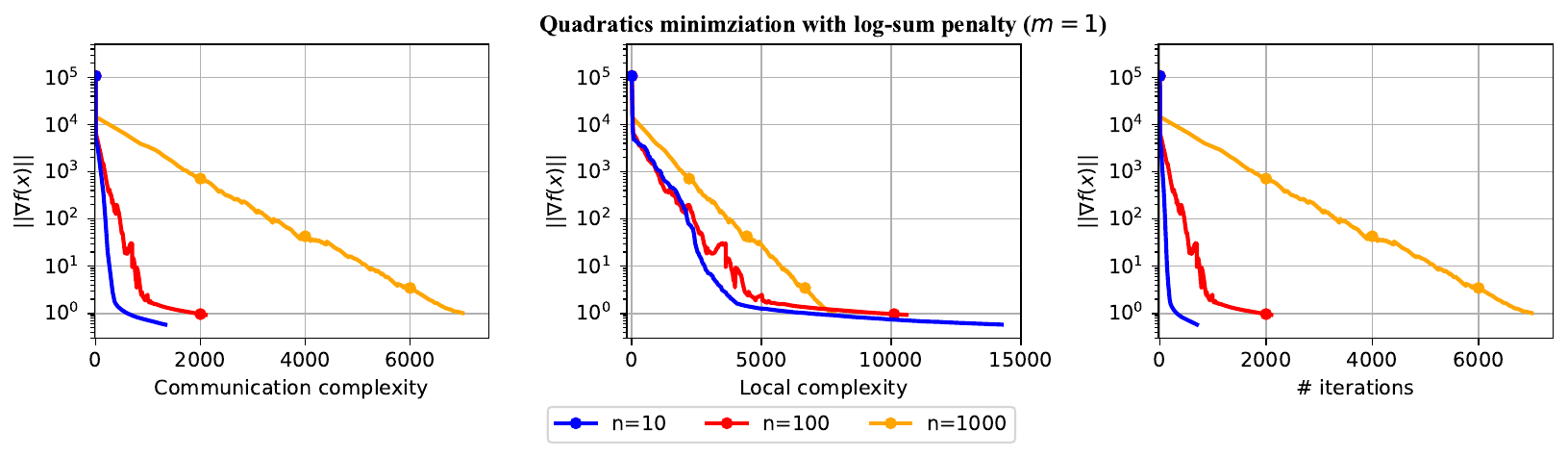}
    \caption{Comparisons of I-CGM-RG-SAGA under different $n_m$ for solving the quadratic minimization problems with non-convex
    log-sum penalty.}
    \label{fig:SAGA_n_quadratics}
\end{figure}

\begin{figure}
    \centering
    \includegraphics[width=0.7\linewidth]{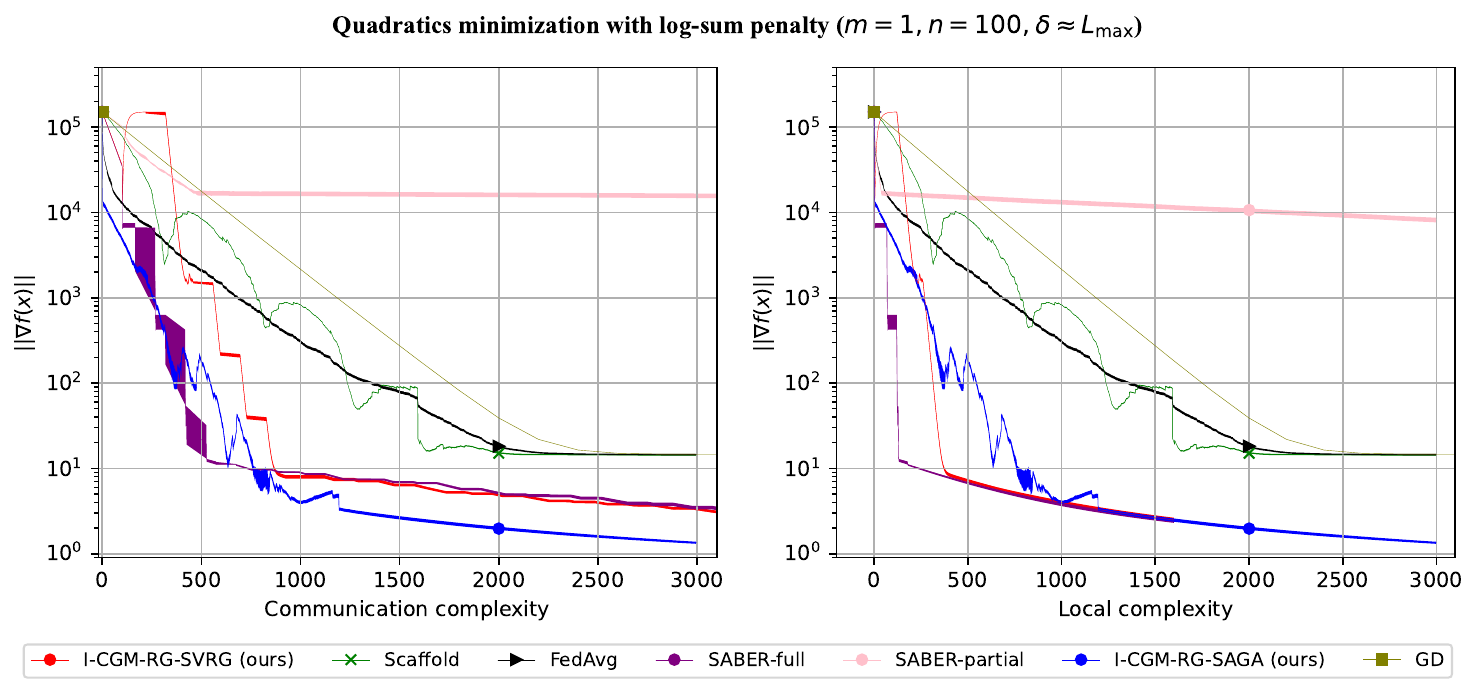}
    \caption{Comparisons of different algorithms for solving the quadratic minimization problems with non-convex log-sum penalty when $\delta \approx L_{\max}$ where the thickness reflects the error bar.}
\end{figure}

\subsection{Logistic regression with nonconvex regularizer.} 
For both datasets, we set $p=0.1$ in Local GD for \algoname{CGM-RG} methods and \algoname{Scaffnew}, and use $K=10$ local steps for the other algorithms. We select the best local learning rate for each method from $\{0.1, 0.2, 0.5, 1.0\}$ for Mushroom and $\{0.002, 0.001, 0.0005\}$ for Duke. For proximal-point methods, we choose the best $\lambda$ from $\{10, 1, 0.1, 0.01\}$ on both datasets. We use $\beta = \frac{m}{n}$ for both \algoname{I-CGM-RG} methods. 

\subsection{EMNIST with Residual CNN}
We now extend our study to neural network training. Specifically, we train a 6-layer Residual CNN on the EMNIST dataset~\cite{emnist}, which consists of a collection of $26$ letter classes. We use $n=26$ and $m=5\approx\sqrt{n}$, and split the dataset according to the Dirichlet distribution with $\alpha=0.1$ (the smaller the $\alpha$, the higher the heterogeneity, $\alpha=0.1$ is highly heterogeneous). We use a batch size of $128$ for computing both the local stochastic gradient and the control variates. For all the methods that use control variates, including \algoname{I-CGM-RG}, \algoname{Scaffold}, \algoname{SABER} and \algoname{Scaffnew}, we add a damping factor $q$ in front of the control variate to enhance their empirical performance, i.e., on line 5 of Algorithm~\ref{Alg:LocalGD}, we use 
$
    \yy_{k+1}^t
    =
    \argmin_{\yy \in \R^d}
    \bigl\{f_1(\yy_k^t) 
    + \langle \gg_1(\yy_k^t) + q (\gg^t - \gg_1(\xx^t)), \yy - \yy_k^t \rangle
    + \frac{\eta}{2} \norm{\yy - \yy_k^t}^2
    + \frac{\lambda}{2}||\yy-\xx^t||^2 \bigr\},$
where $q \in (0,1]$ is a tuned parameter and $\gg_1$ is the stochastic mini-batch gradient of $\nabla f_1$. 
This approach is suggested by~\cite{yin2025a}.
We report the best local stepsize $\frac{1}{\eta}$ among $\{0.05,0.02,0.01,0.001\}$ and the best $\lambda$ among $\{0.001,0.01,0.1,1\}$. The final choices of the parameters can be found in Table~\ref{tab:emnist}. The convergence behaviours can be found in Figure~\ref{fig:emnist}. The best validation accuracy can be found in Table~\ref{tab:accuracy_emnist}, where \algoname{I-CGM-RG-SAGA} performs the best.

\begin{table*}[ht!]
\resizebox{\textwidth}{!}
{\begin{minipage}{1.2\textwidth}
\centering
\begin{tabular}{@{}c|c@{}}
\toprule
\multicolumn{1}{c}{\textbf{optimizers}} &
\multicolumn{1}{c}{\textbf{hyper-parameters used for multi-classification tasks}} \\
\midrule
\algoname{I-CGM-RG-SAGA} & $\frac{1}{\eta} = 0.02$,
$\lambda = 0.01$, $p=0.01$, $\beta = 0.2$,
$q = 0.001$, $t_0 = 0$
\\
\algoname{I-CGM-RG-SVRG} & $\frac{1}{\eta} = 0.02$,
$\lambda = 0.01$, $p=0.01$, $\beta = 0.2$,
$q = 0.001$\\
\algoname{Scaffold}~\cite{scaffold} & $\frac{1}{\eta} = 0.02$,
 $K=100$, $q=0.001$ \\
\algoname{FedAvg}~\cite{fedavg} & $\frac{1}{\eta} = 0.02$, $K=100$
 \\
\algoname{Scaffnew}~\cite{scaffnew} & $\frac{1}{\eta} = 0.02$,
 $p=0.01$, $q=0.001$ \\ 
\algoname{SABER}~\cite{saber} & $\frac{1}{\eta} = 0.02$,
$\lambda = 0.01$, $p=0.01$, $\beta = 0.2$,
$q = 0.001$ \\ 
\bottomrule
\end{tabular}
\end{minipage}}
\caption{Hyper-parameters of the considered optimizers used in the multi-classification task for the EMNIST dataset.}
\label{tab:emnist}
\end{table*}

\begin{figure}
    \centering
    \includegraphics[width=1.0\linewidth]{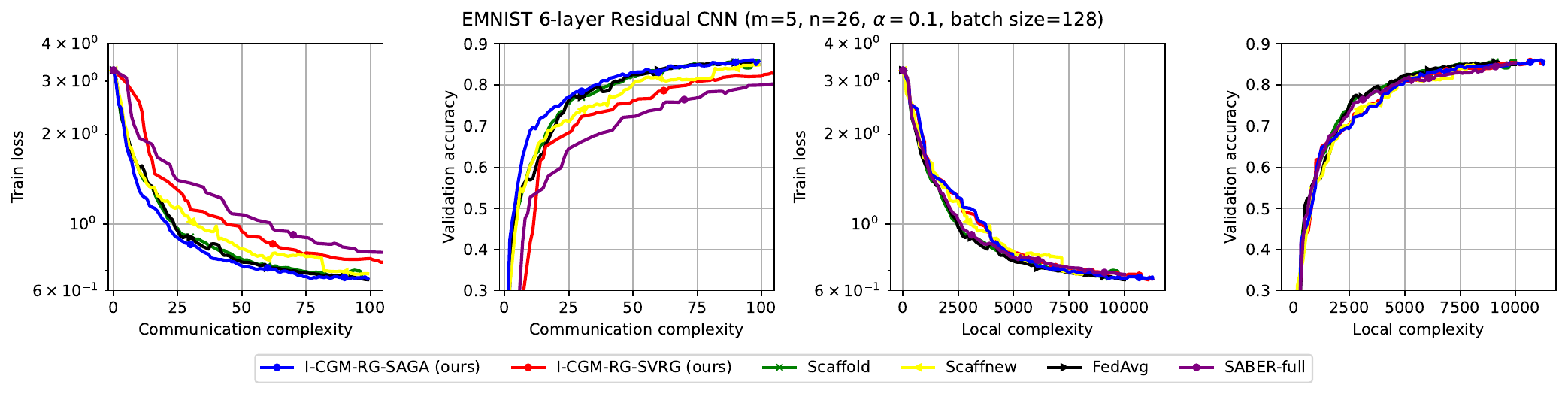}
    \caption{Comparisons of different algorithms on the EMNIST dataset using a 6-layer residual CNN.}
    \label{fig:emnist}
\end{figure}

\begin{table*}[ht!]
\resizebox{\textwidth}{!}
{\begin{minipage}{1.2\textwidth}
\begin{tabular}{l|cccccc}
\textbf{Optimizers} & \algoname{I-CGM-RG-SAGA} & \algoname{I-CGM-RG-SVRG} & \algoname{SABER-full} & \algoname{Scaffold} & \algoname{Scaffnew} & \algoname{FedAvg} \\
\hline
\textbf{Accuracy} & \textbf{86.2} & 86.0 & 85.3 & 85.9 & 84.9 & 85.6 \\
\end{tabular}
\end{minipage}}
\caption{Comparisons of validation accuracy for different optimizers used in the multi-classification task for the EMNIST dataset within $100$ outer iterations.}
\label{tab:accuracy_emnist}
\end{table*}

\subsection{CIFAR10 with ResNet18}
We now consider multi-class classification tasks with CIFAR10~\cite{cifar10} using ResNet18~\cite{resnet}. We use $n=10$ and $m=3\approx\sqrt{n}$, and split the dataset according to the Dirichlet distribution with $\alpha=0.1$ (highly heterogeneous). We use a batch size of $128$ for computing both the local stochastic gradient and the control variates $\mm$. 
We report the best local stepsize $\frac{1}{\eta}$ among $\{0.1,0.05,0.01,0.001\}$ and the best $\lambda$ among $\{0.001,0.01,0.1,1\}$. The final choices of the parameters can be found in Table~\ref{tab:cifar10}. The convergence behaviours can be found in Figure~\ref{fig:cifar10}. The best validation accuracy can be found in Table~\ref{tab:accuracy_cifar10}.

\begin{figure}
    \centering
    \includegraphics[width=1.0\linewidth]{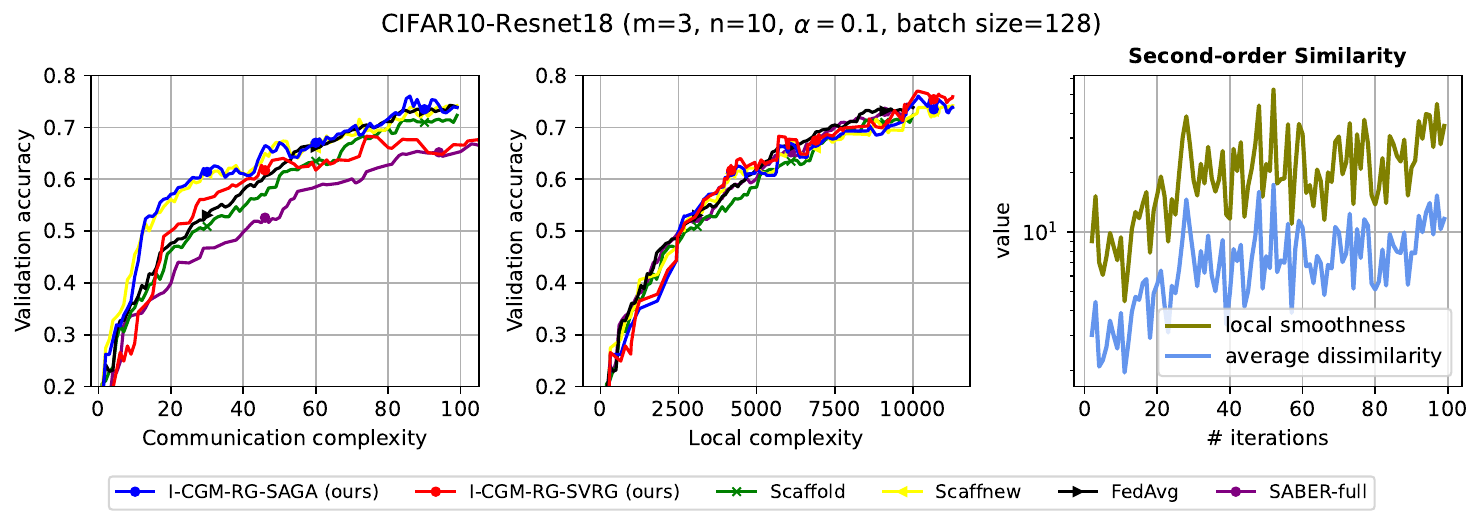}
    \caption{Comparisons of different algorithms on the CIFAR10 dataset using ResNet18.}
    \label{fig:cifar10}
\end{figure}

\begin{table*}[ht!]
\resizebox{\textwidth}{!}
{\begin{minipage}{1.2\textwidth}
\centering
\begin{tabular}{@{}c|c@{}}
\toprule
\multicolumn{1}{c}{\textbf{optimizers}} &
\multicolumn{1}{c}{\textbf{hyper-parameters used for multi-classification tasks}} \\
\midrule
\algoname{I-CGM-RG-SAGA} & $\frac{1}{\eta} = 0.05$,
$\lambda = 0.01$, $p=0.01$, $\beta = 0.2$,
$q = 0.001$, $t_0 = 0$
\\
\algoname{I-CGM-RG-SVRG} & $\frac{1}{\eta} = 0.05$,
$\lambda = 0.01$, $p=0.01$, $\beta = 0.2$,
$q = 0.001$\\
\algoname{Scaffold}~\cite{scaffold} & $\frac{1}{\eta} = 0.05$,
 $K=100$, $q=0.001$ \\
\algoname{FedAvg}~\cite{fedavg} & $\frac{1}{\eta} = 0.05$, $K=100$
 \\
\algoname{Scaffnew}~\cite{scaffnew} & $\frac{1}{\eta} = 0.05$,
 $p=0.01$, $q=0.001$ \\ 
\algoname{SABER}~\cite{saber} & $\frac{1}{\eta} = 0.05$,
$\lambda = 0.01$, $p=0.01$, $\beta = 0.2$,
$q = 0.001$ \\ 
\bottomrule
\end{tabular}
\end{minipage}}
\caption{Hyper-parameters of the considered optimizers used in the multi-classification task for the CIFAR10 dataset.}
\label{tab:cifar10}
\end{table*}

\begin{table*}[ht!]
\resizebox{\textwidth}{!}
{\begin{minipage}{1.2\textwidth}
\begin{tabular}{l|cccccc}
\textbf{Optimizers} & \algoname{I-CGM-RG-SAGA} & \algoname{I-CGM-RG-SVRG} & \algoname{SABER-full} & \algoname{Scaffold} & \algoname{Scaffnew} & \algoname{FedAvg} \\
\hline
\textbf{Accuracy} & 76.1 & \textbf{77.0} & 74.5 & 72.3 & 74.2 & 74.3 \\
\end{tabular}
\end{minipage}}
\caption{Comparisons of validation accuracy for different optimizers used in the multi-classification task for the CIFAR10 dataset within $100$ outer iterations.}
\label{tab:accuracy_cifar10}
\end{table*}

\end{document}